\documentclass{article}
\pdfoutput=1

\PassOptionsToPackage{sort}{natbib}
\usepackage[preprint, nonatbib]{neurips_2020}

\usepackage{microtype}
\usepackage{graphicx}
\usepackage{subfigure}
\usepackage{booktabs} 
\usepackage{amssymb}
\usepackage{bm}
\usepackage{bbm}
\usepackage{amsmath}  
\usepackage{amsthm}
\usepackage{xcolor}
\usepackage{textcomp}
\usepackage{multirow}
\usepackage{mathtools}
\usepackage{algorithm}
\usepackage[noend]{algorithmic}
\usepackage{caption}
\usepackage{xspace}
\usepackage{wrapfig}
\usepackage{verbatim}

\usepackage{tikz}
\usetikzlibrary{arrows}
\usetikzlibrary{positioning}

\tikzset{
  treenode/.style = {align=center, inner sep=0pt, text centered,
    font=\sffamily},
  arn_n/.style = {treenode, circle, black, font=\sffamily\bfseries, draw=black,
    fill=white, text width=1.5em},
  arn_r/.style = {treenode, circle, black, font=\sffamily\bfseries, draw=black,
    fill=white, text width=1.0em},
  arn_x/.style = {treenode, rectangle, draw=black,
    minimum width=0.5em, minimum height=0.5em}
}

\usepackage{hyperref}

\usepackage{bm}
\usepackage{enumitem}

\newtheorem{theorem}{Theorem}
\newtheorem{lemma}{Lemma}
\newtheorem{corollary}{Corollary}

\newtheorem{definition}{Definition}

\newtheorem{proposition}{Proposition}

\newcommand{\E}[2]{\mathbb{E}_{#1}\left[#2\right]}
\newcommand{\Ehat}[1]{\hat{\mathbb{E}}\left[#1\right]}

\newcommand{\ind}[1]{\mathbbm{1}\left\{#1\right\}}

\newcommand{\supp}[0]{\text{supp}}
\newcommand{\dist}[0]{d_{\mathcal{Z}}}
\newcommand\independent{\protect\mathpalette{\protect\independenT}{\perp}}
\def\independenT#1#2{\mathrel{\rlap{$#1#2$}\mkern2mu{#1#2}}}
\newcommand{\sign}[1]{\text{sign}(#1)}
\newcommand{\gdep}[0]{G_{\text{dep}}}
\newcommand{\ydep}[1]{Y^{\text{dep}}(#1)}

\newcommand{\D}[0]{\mathcal{D}}

\newcommand{\lf}[0]{\lambda}

\newcommand{\R}[0]{\mathcal{R}}
\newcommand{\X}[0]{\mathcal{X}}

\ifdefined\ShowNotes
  \newcommand{\colornote}[3]{{\color{#1}\bf{#2 #3}\normalfont}}
\else
  \newcommand{\colornote}[3]{}
\fi

\definecolor{darkred}{rgb}{0.7,0.1,0.1}
\definecolor{darkgreen}{rgb}{0.1,0.5,0.1}
\definecolor{cyan}{rgb}{0.7,0.0,0.7}
\definecolor{dblue}{rgb}{0.2,0.2,0.8}
\definecolor{maroon}{rgb}{0.76,.13,.28}
\definecolor{burntorange}{rgb}{0.81,.33,0}
\definecolor{royalpurple}{rgb}{0.47,.31,0.66}

\ifdefined\ShowNotes
  \newcommand{\num}[1]{{\color{red}\bf{#1}\normalfont}}
\else
  \newcommand{\num}[1]{#1}
\fi

\newcommand{\sysx}{\textsc{Epoxy}\xspace}

\newcommand{\spam}{\textbf{Spam}}
\newcommand{\spouse}{\textbf{Spouse}}
\newcommand{\weather}{\textbf{Weather}}

\newcommand{\commercial}{\textbf{Commercial}}
\newcommand{\tennis}{\textbf{Tennis}}
\newcommand{\basketball}{\textbf{Basketball}}

\newcommand{\avgliftoverlmall}{\num{4.1}}

\newcommand{\avgliftovertl}{\num{12.8}}
\newcommand{\avgliftovertsft}{\num{13.1}}
\newcommand{\avgliftoverfixedr}{\num{3.0}}
\newcommand{\avgliftdnoverepoxy}{\num{0.7}}
\newcommand{\avgliftdnovertsft}{\num{13.8}}

\newif\ifarxiv
\arxivtrue

\title{Train and You'll Miss It: Interactive Model Iteration with Weak Supervision and Pre-Trained Embeddings}

\author{
Mayee F. Chen$^*$ \quad Daniel Y. Fu$^{*\dagger}$ \quad Frederic Sala \quad Sen Wu \quad Ravi Teja Mullapudi \\
\textbf{Fait Poms \quad Kayvon Fatahalian \quad Christopher R\'e} \\
\texttt{\{mfchen, danfu, fredsala, senwu, fpoms, kayvonf, chrismre\}@cs.stanford.edu,} \\
\texttt{mullapudi@cs.cmu.edu}
}

\begin{document}

\maketitle


\begin{abstract}

Our goal is to enable machine learning systems to be trained interactively.
This requires models that perform well and train quickly, without large
amounts of hand-labeled data.
We take a step forward in this direction by borrowing from weak supervision
(WS), wherein models can be trained with noisy sources of signal
instead of hand-labeled data.
But WS relies on training downstream deep networks to
extrapolate to unseen data points, which can take hours or days.
Pre-trained embeddings can remove this requirement.
We do not use the embeddings as features as in transfer learning (TL),
which requires fine-tuning for high performance, but instead use them to define
a distance function on the data and extend WS source votes to
nearby points.
Theoretically, we provide a series of results studying how performance
scales with changes in source coverage, source accuracy, and the
Lipschitzness of label distributions in the embedding space, and
compare this rate to standard WS without extension and TL without
fine-tuning.
On six benchmark NLP and video tasks, 
our method outperforms WS without extension by \avgliftoverlmall\ points,
TL without fine-tuning by \avgliftovertl\
points,
and traditionally-supervised deep networks by \avgliftovertsft\ points,
and comes within \avgliftdnoverepoxy\ points of state-of-the-art
weakly-supervised deep
networks---all while training in less
than \num{half a second}.

\end{abstract}


\section{Introduction}
\label{sec:intro}

The introduction of the interactive console in 1962 revolutionized how
people write computer programs~\cite{corbato1962experimental}. 
Machine learning similarly stands to benefit from enabling
\textit{programmatically-interactive} model iteration cycles---training models,
inspecting their results, and fixing failure modes in seconds, instead of hours
or days.
This requires models that perform well and train quickly, without requiring large
amounts of hand-labeled data.
Modern deep learning models can achieve high performance, but only at the
cost of heavy training over large amounts of labeled data to optimize feature representations,
precluding the possibility of such interactive timescales.
In this paper, we build models that can be iterated on interactively and
achieve high performance without having to optimize features through
training.

To do so, we borrow from work in weak supervision (WS), which enables more
interactive interfaces for model iteration.
WS methods automatically
generate probabilistic labels for training data from multiple weak
sources---such as heuristics,
external
knowledge bases, and user-defined functions \cite{gupta2014improved, Ratner19,
dehghani2017neural,dehghani2017learning,
jia2017constrained,mahajan2018exploring,niu2012deepdive,
karger2011iterative, dawid1979maximum, mintz2009distant, zhang:cacm17,
hearst1992automatic} instead of relying on hand labels.
WS has been used throughout industry~\cite{sheng2020gmail, re2019overton,
bach2018snorkel, shu2020learning} and academia~\cite{fries2018weakly,saab2020weak}
to enable more interactive model iteration; instead of 
adjusting model behavior by hand-labeling training data, users instead write 
heuristics or tweak existing sources to change the training labels.

The WS probabilistic labels can be used directly as
final model predictions, but can have trouble extrapolating to unseen data
points.
WS frameworks often train downstream deep networks to achieve good performance
in these cases.
Deep networks can pull out more general signals from the data than are encoded
in the weak sources, which often have low coverage over the dataset and thus do not
extrapolate well on their own.
Figure~\ref{fig:generalization} (middle top) shows a representative example.
In text applications, deep networks can find synonyms to
user-provided key words or relax word order, whereas user-provided heuristic
functions often encode specific phrases.
But training these deep networks can take hours or
days, slowing down model iteration past interactive
timescales.

How can we get the interactivity benefits of programmatic WS
without having to train deep networks for good performance?
The popularity of transfer learning (TL) suggests that using deep networks
pre-trained on large datasets can reduce
training costs~\cite{devlin2018bert, deng2009imagenet, kolesnikov2019large}.
But TL often requires fine-tuning 
to achieve high
accuracies---since pre-trained embeddings may not contain
enough information to be globally
discriminative (Figure~\ref{fig:generalization} middle
bottom).
Although fine-tuning is cheaper than training from
scratch, it still requires many iterations of gradient descent.

\begin{figure*}[t]
  \centering
  \includegraphics[width=5.5in]{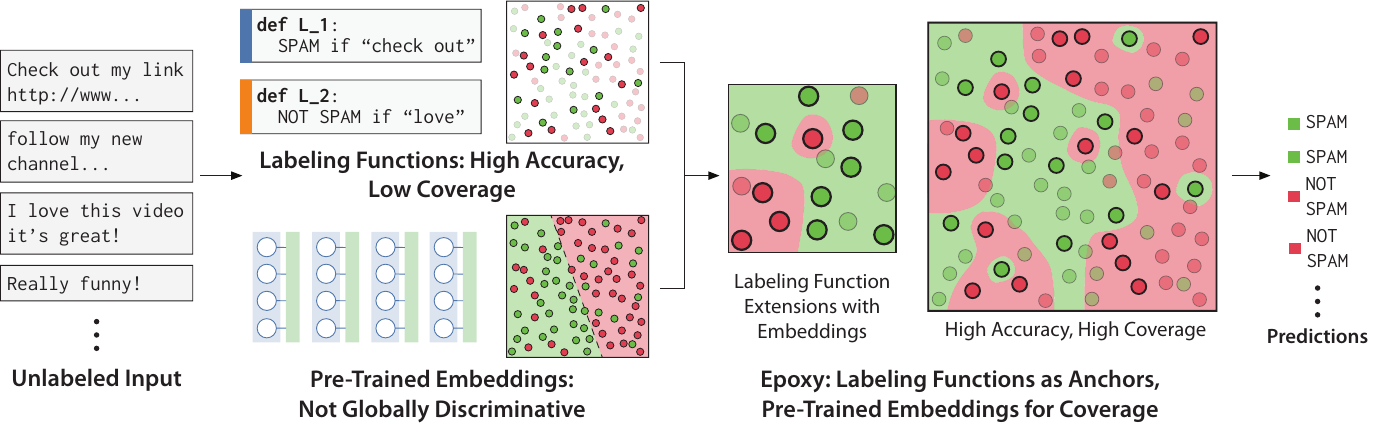}
  \caption{
    Weak sources may have high accuracy but only cover a limited subset of the
    data.
    Pre-trained embeddings may not be globally discriminative without
    fine-tuning.
    \sysx\ combines both by using pre-trained embeddings to
    locally extend labeling sources, achieving high overall performance.
  }
  \label{fig:generalization}
\end{figure*}

We propose a different mechanism for using pre-trained embeddings to enable
high-quality WS without training a deep network.
Although pre-trained embeddings might not be sufficiently discriminative to
use as features alone, we observe that they can often be used to
create a sufficiently-reliable local notion of distance between points (essentially
a type of kernel).
We use this distance notion to smoothly transfer source votes to points that are
nearby in embedding space (Figure~\ref{fig:generalization} right)---exploiting
the idea that nearby points in a relevant embedding space should have similar
labels.
This mechanism is similar to semi-supervised techniques like label
propagation~\cite{zhur2002learning, iscen2019label}, but we use the machinery
of WS to anchor on weakly-labeled points instead of assuming gold
labels (and we do not tune the embeddings).
Perhaps surprisingly, our approach, named \sysx, can sometimes recover the
performance of state-of-the-art weakly-supervised downstream deep networks, and
outperform traditionally-supervised deep
networks---without having to rely on an
expensive training procedure.

Theoretically, we provide a sequence of results evaluating how much signal we
can extract from pre-trained embeddings using
\sysx. First, we define a notion of
probabilistic Lipschitzness~\cite{kushagra2015} to describe the label space's smoothness, with which we can tightly characterize the improvement in rate of
convergence and generalization error of \sysx\ in the size of the (unlabeled)
data over WS without source extensions. We show that improvement depends on two factors: the increase in coverage from extending weak labels, and the accuracy of the extension.
Next, we show how the Lipschitzness of the task label distribution is related to the
intrinsic quality of the model used to produce the pre-trained embedding, and we use this connection to bound how close we are to the performance of such models.
Finally, we describe the conditions under which our method
outperforms TL without fine-tuning; we do so by viewing
the weak sources as an ensemble of specialized models.

We empirically validate \sysx\ on six benchmark NLP and video applications
that have been used to evaluate previous state-of-the-art weak supervision
frameworks (and thus have weak sources readily
available)~\cite{Ratner18, fu2020fast}.
\sysx\ outperforms WS without source extensions and without training a
downstream deep network by
\avgliftoverlmall\ points, and outperforms
non-finetuned TL by
\avgliftovertl\ points.
\sysx\ achieves training cycles of less than half a second---but comes within
within \avgliftdnoverepoxy\ points of training deep networks with 
state-of-the-art weak supervision techniques, which themselves outperform
fully-supervised deep network baselines by \avgliftdnovertsft\ points.


\section{Preliminaries}
\label{sec:background}

We provide necessary background on
weak supervision and
pre-trained deep networks.
A more in-depth discussion is presented in the related work
(Section~\ref{sec:related}).

\textbf{Weakly-supervised machine learning.}
In WS~\cite{Ratner16, Ratner18, Ratner19}, practitioners
programmatically generate training labels from multiple noisy label sources.
For example, a practitioner trying to classify comments as
spam could write heuristics looking for key words like
``check out" (comments that say ``check out my channel" are likely to be spam).
Sources can vote or abstain on
individual points (while ``check out" is a strong signal
for \textit{spam}, its absence does not mean
the comment is \textit{not spam}).

Formally, let $X = [X^1, X^2, \ldots, X^T] \in \mathcal{X}$ be a vector of
$T$ related elements (e.g., contiguous frames in a video)
and $Y = [Y^1, Y^2, \ldots, Y^T] \in \mathcal{Y}$ be the vector of
\textit{unobserved} true task labels (we refer to each $Y^i$ as a
task), with
$(X, Y)$ drawn from some distribution $\mathcal{D}$.
Let there be $m$ weak sources $\bm{\lf}$, each
voting or abstaining per $X$ via a \textit{labeling function}
$\lf_j: \mathcal{X} \rightarrow \mathcal{Y} \cup \{0\}$ for all
$j = 1, \dots, m$ ($\lf_j(X) = 0$ is an abstain).
Let $\lf'_j$ index abstentions, with 
$\lf'_j(X) = 0$ if $\lf_j(X) = 0$ and $\lf_j'(X) = 1$ otherwise.
Let $\supp(\lf_j)$ be the points $\lf_j$ votes on,
with \textit{coverage} the proportional size of $\supp(\lf_j)$,
and \textit{minimal overlap} $o_{\min}$ be
 $o_{\min} = \min_i  \max_{j: j \neq i} \Pr(X \in \supp (\lf_i) \cap \supp (\lf_j))$.

The goal is to learn the joint distribution
$P(Y, \bm{\lambda})$, called the \textit{label model}, to combine votes into
element-wise probabilistic labels by applying
the $m$ sources to an unlabeled dataset $D = \{X_i\}_{i = 1}^n$.
The label model can be used directly for inference (i.e., the probabilistic
labels are the final predictions) but is only useful on points in the support of the
labeling functions.
For better extrapolation, the labels are used to train an end model
$f_w \in \mathcal{F} : \mathcal{X} \rightarrow \mathcal{Y}$.

\textbf{Pre-trained networks.}
Formally, we treat pre-trained networks as mapping
$g: \mathcal{X} \rightarrow \mathcal{Z}$ from data points into an embedding
space $\mathcal{Z}$.
Fine-tuning changes this mapping during training; transfer
learning without fine-tuning learns a function
$f_z : \mathcal{Z} \rightarrow \mathcal{Y}$.
We associate with $\mathcal{Z}$ a distance function
$\dist: \mathcal{Z} \times \mathcal{Z} \rightarrow \mathbb{R}$, and for convenience
write $\dist(X_1, X_2)$ for $\dist(g(X_1), g(X_2))$.
To characterize the behavior of label distributions with respect to the embedding space, we adapt the idea of conditional probabilistic Lipschitzness
in~\cite{kushagra2015, Pentina18} to present a variant of probabilistic bi-Lipschitzness:
\begin{definition}
\textbf{$\bm{(L, M)}$-Lipschitzness}. Define the metric space ($\mathcal{Z}, d_{\mathcal{Z}})$ and functions $L, M: \mathbb{R} \rightarrow [0, 1]$. For $V \in \{\lf'_1, \ldots, \lf'_m, Y\}$, we say that the function $V$ is $(L, M)$-Lipschitz if for all $r > 0$,
\begin{align*}
L(r) \leq P_{X_1, X_2}(V(X_1) \neq V(X_2) \; | \; \dist(X_1, X_2) \le r) \le M(r),
\end{align*} 
where $Y(X_1) = Y_1, Y(X_2) = Y_2$, and $(X_1, Y_1), (X_2, Y_2) \sim \mathcal{D}$.
\end{definition}


\section{Improving Label Model Performance with Pre-trained Embeddings}
\label{sec:model}

\begin{algorithm}[t]
	\caption{Labeling Function Extension Method}
	\begin{algorithmic}
		\STATE \textbf{Input:}
		Unlabeled dataset $D = \{X_i\}_{i = 1}^n$, labeling functions $\bm{\lf} = \lf_1, \dots, \lf_m$,  embedding distance function $\dist$, threshold radii $\bm{r} \in \mathbb{R}^m$, weighting function $W$.
		\STATE \textbf{Returns:} $\bar{\lf}_1, \dots \bar{\lf}_m$
		\FOR{$\lf_j \in \bm{\lf}, X_i \in D$}
			\IF{$\lf_j(X_i) \neq 0$}
				\STATE $\bar{\lf}_j(X_i) := \lf_j(X_i)$
			\ELSE
				\STATE Compute the set of $X_i$'s labeled neighbors, $N_{r_j}(X_i) = \{ X \in \supp(\lf_j): \dist(X, X_i) \le r_j \} $.
				\vspace{-1em}
				\IF{$|N_{r_j}(X_i)| \neq 0$}
					\STATE Assign $\bar{\lf}_j(X_i) = W(N_{r_j}(X_i))$.
				\ELSE
					\STATE $X_i$ is too far from $\supp(\lf_j)$, so the new labeling function still abstains with $\bar{\lf}_j (X_i) = 0$.
				\ENDIF
			\ENDIF
		\ENDFOR
	\end{algorithmic}
	\label{alg:ext}
\end{algorithm}

We describe \sysx. Our goal is to combine pre-trained embeddings and noisy labeling functions to approach the performance of a deep network without the need for fine-tuning. As in the standard WS approach, we are given
labeling functions $\bm{\lf} = \lf_1, \dots, \lf_m$ and
an unlabeled dataset $D = \{X_i\}_{i = 1}^n$.
Instead of learning $P(Y, \bm{\lf})$ directly from these inputs, we
add an extension stage to generate extended labeling functions $\bm{\bar{\lf}}$, and
then learn $P(Y, \bm{\bar{\lf}})$ using WS techniques.

\textbf{Extension stage.} In the \emph{extension} stage, we examine the output of $\lf_j$.
If it votes on $X_i$, e.g. $X_i \in \supp(\lf_j)$, the vote is unchanged.
If $\lf_j$ abstains, we use pre-trained embeddings to potentially produce a vote on $X_i$. To do so, 
we use an \textit{extension radius} $r_j$ for each labeling function $\lf_j$.
The extension process is simple---if $X_i$ is not labeled by $\lf_j$, we search for the neighborhood of points labeled by $\lf_j$ within a radius of $r_j$ from $X_i$, which we denote as the set $N_{r_j}(X_i)$, and assign to $X_i$ the thresholded output of the function $W(N_{r_j}(X_i)) = \sum_{X_k \in N_{r_j}(X_i)} w_k \lf_j(X_k)$ for some weighting rule.
The newly labeled area is now defined as $B_r(\lf_j) = \{X_i \notin \supp(\lf_j): \exists \; X_k \in \supp(\lf_j) \; \text{s.t.} \; \dist(X_i, X_k) \le r_j \}$, while in other areas we maintain the abstention.
There are multiple choices for the weighting rule, some of which provide strong theoretical guarantees as described in in Section~\ref{sec:theory}.
In practice, we observe that simple rules such as $1$-nearest neighbor yield excellent performance.
Afterwards, we have extended labeling functions $\bar{\lf}_i$, and we can use
any WS technique that uses the interface described in the background.\footnote{In the appendix, we review in more detail the choice used in our implementation and experimental setup.}
This procedure is summarized in Algorithm \ref{alg:ext}.


\section{Theoretical Analysis}
\label{sec:theory}

Now we analyze our method for improving label model generalization. We
bound the generalization lift and rate improvement of our method over standard WS in terms of the Lipschitzness of our labels. Next, we connect label Lipschitzness with the intrinsic quality of the pre-trained model used for the embeddings. Finally, we bound the generalization gap between \sysx and an end model trained on the embeddings, and we describe the conditions under which our method outperforms transfer learning (TL).
We present our proofs and
synthetic experiments in the Appendix. 

\textbf{Comparison to standard WS.}
We characterize the minimum improvement and convergence rate of generalization error for \sysx versus WS label model without source extensions
in terms of two factors: accuracy and coverage.
For simplicity, we use
a single class-balanced task $Y \in \{-1, 1\}$. We assume that $Y$ is $(L_Y, M_Y)$-Lipschitz, and each labeling function (LF) is $(L_{\lf_i'}, M_{\lf_i'})$-Lipschitz. We consider the case of extending one LF $\lf_i$ whose accuracy is uniform on its support. The analysis for more general cases is found in the Appendix. 

The first quantity of interest is the \textit{accuracy} of an extended labeling function. Define the accuracy of the original $\lf_i$ to be $a_i := \Pr(\lf_i = Y | \lf_i \neq 0)$, $p_i = \Pr(X \in \supp(\lf_i))$ for the support of $\lf_i$, and $p_{d(r)} = \Pr(\dist(X, X') \le r)$. We can express the \textit{extended accuracy} $\bar{a}_i$  as a function of ${a}_i$. 

\begin{proposition}
Denote the extended accuracy of $\bar{\lf}_i$ with threshold radius $r$ to be $\bar{a}_i := \Pr(\bar{\lf}_i = Y | \bar{\lf}_i \neq 0)$. There exists a weighting rule such that
\begin{align*}
\bar{a}_i \ge a_i - (2a_i - 1) \frac{M_Y(r)}{p_i^2 (1 + L_{\lf_i'}(r) p_{d(r)})}.
\end{align*}
\label{prop:acc}
\end{proposition}

The accuracy of the extended LF is dependent on the Lipschitzness of the task and the original LF accuracy. Assuming that $a_i > 0.5$, i.e. the LF votes better than random, the extended accuracy roughly gets worse as the ratio $M_Y(r) / L_{\lf_i'}(r)$ increases, as we would expect.

The other quantity of interest is how much we extend the \textit{coverage} of LFs. As coverage increases, there are two effects. First, more points are labeled, resulting in potentially greater generalization lift. Second, the supports of extended labeling functions will overlap more, resulting in better estimation of accuracy parameters for the WS label model. Recall our definition of $o_{\min}$. Define $L_{\min} = \min_{\lf_i', r_i} L_{\lf_i'}(r_i)$ and $r_{\min} = \min_i r_i$. The increase in the minimum overlap of supports $o_{\min}$ after extending depends on $L_{\min}$ and $r_{\min}$, which thus controls estimation error.

We can now compare the generalization error of \sysx versus standard WS in terms of rates and asymptotic risk. Define the risk of a model $f$ as $\mathcal{R}(f) = \E{(X, Y) \sim \mathcal{D}}{\ell(f(X), Y)}$, where $\ell$ is a loss function. In WS, the output is a probabilistic label $\widetilde{Y}$ estimated using our LFs $\bm{\lf}$. The label model risk is $\mathcal{R}(f_{WS})$, where 
$    f_{WS}(X) = 2 \Pr(\widetilde{Y} = 1 | \bm{\lf} = \lf_1(X), \dots \lf_m(X)) - 1$.
We have similar definitions using $\bm{\bar{\lf}}$ for $\mathcal{R}(f_E)$ for \sysx and we take the loss for both models to be $\ell (Y', Y) = \frac{1}{2} |Y - Y'|$. Next, $c_p = \min_X \Pr(\bm{\lf} = \bm{\lf}(X))$, $e_{\min}$ is a lower bound on $\E{}{\lf_i Y | \lf_i \neq 0}$, $c_1$ is a lower bound on $\E{}{\lf_i \lf_j | \lf_i, \lf_j \neq 0}$, and $c_2 = \E{}{\Pr(Y = 1 | \bm{\lf} = \bm{\lf}(X))}$. Similar definitions hold for $\bar{c}_p$ and other terms using $\bar{\lf}$.

\begin{theorem}
Define $\varepsilon_n = \sqrt{\frac{\log(2/\delta)}{2n}}$. With probability $\ge 1 - \delta$, the estimation error of \sysx is
\begin{align*}
\E{}{|\R(f_{E}) - \R(\hat{f}_{E})|} \le \frac{1}{\bar{c}_p - \varepsilon_n} \Bigg(\frac{81 \sqrt{\pi}}{2 \bar{e}_{\min} \bar{c}_1^2} \cdot \frac{m}{\sqrt{n \cdot o_{\min} (1 + (2L_{\min} - L_{\min}^2) p_{d(r_{\min})})}} + \varepsilon_n \bar{c}_2 \Bigg),
\end{align*}

while for WS it is
\begin{align*}
\E{}{|\R(f_{WS}) - \R(\hat{f}_{WS})|} \le \frac{1}{c_p - \varepsilon_n} \left( \frac{81 \sqrt{\pi}}{2 e_{\min} c_1^2 } \cdot \frac{m}{\sqrt{n \cdot o_{\min}}} + \varepsilon_n  c_2 \right).
\end{align*}

When $\lf_i$ is extended using threshold radius $r$ and the same weighting rule as in Prop.~\ref{prop:acc}, the asymptotic improvement in generalization error is at least
\begin{align*}
\mathcal{R}(f_{WS}) - \mathcal{R}(f_E) &\ge L_{\lf_i'}(r) p_{d(r)} p_i \left(\frac{1}{2}(C + 1) (\widetilde{a}_i \bar{a}_i + (1 - \widetilde{a}_i)(1 - \bar{a}_i)) - C \right),
\end{align*}
where $\widetilde{a}_i \ge \bar{a}_i - (2a_i - 1) \frac{M_Y(r)}{L_{\lf_i'}(r) p_{d(r)} p_i^2 (1 + L_{\lf_i'}(r)p_{d(r)})}$ is the accuracy of $\bar{\lf}_i$ over $B_r(\lf_i)$, and $C$ is a constant lower bounding the performance of the other labeling functions.

\label{thm:risk_gap}
\end{theorem}

The convergence rate of \sysx has asymptotic improvement of $\sqrt{1 + (2L_{\min} - L_{\min}^2) p_{d(r_{\min})}}$ over standard WS,
which increases
as the threshold radius increases. In the asymptotic case, we make the following observations about the generalization lift:
\begin{itemize}[itemsep=0.5pt,topsep=0pt]
\item The lift increases as the original accuracy $a_i$ is increased, suggesting that the most accurate labeling functions should be extended.
\item Given a fixed threshold radius, the lift decreases in $M_Y(r)$, suggesting that our extension method performs more poorly on less smooth embeddings.
\item The lift improves as $r$ increases due to $L_{\lf_i'}(r)$ (more likely to extend to areas outside of support) and $p_{d(r)}$ (more points voted on), but the accuracies $\tilde{a}_i$ and $\bar{a}_i$ also become worse. 
\end{itemize}

These insights imply that \emph{there is a trade off between accuracy and coverage}. 
We illustrate this behavior and show how to set an approximate optimal radius in synthetic experiments (Appendix).

\textbf{Comparison to models trained using the embeddings.}
Theorem~\ref{thm:risk_gap} specifies the lift in terms of the Lipschitzness. How do we know what kind of Lipschitzness our embeddings-based approach will yield? We show the connection via a simple proposition. We write $f_z$ for a model using the embeddings $\mathcal{Z}$. Below we implicitly take $\ell(f(X), Y)$ to be the $0/1$ loss. Then,

\begin{proposition}
\label{thm:prop2}
Suppose we have an embedding function $g$ achieving risk $\R(f_z)$. Suppose also that the model class $f_z$ satisfies $\Pr_{X_1,X_2}(f_z(X_1) \neq f_z(X_2) | \dist(X_1, X_2) \leq r)  \leq M_{f_z}(r)$. Then, the labels $Y$ have smoothness given by 
\[
 \Pr_{X_1,X_2}(Y_1 \neq Y_2 | \dist(X_1,X_2) \leq r) \leq M_{f_z}(r) + 2\R(f_z).
\]
\end{proposition}
This simple bound on smoothness inherited from an embedding-based model plugs in as the smoothness in Theorem~\ref{thm:risk_gap}. The key requirement is that the distance used to extend, $d_{\mathcal{Z}}$, is the same as that providing model smoothness (the $M_Y(r)$ above).
Proposition~\ref{thm:prop2} both explains when and why \sysx produces a quality label model, and motivates its use (as opposed to simply extending one source). 

An easy argument bounds the risk for a source $\bar{\lf}_i$ with original accuracy $a_i$ extended to cover all of $\mathcal{X}$. Suppose an embedding $Z = g(X)$ achieves a risk of $\R(f_z)$, and that the smallest radius needed to cover all of $\mathcal{X}$ when extending $\lf_i$ is $d_V$. Then its risk is upper bounded as \[\mathcal{R}(\bar{\lf_i}) \leq  1 - a_i  + (2a_i - 1) \frac{M_{f_z}(d_V) +2\R(f_z)}{p_i^2(1 + L_{\lf_i'}(d_V) p_{d(d_V)})}.\] 

We note some implications: if $\lf_i$ is a low-quality source, with accuracy $a_i = 1/2$, then $\mathcal{R}(\bar{\lf_i}) \leq 1$, which says nothing about the extension, as signal from the embeddings is washed out by the poor-quality source. If $a_i = 1$, our source is perfect on its support. Then the risk is a constant times the risk of the embedding-based model, with an additive term and scaling coming from the model smoothness and the radius. If we have a large representative support, driving $d_V \rightarrow 0$ and thus $M_{f_z}(d_V), L_{\lf_i'}(d_V) \rightarrow 0$, the risk of the extended source is $2\R(f_z)$. This reproduces a learned model, without training---but such cases are rare,
\emph{motivating extensions of varying radii}.

\textbf{Comparison to end models.}
While we may not be able to match the performance of a fully-trained deep neural network as an end model, we hope to approach it, and potentially even outperform more limited models. The next result uses the machinery built so far to derive conditions for the latter.

The intuition is to compare a model $f_z$ from a class with fixed capacity trained over data from $\mathcal{D}$ with an ensemble of $m$ \emph{specialized} models $f_i$ each specialized to a distribution $\mathcal{D}_i$ defined on a subspace $\mathcal{X}_i$. Intuitively, the risk of a specialized model, $\mathcal{R}(f_i)$, can be lower than $\mathcal{R}(f_z)$ when restricted to $\mathcal{X}_i$. When this happens, we expect the ensemble model's risk to be superior.

We can relate extended sources to specialized models with Proposition~\ref{thm:prop2},
and use \sysx\ instead of emsembling the sources, where each $\supp(\bar{\lf}_i) \supseteq \mathcal{X}_i$.

\begin{theorem}
  \label{thm:thm2}
Let $f_i$, $1 \leq i \leq m$ be defined as above. Define $p = \Pr(Y = 1)$, the max odds to be $b = \max \{ \frac{p}{1-p}, \frac{1-p}{p} \}$, and $\mathcal{X}_0 = \{X: X \notin \bigcup_{i = 1}^m \supp(\lf_i)\}$. Then,
\begin{align*}
\R(f_E) \le 2b \cdot \E{\mathcal{D}_i}{1 - a_i + (2a_i - 1)\frac{M_{f_i}(r_i) + 2\R(f_i)}{p_i^2 (1 + L_{\lf_i'}(r_i) p_{d(r_i)})} } + 2 \Pr(X \in \mathcal{X}_0) p(1 - p).
\end{align*}
\label{thm:transfer_learning}
\end{theorem} 

Consequently, setting the right-hand expression to be less than $\R(f_z)$ gives us a condition characterizing when \sysx beats an end model.
That is, as long as the (average) specialization offers improved performance and the loss when passing from the specialized model to the extended source is smaller than this improvement, \sysx outperforms an end model. As we shall see, this does not typically occur for a fully-trained deep model, but does for the TL without fine-tuning.


\section{Evaluation}
\label{sec:evaluation}

\begin{table*}[t]
    \centering
    \scriptsize
    \begin{tabular}{@{}rlccccccccccccc@{}}
    \toprule
    \multicolumn{3}{c}{}                                 & \multicolumn{3}{c}{\textbf{Interactive-Speed Baselines}} &\multicolumn{3}{c}{\textbf{Fully-Trained End Models}}    \\
    \cmidrule(l){4-6} \cmidrule(l){7-9}                                                              
    & \textbf{Task} (Embedding) & \textbf{\sysx}   & \textbf{TS-NFT} & \textbf{WS-NFT} & \textbf{WS-LM}      &\textbf{TS-FT}&\textbf{WS-FT}&\textbf{\sysx-FT} \\
    \midrule                                                                                                            
    \parbox[t]{0mm}{\multirow{3}{*}{\rotatebox[origin=c]{90}{\textbf{NLP}}}}
    & \spam\ (BERT)             & \underline{89.6} & 78.5           & 87.2           & 83.6                & 76.7       & 90.0         & \textbf{94.1}      \\
    & \weather\ (BERT)          & \underline{82.0} & 71.6           & 74.4           & 78.0                & 71.2       & 85.6         & \textbf{86.8}      \\
    & \spouse\ (BERT)           & \underline{48.4} & 17.7           & 17.5           & 47.0                & 20.4       & 49.6         & \textbf{51.3}      \\
    \midrule                                                                                                                                    
    \parbox[t]{0mm}{\multirow{6}{*}{\rotatebox[origin=c]{90}{\textbf{Video}}}}
    & \basketball\ (RN-101)     & \underline{31.3} & 18.8           & 16.8           & 27.9                & 26.8       & 35.8         & \textbf{36.7}      \\
    & \commercial\ (RN-101)     & \underline{90.1} & 73.6           & 75.5           & 88.4                & 90.9       & 92.5         & \textbf{93.0}      \\
    & \tennis\ (RN-101)         & \underline{82.9} & 76.7           & 79.5           & 82.0                & 57.6       & 82.9         & \textbf{83.1}      \\
    \cmidrule(l){2-9}                                                                         
    & \basketball\ (BiT-M)      & \underline{42.5} & 22.8           & 23.2           & 27.9                & 29.1       & 33.8         & \textbf{45.8}      \\
    & \commercial\ (BiT-M)      & \underline{91.8} & 71.7           & 73.8           & 88.4                & 93.2       & 93.7         & \textbf{94.4}      \\
    & \tennis\ (BiT-M)          & \underline{83.1} & 75.5           & 79.0           & 82.0                & 47.5       & 83.7         & \textbf{83.8}      \\
    \midrule                                                                                                            
    \multicolumn{2}{c}{Average Training Time (s)}         
                                & 0.5              & 1.2            & 5.7            & 0.1                 & 1,243.0    & 5,354.1      & 4,995.0            \\
    \bottomrule
    \end{tabular}
    \caption{\sysx\ performance and training time compared to baselines.
    Scores are F1 except for Spam and Weather (accuracy)
    and are averaged across five random seeds; best score in bold, best
    interactive-speed model underlined.
    Interactive-speed baselines include transfer learning without fine-tuning
    using traditional hand labels (TS-NFT) and WS probabilistic
    labels (WS-NFT), as well
    as WS label model (WS-LM).
    Fully-trained end models use traditional supervision (TS-FT),
    WS labels (WS-FT), and \sysx-generated labels (\sysx-FT).
    }
    \vspace{-1em}
    \label{table:main_results}
\end{table*}

We evaluate \sysx\ on benchmark WS tasks.
We compare \sysx's performance
against
WS without label extension,
against
TL without fine-tuning,
and
against fully-trained deep networks,
discuss ablations on optimizing $r_i$,
and evaluate runtime.

\ifarxiv
\textbf{Datasets.}
\fi
We evaluate \sysx\ on six benchmark NLP and video analysis tasks that have been
used to evaluate previous state-of-the-art weak supervision
systems~\cite{fu2020fast, Ratner18}.
For the text datasets, we use pre-trained BERT embeddings for label extension;
for the video datasets, we use ResNet-101 pre-trained on ImageNet and BiT-M
with ResNet-50x1 backbone, a recently-released model designed for visual
transfer learning~\cite{kolesnikov2019large}.
We use cosine distance as the distance function for all tasks, and use 1-nearest
neighbor as the weighting rule.
\spam\ classifies whether YouTube comments are spam~\cite{alberto2015tubespam},
\weather\ classifies sentiment over Tweets about weather~\cite{CrowdflowerWeather},
and \spouse\ seeks to extract spouse relationships in a set of news
articles~\cite{corney2016million}.
\basketball\ identifies basketball videos from a subset of
ActivityNet~\cite{caba2015activitynet}, \commercial\ identifies commercials in
a corpus of TV News~\cite{fu2019rekall, InternetArchive}, and \tennis\
identifies tennis rallies from broadcast footage.
Each dataset consists of a large unlabeled training set, a
smaller hand-labeled \textit{development set} (train/dev split sizes from
\num{187/50} points to \num{64,130/9,479} points), and a held-out test set.
We use the unlabeled training set to train label models and end models, and use
the development set for a) training of traditional supervision baselines, and
b) hyperparameter tuning of the label and end models.

\ifarxiv
\textbf{Baselines.}
\fi
For each task, we evaluate \sysx\ against the WS label model
from~\cite{fu2020fast} without extensions (WS-LM) to evaluate how much lift our
extensions provide;
TL without fine-tuning, trained with both hand labels on the dev set
(TS-NFT) and probabilistic labels (WS-NFT) on the training set, to evaluate
against other models
that can train at interactive speeds;
and training deep networks with hand labels (TS-FT),
probabilistic labels (WS-FT), and labels generated by \sysx\ (\sysx-FT) to see
how close we can get to the performance of training a deep network (these
approaches often use TL with fine-tuning, see Appendix for details).
For TL without fine-tuning, we train a fully-connected layer
over the embeddings.
The Appendix contains details about the tasks and experimental setup,
error analysis of each approach,
and comparisons against other weak supervision frameworks~\cite{Ratner19}
and approaches to TL without fine-tuning.
Table~\ref{table:main_results} shows the performance and training time of
\sysx\ and these baselines on our datasets.
In the subsequent sections, we compare \sysx\ performance against baselines
and explain the differences with explanatory metrics
(Table~\ref{table:embedding_utility}).


\begin{table*}[t]
    \centering
    \tiny
    \begin{tabular}{@{}rcccccccccccccc@{}}
        \toprule
        \multicolumn{1}{c}{} & \multicolumn{3}{c}{BERT} & \multicolumn{3}{c}{RN-101 ImageNet} & \multicolumn{3}{c}{BiT-M RN-50x1} \\
        \cmidrule(l){2-4} \cmidrule(l){5-7} \cmidrule(l){8-10}
        & \spam & \weather & \spouse & \basketball & \commercial & \tennis & \basketball & \commercial & \tennis \\
        \textbf{WS-LM} average LF coverage
        & 16.2  & 8.8      & 3.8     & 49.2        & 54.5        & 66.9    & 49.2        & 54.5        & 66.9 \\
        \textbf{\sysx} average LF coverage
        & 22.9  & 9.0      & 7.7     & 63.4        & 80.8        & 72.5    & 57.0        & 84.7        & 74.7 \\
        $1 - d(PT, DN)$
        & 0.735 & 0.677    & 0.248   & 0.562       & 0.284       & 0.113   & 0.907       & 0.559       & 0.408\\
        \cmidrule(l){2-10}
        \textbf{\sysx\ vs. WS-LM}
        & +6.0  & +4.0     & +1.4    & +3.4        & +1.7        & +0.9    & +14.6       & +3.4        & +1.1 \\
        \midrule
        Average LF accuracy
        & 82.8  & 75.4     & 58.6    & 59.3        & 92.3        & 81.9    & 59.3        & 92.3        & 81.9 \\
        LF accuracy vs. \textbf{TS-NFT}
        & +4.3  & +3.8     & +41.1   & +42.5       & +18.7       & +5.2    & +36.1       & +20.6       & +6.4 \\
        LF accuracy vs. \textbf{WS-NFT}
        & -4.4  & +1.0     & +40.9   & +40.5       & +16.8       & +2.4    & +36.5       & +18.5       & +2.9 \\
        \cmidrule(l){2-10}
        \textbf{\sysx\ vs. TS-NFT}
        &+11.1  & +10.4    & +30.7   & +12.5       & +16.5       & +6.2    & +19.7       & +20.1       & +7.6 \\
        \textbf{\sysx\ vs. WS-NFT}
        & +2.4  & +7.6     & +30.9   & +14.5       & +14.6       & +3.4    & +25.7       & +16.3       & +3.6 \\
        \midrule
        $N_{train} / N_{dev}$
        & 13.2  & 3.7      & 7.9     & 17.0        & 6.8         & 9.3     & 17.0        & 6.8         & 9.3  \\
        \cmidrule(l){2-10}
        \textbf{\sysx\ vs. TS-FT}
        & +12.9 & +10.8    & +28.0   & +4.5        & -0.8        & +25.3   & +13.4       & -1.4        & +25.0 \\
        \textbf{\sysx\ vs. WS-FT}
        & -0.4  & -3.6     & -1.2    & -4.5        & -2.4        & +0.0    & +8.7        & -1.9        & -0.6 \\
        \bottomrule
    \end{tabular}
    \caption{\sysx\ lift compared to baselines with explanatory metrics.
    Top:
    \sysx\ improves over WS-LM by improving average coverage of the labeling
    functions.
    Lift is correlated with the quantity $1-d(PT, DN)$, the
    similarity between pre-trained and fine-tuned embeddings for
    each task according to distance.
    Middle: 
    \sysx\ outperforms transfer learning without fine-tuning because the
    labeling sources are accurate on average over their support sets.
    Bottom:
    \sysx\ outperforms training a deep network over a
    small collection of hand labels (TS-FT) because it uses more data, 
    and approaches the performance of training a deep network over
    probabilistic labels (WS-FT).
    }
    \label{table:embedding_utility}
    \vspace{-3em}
\end{table*}

\ifarxiv
\subsection{Comparison Against WS Label Model}
\fi
We compare the performance of \sysx\ with the WS label model
without label extensions.
Table~\ref{table:embedding_utility} (top) shows the lift of \sysx\
compared to WS-LM, along with average per-labeling function coverage
before and after extensions.
\sysx\ outperforms WS-LM on all six tasks because it increases
labeling function coverage, allowing
\sysx\ to generate accurate predictions on more points (critically, improving recall, discussed in further detail in Appendix).

We also report the quantity $1 - d(PT, DN)$, which is the average cosine similarity between
pre-trained embeddings and fine-tuned embeddings.
This quantity measures how similar the pre-trained embeddings are to the fine-tuned
embeddings according to our distance function, which helps explain lift of
\sysx\ over WS-LM.
We find that there is a strong correlation between this metric and lift
($R^2 = 0.692$), helping to explain how well our method works.
When this metric is high, we can extract more useful information out of the
pre-trained embeddings via distance functions, as we do in \sysx.
Although this metric is informative for exploiting embeddings via distances, we
note that it is not necessarily informative for predicting performance of using
the embeddings as features without fine-tuning.
Fine-tuning can adjust the embeddings in ways that result in
large changes to performance without affecting distance; we measure this effect in the Appendix.

\ifarxiv
\subsection{Comparison Against TL Without Fine-Tuning}
\fi
We now compare the performance of \sysx\ against TL without
fine-tuning.
Unlike fully-trained deep networks, both approaches can train at
programmatically-interactive speeds (less than $10$ seconds on average, see
Table~\ref{table:main_results} bottom).
Table~\ref{table:embedding_utility} (middle) shows the relative lift of \sysx\
compared to TS-NFT and WS-NFT
along with the average accuracy of labeling functions on their support sets
(formally, average $P(\lambda_j(X_i) = Y_i | X_i \in \supp (\lambda_j))$ for
$j \in {1, \cdots, m}$).
\sysx\ can outperform non-finetuned models because the labeling
functions are accurate on their support sets, whereas the limited-capacity TL 
models need to learn a global mapping
from the pre-trained embeddings to the class label.
As Theorem~\ref{thm:thm2} suggests, this local specialization allows \sysx\ to
have higher overall performance.

\ifarxiv
\subsection{Comparison Against Fully-Trained Deep Networks}
\fi
We now compare the performance of \sysx\ against fully-trained deep networks
supervised both with traditional hand labels (TS-FT) and probabilistic labels
(WS-FT).
Table~\ref{table:embedding_utility} (bottom) shows the relative lift of \sysx\
compared to these baselines, along with the ratio of the size of the unlabeled
training set to the labeled dev set.
\sysx\ can outperform the traditionally-supervised end model in many cases,
since it has access to much more data (up to 17 times, labeled automatically).
Meanwhile, it approaches the performance of a weakly-supervised end model,
coming within one point of WS-FT in two tasks and outperforming in one case.
We note that this lift can further improve the performance of
training a deep network by using \sysx\ to generate weak labels for an end model
(reported as \sysx-FT in
Table~\ref{table:main_results}).

\ifarxiv
\subsection{Ablations}
\fi
We summarize the results of ablations on properly optimizing the extension
thresholds $r_i$ (details in Appendix).
First, we use a single threshold $r$ for all labeling functions,
instead of setting different thresholds for each labeling function.
This results in a performance
degradation in the \sysx\ label model of \avgliftoverfixedr\ points on average.
We also study performance from sweeping $r_i$ and find that, for
each task, performance improves as we increase $r_i$ from $0$, but
degrades if the threshold is too large, demonstrating the accuracy-coverage trade
off.

\ifarxiv
\subsection{Runtime Evaluation}
\fi
We measure \sysx's training time compared to other methods.
All timing measurements were taken on a machine with an Intel Xeon E5-2690 v4
CPU with a Tesla P100-PCIE-16GB GPU.
Table~\ref{table:main_results} (bottom row) reports the average training time in seconds of
each method.
\sysx\ does not train a deep network, so
it can run
in less than \num{half a second} on average---enabling
programmatically-interactive re-training cycles.
In contrast, deep end models are much more expensive, requiring hours to train
on average.
The non-finetuned approaches (*-NFT) are much less expensive and train in
seconds, but achieve lower performance scores than \sysx.
Both \sysx\ and the non-finetuned transfer learning approaches require a
one-time upfront cost for deep network inference to compute embeddings (exact timings in Appendix),
which takes less than \num{three minutes} on
average.


\section{Related Work}
\label{sec:related}

\textbf{Weak supervision} is a broad set of techniques using weak sources of
signal to supervise models, such as
distant supervision~\cite{mintz2009distant,craven1999constructing,
hoffmann2011knowledge,takamatsu2012reducing}, co-training
methods~\cite{blum1998combining}, pattern-based
supervision~\cite{gupta2014improved} and feature
annotation~\cite{mann2010generalized,zaidan2008modeling,liang2009learning}.
Recently, weak supervision frameworks have systematically integrated multiple
noisy sources in two stages---first using a latent variable model to
de-noise source votes, and then using a powerful end model to improve
performance~\cite{Ratner16, Ratner18, bach2017learning, bach2018snorkel,
guan2018said, khetan2017learning, sheng2020gmail, re2019overton, fu2020fast,
zhan2019sequentialws, sala2019multiresws, safranchik2020weakly, boecking2019pairwise}.
Our work focuses on removing this second stage from the pipeline to enable
faster iteration cycles.

\textbf{Transfer learning} uses large datasets to learn useful feature
representations that can be fine-tuned for downstream tasks~\cite{kolesnikov2019large,
devlin2018bert, sun2017revisiting, raffel2019exploring}.
Transfer learning techniques for text applications typically pre-train on large
corpora of unlabeled data~\cite{devlin2018bert,
yang2019xlnet, dong2019unified, liu2019roberta, lan2019albert,
brown2020language, radford2019language}, while common
applications of transfer learning to computer
vision pre-train on both large supervised datasets such as
ImageNet~\cite{deng2009imagenet, russakovsky2015imagenet} and large
unsupervised or weakly-supervised datasets~\cite{joulin2016learning,
li2017learning, thomee2016yfcc100m, sun2017revisiting, mahajan2018exploring,
he2019momentum}.
Pre-trained embeddings have also been used as data point descriptors for 
similarity-based search algorithms, such as KNN search, to improve model
performance~\cite{khandelwal2019generalization, papernot2018deep,
orhan2018simple, zhao2018retrieval}.
We view our work as complementary to these approaches, presenting another
mechanism for using pre-trained networks.

\textbf{Semi-supervised and few-shot learning} approaches aim to learn good
models for downstream tasks given a few labeled examples.
Semi-supervised approaches like label propagation~\cite{zhur2002learning,
iscen2019label} start from a few labeled examples and iteratively fine-tune
representations on progressively larger datasets, while few-shot learning
approaches such as meta-learning and metric learning aim to build networks that
can be directly trained with a few labels~\cite{vinyals2016matching,
snell2017prototypical, sung2018learning}.
Our work is inspired by these approaches for expanding signal from a subset of
the data to the entire dataset using deep representations, but we do not assume
that our labeling sources are perfect, and we do not tune the representation.


\section{Conclusion}
\label{sec:conc}
We study when it is possible to build weakly-supervised models without training
a downstream deep network by using pre-trained embeddings to extend noisy
labeling sources.
We develop theoretical results characterizing how much information we can
extract from distances between data points based on the probabilistic
Lipschitzness of the embedding space, and empirically validate our method on
six benchmark applications.
As pre-trained networks grow increasingly larger and more powerful, we hope
that our work inspires a variety of approaches to using pre-trained networks
beyond fine-tuning.

\ifarxiv
\section*{Acknowledgments}

We gratefully acknowledge the support of DARPA under Nos. FA86501827865 (SDH) and FA86501827882 (ASED); NIH under No. U54EB020405 (Mobilize), NSF under Nos. CCF1763315 (Beyond Sparsity), CCF1563078 (Volume to Velocity), and 1937301 (RTML); ONR under No. N000141712266 (Unifying Weak Supervision); the Moore Foundation, NXP, Xilinx, LETI-CEA, Intel, IBM, Microsoft, NEC, Toshiba, TSMC, ARM, Hitachi, BASF, Accenture, Ericsson, Qualcomm, Analog Devices, the Okawa Foundation, American Family Insurance, Google Cloud, Swiss Re,
Brown Institute for Media Innovation, the HAI-AWS Cloud Credits for Research program,
Department of Defense (DoD) through the National Defense Science and
Engineering Graduate Fellowship (NDSEG) Program, 
and members of the Stanford DAWN project: Teradata, Facebook, Google, Ant Financial, NEC, VMWare, and Infosys. The U.S. Government is authorized to reproduce and distribute reprints for Governmental purposes notwithstanding any copyright notation thereon. Any opinions, findings, and conclusions or recommendations expressed in this material are those of the authors and do not necessarily reflect the views, policies, or endorsements, either expressed or implied, of DARPA, NIH, ONR, or the U.S. Government.

\fi

\bibliographystyle{plain}
\bibliography{main}

\newpage

\appendix


\section*{Appendix}
First, we provide a glossary of terms and notation that we use throughout this
paper for easy summary (Section~\ref{sec:gloss}).
Next, we give details about the label model we use after extending our labeling
functions (Section~\ref{sec:supp_alg}).
Next, we give the proofs of each theorem (Section~\ref{sec:proofs}).
Then we give additional experimental details (Section~\ref{sec:supp_details})
and present further evaluation (Section~\ref{sec:supp_eval}).

\section{Glossary}
\label{sec:gloss}

The glossary is given in Table~\ref{table:glossary} below.
\begin{table*}[h]
\centering
\small
\begin{tabular}{l l}
\toprule
Symbol & Used for \\
\midrule
$X$ & Unlabeled data vector, $X = [X^1, X^2, \ldots, X^T] \in \mathcal{X}$. \\
$Y$ & Latent, ground-truth task label vector $[Y^1, Y^2, \ldots, Y^T] \in \mathcal{Y} = \{-1, +1\}$.  \\
$\mathcal{D}$ & Distribution from which $(X, Y)$ data points are sampled i.i.d. \\
$m$ & Number of weak supervision labeling functions. \\
$\lf_i$ & Labeling function $\lf_i: \mathcal{X} \rightarrow \mathcal{Y} \cup \{0\}$; all $m$ labels per $X$ collectively denoted $\bm{\lf}$. \\
$o_{\min}$ & The minimal overlap between pairs of labeling functions used for parameter \\
& recovery, $\min_i \max_{j: j \neq i} \Pr(X \in \supp(\lf_i) \cap \supp(\lf_j))$.\\
$D$ & Unlabeled dataset $\{X_i\}_{i = 1}^n$ that $\bm{\lf}$ is applied to in order to produce labels. \\
$n$ & Number of data vectors. \\
$g$ & Mapping from $\mathcal{X}$ to embedding space $\mathcal{Z}$. \\
$\mathcal{Z}$ & The embedding space corresponding to the pre-trained network. \\
$f_z$ & A classifier $f_z: \mathcal{Z} \rightarrow \mathcal{Y}$ that uses the embeddings in $\mathcal{Z}$ (i.e. transfer learning \\
& without fine-tuning). \\
$\dist$ & A given distance function on the embedding space $\mathcal{Z}$. \\
$M_Y(r)$ & An upper bound on $\Pr(Y_1 \neq Y_2 | \dist(X_1, X_2) \le r)$. \\
$L_{\lf_i'}(r)$ & A lower bound on $\Pr(\lf_i(X_1) = 0, \lf_i(X_2) \neq 0 | \dist(X_1, X_2) \le r)$. \\
$\bm{\bar{\lf}}$ & The set of extended labeling functions. \\
$\bm{r}$ & Threshold radii, where each $r_i$ specifies how much to extend $\lf_i$. \\
$N_{r_j}(X)$ & The points in the support of $\lf_j$ that are within $r_j$ of $X$. \\
$W$ & Weighing rule $W(N_{r_j}(X))$ used to assign labels to extended points based on \\
& neighbors in $\supp(\lf_j)$. \\
$B_r(\lf_j)$ & The set of newly labeled points $\{X_i \notin \supp(\lf_j): \; \exists X_k \in \supp(\lf_j) \; \text{s.t.} \; \dist(X_i, X_k) \le r_j \}$. \\
$a_i$ & Accuracy of original labeling function $\Pr(\lf_i = Y | \lf_i \neq 0)$. \\
$\bar{a}_i$ & Accuracy of extended labeling function $\Pr(\bar{\lf}_i = Y | \bar{\lf}_i \neq 0)$. \\
$p_i$ & Relative size of $\lf_i$'s support set, $\Pr(X \in \supp(\lf_i))$. \\
$p_{d(r)}$ & Proportion of points within $r$ of each other, $\Pr(\dist(X, X') \le r)$. \\
$L_{\min}$ & Minimum value of $L_{\lf_i'}(r_i)$ across all $i = 1, \dots, m$. \\
$r_{\min}$ & Smallest extension radius $r_{\min} = \min_i r_i$. \\
\toprule
\end{tabular}
\caption{
	Glossary of variables and symbols used in this paper.
}
\label{table:glossary}
\end{table*}


\section{Additional Algorithmic Details}
\label{sec:supp_alg}

We present an overview of the weak supervision model used after extending our labeling functions via Algorithm \ref{alg:ext} (Section \ref{subsec:pgm}). Then, we discuss some of the basic properties of this model that are important for our theoretical results (Section \ref{subsec:pgm_properties}). Since this model applies to both the extended labeling functions and original labeling functions, we use $\bm{\lf}$ for simplicity in this section.

\subsection{Probabilistic label model}
\label{subsec:pgm}

Recall that the goal of weak supervision is to combine noisy sources $\bm{\lf}$ on an unlabeled dataset to produce probabilistic labels $\widetilde{Y}$. We review the modeling and parameter recovery method presented in \cite{fu2020fast}. We first describe our choice of probabilistic graphical model based on the extended labeling functions. Then, we discuss how to recover parameters of the graphical model. Finally, we explain how to perform inference with the recovered parameters to generate probabilistic labels on the data.

\paragraph{Binary Ising Model} In the standard weak supervision setting, we have a set of labeling functions $\bm{\lf} = \lf_1, \dots \lf_m$ that vote or abstain on $Y = [Y^1, \dots Y^T]$. Each labeling function $\lf_i$ produces a vote on one element of $Y$, which we denote as $\ydep{i}$. Therefore, we can view $\bm{\lf}$ and $Y$ as random variables that are functions of $X$, where the labeling functions are independent of each other conditioned on $Y$. Let the graph $\gdep = (\{Y, \bm{\lf}\}, E_{\text{dep}})$ specify the dependencies between the noisy labeling functions and the task labels, using standard technical notions from the PGM literature~\cite{koller2009probabilistic, Lauritzen, wainwright2008graphical}. In particular, the lack of an edge in $\gdep$ between a pair of variables indicates independence conditioned on a separator set of variables~\cite{Lauritzen}, and there exists an edge between each $\lf_i$ and $\ydep{i}$. We assume that $\gdep$ is user-provided, although it can be estimated directly from the votes of the labeling functions~\cite{Varma19}. We also assume that we are provided with the class balance prior $\Pr(Y)$, although this can also be learned from the data~\cite{Ratner19}.

We augment the dependency graph $\gdep$ to produce a binary Ising model. To incorporate behavior when sources abstain, we represent the outputs of each labeling function $\lf_i$ using a pair of binary random variables, $(v_i^1, v_i^{-1}) \in \{-1, +1\}$. More formally, when $\lf_i = 1$, we set $(v_i^1, v_i^{-1}) = (1, -1)$; when $\lf_i = -1$, we set them to $(-1, 1)$, and when $\lf_i = 0$, we set them with equal probability to $(1, 1)$ or $(-1, -1)$. This allows us to encode abstaining behavior with $v_i^1 v_i^{-1} = 1$ implying $\lf_i = 0$. We thus have edges from $v_i^1$ and $v_i^{-1}$ to $\ydep{i}$, as well as an edge between $v_i^1$ and $v_i^{-1}$. Using this augmentation, we have a new graph $G = (V, E)$ based on $\gdep$, where $V = \{Y, \bm{v} \}$, that follows a binary Ising model. The joint distribution among $Y$ and $\bm{v}$ is 
\begin{align}
f_G(Y, \bm{v}) = \frac{1}{Z} \exp \bigg(&\sum_{k = 1}^T \theta_{Y_k} Y_k + \sum_{(Y_k, Y_l) \in E} \theta_{Y_k, Y_l} Y_k Y_l + \nonumber \\
&\sum_{i = 1}^m \theta_i (v_i^1 - v_i^{-1}) \ydep{i} + \sum_{i = 1}^m \theta_{i, i} v_i^1 v_i^{-1} \bigg)
\label{eq:pgm}
\end{align}

where $Z$ corresponds to the cumulant function and $\theta$ parametrizes this density function. In this graphical model, $\theta_{Y_k}$ and $\theta_{Y_k, Y_l}$ are parameters corresponding to the prior $\Pr(Y)$. $\theta_i$ parametrizes the edges between the labeling functions and the task label, hence corresponding to the accuracy of the labeling function. $\theta_{i, i}$ between each $v_i^1, v_i^{-1}$ represents the abstain rate.

\paragraph{Parameter Recovery}
We now explain how to recover the unknown accuracy parameters of this graphical model using an efficient method-of-moments based approach from \cite{fu2020fast}. In particular, we use the following independence property:
\begin{proposition}
$\lf_i \ydep{i} \independent \lf_j \ydep{i} | \lf_i, \lf_i \neq 0$ for all $i \neq j \in 1, \dots, m$.
\end{proposition}

\begin{proof}
To see how this holds true, denote $S = \{Y^k \in Y: \lf_i \independent \lf_j | Y^k \}$, e.g. $Y^k$ is a variable on the path from $\lf_i$ to $\lf_j$ in $\gdep$. Then, marginalizing over $V \backslash \{S \cup (v_i^1, v_i^{-1}), (v_j^1, v_j^{-1}) \}$ and conditioning on $\lf_i, \lf_j \neq 0$, e.g. $v_i^1 \neq v_i^{-1}, v_j^1 \neq v_j^{-1}$, we use \eqref{eq:pgm} to get
\begin{align*}
& \Pr(S, (v_i^1, v_i^{-1}), (v_j^1, v_j^{-1}) | v_i^1 \neq v_i^{-1}, v_j^1 \neq v_j^{-1}) = \\
&\frac{1}{Z_{ij}} \exp \bigg(\sum_{Y_k \in S} \theta_{Y_k} Y_k + \sum_{(Y_k, Y_l) \in S} \theta_{Y_k, Y_l} Y_k Y_l + \theta_i (v_i^1 - v_i^{-1}) \ydep{i} + \theta_j (v_j^1 - v_j^{-1}) \ydep{i} \bigg)
\end{align*}

for some different cumulant function $Z_{ij}$. Moreover, $v_i^1 - v_i^{-1}$ and $v_j^1 - v_j^{-1}$ are always equal to $-2$ or $2$ since this probability is conditioned on $\lf_i, \lf_j \neq 0$. Then, by setting $\theta_i' = 2\theta_i$ and $\theta_j' = 2\theta_j$, we can write this density function directly in terms of $\lf_i$ and $\lf_j$:
\begin{align*}
\Pr(S, \lf_i, \lf_j &| \lf_i, \lf_j \neq 0) = \\
&\frac{1}{Z_{ij}} \exp \bigg(\sum_{Y_k \in S} \theta_{Y_k} Y_k + \sum_{(Y_k, Y_l) \in S} \theta_{Y_k, Y_l} Y_k Y_l + \theta_i' \lf_i \ydep{i} + \theta_j' \lf_j \ydep{i} \bigg).
\end{align*}

Then, by Proposition $1$ of~\cite{fu2020fast}, we know that $\lf_i \ydep{i}$ and $\lf_j \ydep{i}$ are independent conditioned on $\lf_i, \lf_j \neq 0$.
\end{proof}

Based on this independence property, we can say that
\begin{align*}
\E{}{\lf_i \ydep{i} | \lf_i \neq 0} \cdot \E{}{\lf_j \ydep{i} | \lf_j \neq 0} = \E{}{\lf_i \lf_j | \lf_i, \lf_j \neq 0}.
\end{align*}

This is because $\ydep{i}^2$ is always equal to $1$, and $\E{}{\lf_i \ydep{i} | \lf_i, \lf_j \neq 0}$ is equal to $\E{}{\lf_i \ydep{i} | \lf_i \neq 0}$  since $\lf_i$ and $\lf_j$ are conditionally independent given $\ydep{i}$. While $\E{}{\lf_i \ydep{i} | \lf_i \neq 0}$, $\E{}{\lf_j \ydep{i} | \lf_j \neq 0}$ correspond to the unknown accuracies of labeling functions, their product is observable. If we are able to introduce a third $\lf_k$ that is conditionally independent of $\lf_i$ and $\lf_j$ given the same $Y$, we have a system of three equations that can be solved using observable pairwise products that represent the rates of agreement between labeling functions. Then, we can use \cite{fu2020fast}'s \textit{triplet method} to recover these expectations, which we denote as $a_i^E = \E{}{\lf_i \ydep{i} | \lf_i \neq 0}$.

\paragraph{Inference}

These expectations can be converted into marginal clique probabilities of the form $\mu_{y, i}(X) := \Pr(\ydep{i} = y, \lf_i = \lf_i(X))$ based on a simple linear transformation using observable probabilities and the distribution prior $\Pr(Y)$. The overall joint probability can be written as a product over these marginal clique probabilities using a junction tree representation based on the maximal cliques and separator sets of $\gdep$. Because all labeling functions are conditionally independent given $Y$ and each labeling function only votes on one task, all maximal cliques are either of the form $\mu_{y, i}$ or $\mu_{y, C} := \Pr(Y_C = y_C)$, where $C \in \mathcal{C}$, and $\mathcal{C}$ is the set of maximal cliques on $Y$. Furthermore, all separator sets are singletons on elements of $Y$, which we express as $\mu_{y, j} := \Pr(Y_j = y_j)$ Then we can write the recovered distribution as
\begin{align}
\Pr_{\mu}(Y = y, \bm{\lf} = \bm{\lf}(X)) &= \frac{\prod_{i = 1}^m \mu_{y, i}(X) \prod_{C \in \mathcal{C}} \mu_{y, C}}{\prod_{j = 1}^T \mu_{y, j}^{d(Y^j) - 1}},
\label{eq:junctiontree}
\end{align}

where $d(Y^j)$ is the degree of $Y^j$ in $\gdep$. Then the probabilistic label $\widetilde{Y}$ on a data point $X$ can be computed using $\Pr_{\mu}(\widetilde{Y} = 1 | \bm{\lf} = \bm{\lf}(X))$ based on \eqref{eq:junctiontree}.

\subsection{Properties of the graphical model for section \ref{sec:theory}}
\label{subsec:pgm_properties}

We now discuss some properties of the single-task version of the graphical model in \eqref{eq:pgm}. These properties, which can also be shown for $T > 1$, are used in proving the bounds presented in section \ref{sec:theory} and also allow us to concisely write out the WS algorithm used for our theoretical results (Algorithm~\ref{alg:ws}).

First, when $T = 1$, \eqref{eq:pgm} can be rewritten as
\begin{align}
f_G(Y, \bm{v}) = \frac{1}{Z} \exp \bigg(\theta_Y Y + \sum_{i = 1}^m \theta_i (v_i^1 - v_i^{-1}) Y + \sum_{i = 1}^m \theta_{i, i} v_i^1 v_i^{-1} \bigg),
\end{align}

and \eqref{eq:junctiontree} becomes simpler:
\begin{align}
\Pr_{\mu}(Y = y, \bm{\lf} = \bm{\lf}(X)) = \frac{\prod_{i = 1}^m \mu_{y, i}(X)}{\Pr(Y = y)^{m - 1}} = \prod_{i = 1}^m \Pr(\lf_i = \lf_i(X) | Y = y) \Pr(Y = y).
\label{eq:junctiontree_single}
\end{align}

We present a symmetry property on the accuracies.

\begin{lemma}
For any labeling function $\lf_i$ with accuracy $a_i$,
\begin{align*}
\Pr(\lf_i = 1 | Y = 1, \lf_i \neq 0) &= \Pr(\lf_i = -1 | Y = -1, \lf_i \neq 0) = a_i  \\
\Pr(\lf_i = -1 | Y = 1, \lf_i \neq 0) &= \Pr(\lf_i = 1 | Y = -1, \lf_i \neq 0) = 1 - a_i.
\end{align*}

\label{lemma:fs_symmetry}
\end{lemma}

\begin{proof}
We can write the marginal distribution of $\Pr(Y, \lf_i | \lf_i \neq 0)$ as 
\begin{align*}
\Pr(Y, \lf_i | \lf_i \neq 0) = \frac{1}{Z_i} \exp \left(\theta_Y + \theta_i \lf_i Y \right).
\end{align*}

By Proposition $2$ of \cite{fu2020fast}, we know that $\lf_i Y \independent Y | \lf_i \neq 0$. This means that
\begin{align*}
\Pr(\lf_i Y = 1, Y = 1 | \lf_i \neq 0) = \Pr(\lf_i Y = 1 | \lf_i \neq 0) \cdot \Pr(Y = 1 | \lf_i \neq 0)
\end{align*}

Dividing both sides by $\Pr(Y = 1 | \lf_i \neq 0)$, we get 
\begin{align*}
&\Pr(\lf_i Y = 1 | Y = 1, \lf_i \neq 0) = \Pr(\lf_i Y = 1 | \lf_i \neq 0) \\
\Rightarrow \; & \Pr(\lf_i = 1 | Y = 1, \lf_i \neq 0) = a_i
\end{align*}

Repeating this calculation with $\Pr(\lf_i Y = 1, Y = -1 | \lf_i \neq 0$ gives us $\Pr(\lf_i = -1 | Y = -1, \lf_i \neq 0) = a_i$. We do the same again with $\Pr(\lf_i Y = -1, Y = -1 | \lf_i \neq 0)$ and $\Pr(\lf_i Y = -1, Y = 1 | \lf_i \neq 0)$ to get the second equation by noting that $\Pr(\lf_i Y = -1 | \lf_i \neq 0) = 1 - a_i$.
\end{proof}

Next, our graphical model assumes that abstaining sources do not affect the label $Y$.

\begin{lemma}
For any labeling function $\lf_i$,
\begin{align*}
\Pr(Y = y, \lf_i = 0) = \Pr(Y = y) \cdot \Pr(\lf_i = 0)
\end{align*}
\label{lemma:abstain}
\end{lemma}

\begin{proof}
By definition, we can write $\Pr(Y = y, \lf_i = 0)$ as $\Pr(Y = y, v_i^1 v_i^{-1} = 1)$. Using Proposition $2$ of \cite{fu2020fast}, we have that $v_i^{1} v_i^{-1} \independent Y$, and therefore
\begin{align*}
\Pr(Y = y, v_i^1 v_i^{-1} = 1) = \Pr(Y = y) \Pr(v_i^1 v_i^{-1} = 1) = \Pr(Y = y) \Pr(\lf_i = 0)
\end{align*}

Note that this property means that $\Pr(Y = y | \lf_i = 0) = \Pr(Y = y | \lf_i \neq 0) = \Pr(Y = y)$.
\end{proof}

Using these two properties, the output of the probabilistic label model is a simple expression in terms of the accuracies of labeling functions. By Lemmas \ref{lemma:fs_symmetry} and \ref{lemma:abstain}, we know that
\begin{align*}
\Pr(\lf_i Y = \lf_i(X) \cdot y | \lf_i \neq 0) &= \Pr(\lf_i = \lf_i(X) | Y = y, \lf_i \neq 0) = \frac{\Pr(\lf_i = \lf_i(X), Y = y, \lf_i \neq 0)}{\Pr(Y = y, \lf_i \neq 0)} \\
&= \frac{\Pr(\lf_i = \lf_i(X), Y = y)}{\Pr(Y = y) \Pr(\lf_i \neq 0)} = \frac{\Pr(\lf_i = \lf_i(X) | Y = y)}{\Pr(\lf_i \neq 0)}.
\end{align*}

$\Pr(\lf_i = \lf_i(X) | Y = y)$ is the probability used in \eqref{eq:junctiontree_single}, so we are able to easily use the accuracy $a_i$ to compute probabilistic labels. We now summarize this method of learning the label model for $T = 1$ in Algorithm \ref{alg:ws}, which is also the algorithm used for the bounds stated in section \ref{sec:theory}. Note that in order to reduce estimation error, we select triplets $i, j, k$ that have the largest possible overlaps in their support sets.

\begin{algorithm}[t]
	\caption{Label Model (Single Task)}
	\begin{algorithmic}
		\STATE \textbf{Input:}
		Labeling functions $\bm{\lf} = \lf_1, \dots, \lf_m$, prior $\Pr(Y)$, unlabeled dataset $D = \{X_i\}_{i = 1}^n$.
		\STATE \textbf{Returns:} $\Pr(\widetilde{Y} = 1 | \bm{\lf} = \bm{\lf}(X))$
		\FOR{$\lf_i \in \bm{\lf}$}
			\STATE Choose $\lf_j, \lf_k$ that have the largest minimum overlap among pairs of $\lf_i, \lf_j, \lf_k$. 
			\STATE $\hat{a}_i^E \gets \sqrt{|\Ehat{\lf_i \lf_j} \cdot \Ehat{\lf_i \lf_k} / \Ehat{\lf_j \lf_k} |}$ using observable estimates from $D$.
			\STATE Convert $\hat{a}_i^E$ into a probability $\hat{a}_i = \frac{1}{2}(\hat{a}_i^E + 1)$.
			\IF{$\lf_i(X) = 1$}
				\STATE  $\Pr_{\hat{\mu}}(\lf_i = \lf_i(X) | Y = 1) = \hat{a}_i \Pr(\lf_i \neq 0)$.
			\ELSIF{$\lf_i(X) = -1$}
				\STATE  $\Pr_{\hat{\mu}}(\lf_i = \lf_i(X) | Y = 1) = (1 - \hat{a}_i) \Pr(\lf_i \neq 0)$.
			\ELSE
				\STATE $\Pr_{\hat{\mu}}(\lf_i = \lf_i(X) | Y = 1) = \Pr(\lf_i = 0)$.
			\ENDIF
		\ENDFOR
		\STATE Compute $\Pr_{\hat{\mu}}(\widetilde{Y} = 1 | \bm{\lf} = \bm{\lf}(X)) = \prod_{i = 1}^m \Pr_{\hat{\mu}}(\lf_i = \lf_i(X) | Y = 1) \Pr(Y = 1) / \Pr(\bm{\lf} = \bm{\lf}(X))$.
	\end{algorithmic}
	\label{alg:ws}
\end{algorithm}


\section{Proofs}
\label{sec:proofs}


\subsection{Proposition~\ref{prop:acc} Proof}
\label{subsec:prop1}

Bounding the accuracy can be viewed through the following extension procedure: for each $X \in \sup\Pr(\bar{\lf}_i)$, we choose an $X' \in \supp(\lf_i)$ and assign $\bar{\lf}_i(X) = \lf_i(X')$. The method for choosing $X$' is as follows: we select the set of neighbors $N(X)$ in $\supp(\lf_i)$ within radius $r$ of $X$, choose some $x_j \in N(X)$ using the pdf of $X$ conditioned on $N(X)$, and set $X' = x_j$. Define $\mathcal{N} = \{N(X)\}$ to be the set of all possible neighbors. We can now write $\bar{a}_i$ using our procedure:
\begin{align*}
\bar{a}_i =& \Pr(\bar{\lf}_i(X) Y = 1 | \bar{\lf}_i(X) \neq 0) \\
=& \sum_{k \in \mathcal{N}} \Pr(\bar{\lf}_i(X) Y = 1 | \bar{\lf}_i(X) \neq 0, N(X) = k) \Pr(N(X) = k | \bar{\lf}_i(X) \neq 0) \\
 =& \sum_{k \in \mathcal{N}} \Pr(\bar{\lf}_i(X) Y = 1 | \bar{\lf}_i(X) \neq 0, N(X) = k, X' \in k) \Pr(X' \in k | N(X) = k, \bar{\lf}_i(X) \neq 0) \\
 &\times \Pr(N(X) = k | \bar{\lf}_i(X) \neq 0).
\end{align*}

Note that $\Pr(X' \in k | N(X) = k, \bar{\lf}_i(X) \neq 0)$ is equal to $1$, since $X'$ is always chosen from $k$ if $X$'s neighbors are $k$. So we now have
\begin{align}
\bar{a}_i =& \sum_{k \in \mathcal{N}} \Pr(\bar{\lf}_i(X) Y = 1 | \bar{\lf}_i(X) \neq 0, N(X) = k, X' \in k) \Pr(N(X) = k | \bar{\lf}_i(X) \neq 0) \nonumber \\
=& \sum_{k \in \mathcal{K}} \sum_{x_j \in k} \Pr(\bar{\lf}_i(X) Y = 1 | \bar{\lf}_i(X) \neq 0, N(X) = k, X' = x_j, X' \in k) \times \nonumber \\
& \Pr(X' = x_j | \bar{\lf}_i(X) \neq 0,  N(X) = k, X' \in k) \times \Pr(N(X) = k | \bar{\lf}_i(X) \neq 0). \label{eq:bar_a}
\end{align}

$\Pr(X' = x_j | \bar{\lf}_i(X) \neq 0,  N(X) = k, X' \in k)$ corresponds to our weighting rule, which here is equal to the conditional pdf $\Pr(X' = x_j | X' \in k)$. We now focus on the first term $\Pr(\bar{\lf}_i(X) Y = 1 | \bar{\lf}_i(X) \neq 0, N(X) = k, X' = x_j, X' \in k)$, which can be decomposed into the sum of four probabilities:
\begin{align}
\Pr(&\bar{\lf}_i(X) Y = 1 |  \bar{\lf}_i(X) \neq 0, N(X) = k, X' = x_j, X' \in k)  \nonumber \\
= \; & \Pr(\bar{\lf}_i(X) = \lf_i(X'), \lf_i(X') = Y', Y' = Y | \bar{\lf}_i(X) \neq 0, N(X) = k, X' = x_j, X' \in k) + \label{eq:decomposition}
 \\
& \Pr(\bar{\lf}_i(X) \neq \lf_i(X'), \lf_i(X') = Y', Y' \neq Y | \bar{\lf}_i(X) \neq 0, N(X) = k, X' = x_j, X' \in k) + \nonumber \\
& \Pr(\bar{\lf}_i(X) \neq \lf_i(X'), \lf_i(X') \neq Y', Y' = Y | \bar{\lf}_i(X) \neq 0, N(X) = k, X' = x_j, X' \in k) + \nonumber \\
& \Pr(\bar{\lf}_i(X) = \lf_i(X'), \lf_i(X') \neq Y', Y' \neq Y | \bar{\lf}_i(X) \neq 0, N(X) = k, X' = x_j, X' \in k) \nonumber.
\end{align}

Each of these probabilities can be expressed using the chain rule. Taking the first line \eqref{eq:decomposition} as an example, we have 
\begin{align}
\Pr(\bar{\lf}_i&(X) = \lf_i(X'), \lf_i(X') = Y', Y' = Y | \bar{\lf}_i(X) \neq 0, N(X) = k, X' = x_j, X' \in k) \nonumber \\
= \; & \Pr(\bar{\lf}_i(X) = \lf_i(X') | \lf_i(X') = Y', Y' = Y, \bar{\lf}_i(X) \neq 0, N(X) = k, X' = x_j, X' \in k) \label{eq:prop_rule} \\
& \times \Pr(\lf_i(X') = Y'| Y' = Y, \bar{\lf}_i(X) \neq 0, N(X) = k, X' = x_j, X' \in k) \label{eq:conditional_acc}\\
& \times \Pr(Y' = Y | \bar{\lf}_i(X) \neq 0, N(X) = k, X' = x_j, X' \in k). \label{eq:conditional_smooth}
\end{align}

\eqref{eq:prop_rule} is always equal to $1$ according to our extension procedure. \eqref{eq:conditional_acc} describes the accuracy of the labeling function on $x_j$, and \eqref{eq:conditional_smooth} describes the smoothness of the $Y$ labels over $X$'s neighbors. Define
\begin{align*}
h(r) = \Pr(Y' \neq Y | \bar{\lf}_i(X) \neq 0, X' \in \supp(\lf_i), \dist(X, X') \le r),
\end{align*}

and
\begin{align*}
a_i^j = \Pr(\lf_i(X') = Y' | X' = x_j).
\end{align*}

Then, since $X' \in \supp(\lf_i)$ and $\dist(X, X') \le r$ are implied by $x' \in k$, 
\begin{align*}
\Pr(\bar{\lf}_i(X) Y = 1 | \bar{\lf}_i(X) \neq 0, N(X) = k, X' = x_j, X' \in k) = a_i^j(1 - h(r)) + (1 - a_i^j) h(r).
\end{align*}

The new accuracy per point is always worse than before; therefore, we can use this extension method applied on the entirety of $\supp(\bar{\lf}_i)$ to lower bound the extension method applied on only $B_r(\lf_i)$ while retaining $a_i^j$ over $\supp(\lf_i)$ without propagating $X'$. Next, \eqref{eq:bar_a} can be written as
\begin{align*}
\bar{a}_i =& \sum_{k \in \mathcal{K}} \sum_{x_j \in k} (a_i^j (1 - h(r)) + (1 - a_i^j) h(r))  \Pr(X' = x_j | X' \in k) \Pr(N(X) = k | \bar{\lf}_i(X) \neq 0) \\
=& (1 - h(r)) \sum_{k \in \mathcal{K}} \sum_{x_j \in k} a_i^j  \times  \Pr(X' = x_j | X' \in k) \times \Pr(N(X) = k | \bar{\lf}_i(X) \neq 0) + \\
&h(r) \sum_{k \in \mathcal{K}} \sum_{x_j \in k} (1 - a_i^j)  \times  \Pr(X' = x_j | X' \in k) \times \Pr(N(X) = k | \bar{\lf}_i(X) \neq 0).
\end{align*}

Moreover, since accuracies on each $N(X)$ on average are equal to $a_i$ by assumption, our expression just becomes
\begin{align*}
\bar{a}_i = &(1 - h(r)) \sum_{k \in \mathcal{K}} a_i \times \Pr(N(X) = k | \bar{\lf}_i(X) \neq 0) + \\
 \; &h(r) \sum_{k \in \mathcal{K}} a_i \times \Pr(N(X) = k | \bar{\lf}_i(X) \neq 0)
\bar{a}_i\\
 =& (1 - h(r)) a_i + h(r) (1 - a_i).
\end{align*}

Lastly, we aim to upper bound $h(r)$, which will lower bound $\bar{a}_i$ if $a_i \ge 0.5$ (i.e. our labeling function is better than random). We write $h(r)$ in terms of $M_Y(r)$:
\begin{align*}
M_Y(r) \ge & \Pr(Y \neq Y' | \dist(X, X') \le r) \\
\ge & \Pr(Y \neq Y' | \dist(X, X') \le r, X \in \supp(\bar{\lf}_i), X' \in \supp(\lf_i)) \\
& \times \Pr(X \in \supp(\bar{\lf}_i), X' \in \supp(\lf_i) | \dist(X, X') \le r)
\end{align*}

Then, \begin{align*}
&h(r) \le  \frac{M_Y(r)}{\Pr(X \in \supp(\bar{\lf}_i), X'  \in \supp(\lf_i) | \dist(X, X') \le r)} \\
&= M_Y(r) \frac{\Pr(\dist(X, X') \le r)}{\Pr(\dist(X, X') \le r | X \in \supp(\bar{\lf}_i), X' \in \supp(\lf_i)) \Pr(X \in \supp(\bar{\lf}_i), X' \in \supp(\lf_i))}.
\end{align*}

$\Pr(\dist(X, X') \le r)$ is less than $\Pr(\dist(X, X') \le r | X \in \supp(\bar{\lf}_i), X' \in \supp(\lf_i))$, since the condition limits the maximum distance between points. Then our bound becomes
\begin{align*}
h(r) \le & \frac{M_Y(r)}{\Pr(X \in \supp(\bar{\lf}_i), X' \in \supp(\lf_i))} \\
\le& \frac{M_Y(r)}{\Pr(X \in \supp(\bar{\lf}_i)) \Pr(X' \in \supp(\lf_i))} \le \frac{M_Y(r)}{(1 + L_{\lf_i'}(r) p_{d(r)}) p_i^2},
\end{align*}

where we have used Lemma \ref{lemma:coverage} in the last inequality. Therefore, our final bound is
\begin{align*}
\bar{a}_i \ge a_i + (1 - 2a_i) \frac{M_Y(r)}{(1 + L_{\lf_i'}(r) p_{d(r)}) p_i^2}.
\end{align*}

We can also compute $\widetilde{a}_i = \Pr(\bar{\lf}_i(X) Y = 1 | X \in B_r(\lf_i))$ in terms of $a_i$ and $\bar{a}_i$.
\begin{align*}
\bar{a}_i = \Pr(\bar{\lf}_i Y = 1 | \bar{\lf}_i \neq 0) = &\Pr(\bar{\lf}_i Y = 1 | \lf_i \neq 0) \Pr(X \in \supp(\lf_i) | X \in \supp(\bar{\lf}_i)) + \\
& \Pr(\bar{\lf}_i(X) Y = 1 | X \in B_r(\lf_i)) \Pr(X \in B_r(\lf_i) | X \in \supp(\bar{\lf}_i)) \\
=& a_i \frac{1}{1 + L_{\lf_i'}(r) p_{d(r)}} + \widetilde{a}_i \frac{L_{\lf_i'}(r) p_{d(r)}}{1 + L_{\lf_i'}(r) p_{d(r)}}
\end{align*}

Therefore,
\begin{align*}
\widetilde{a_i} &= \frac{(1 + L_{\lf_i'}(r) p_{d(r)}) \bar{a}_i - a_i}{L_{\lf_i'}(r) p_{d(r)}} = \frac{\bar{a}_i - a_i }{L_{\lf_i'}(r) p_{d(r)}} + \bar{a}_i \\
&\le  \bar{a}_i + (1 - 2a_i) \frac{ M_Y(r)}{L_{\lf_i'}(r) p_{d(r)}(1 + L_{\lf_i'}(r) p_{d(r)})p_i^2}
\end{align*}


\subsection{Theorem~\ref{thm:risk_gap} Risk Estimation Error}
\label{subsec:thm1esterr}

We first bound the estimation error $\E{}{\|\hat{a} - a\|_1}$ by analyzing the given WS parameter recovery method. Note that this approach also holds for \sysx.

\begin{lemma}
Denote $M$ as the conditional second moment matrix over all observed variables, e.g. $M_{ij} = \E{}{\lf_i \lf_j | \lf_i \neq 0, \lf_j \neq 0}$ and assume that there are a sufficient number of samples such that all $\hat{M}_{ij} \neq 0$. Define $e_{\min} = \min \{ \min_i |\hat{a}_i^E|, |a_i^E| \}$ and $c_1 = \min_{i, j} \{\E{}{\lf_i \lf_j | \lf_i, \lf_j \neq 0} , \hat{\mathbb{E}} [\lf_i \lf_j | \lf_i, \lf_j \neq 0 \}]$. Assume $\text{sign}(a_i^E) = \text{sign}(\hat{a}_i^E)$ for all $a_i^E$. Then, when the minimum overlap of labeling functions used in the parameter recovery is $o_{\min} \le \Pr(X \in \supp(\lf_i) \cap \supp(\lf_j))$ for all $i, j$ pairs used for parameter recovery, the estimation error of the accuracies is
\begin{align*}
\mathbb{E}[\|\hat{a} - a \|_1] \le \frac{81 \sqrt{\pi}}{ e_{\min} c_1^2} \cdot \frac{m}{\sqrt{n \cdot o_{\min}}}.
\end{align*}
\label{lemma:a}
\end{lemma}

\begin{proof}
Suppose $\lf_i, \lf_j, \lf_k$ make up a conditionally independent triplet, and in total let $T$ be the set of triplets we need to recover all accuracies. Our estimate of $a^E$ can be obtained with 
\begin{align*}
|\hat{a}_i^E| = \left(\frac{|\hat{M}_{ij}| |\hat{M}_{ik}|}{|\hat{M}_{jk}|}\right)^{\frac{1}{2}}.
\end{align*}

Because we assume that signs are completely recoverable, 
\begin{align}
\|\hat{a}^E - a^E||_1 = \| |\hat{a}^E| - |a^E| \|_1 =  \sum_{\{i, j, k \} \in T} \big| |\hat{a}_i^E| - |a_i^E| \big| + \big||\hat{a}_j^E| - |a_j^E| \big| + \big||\hat{a}_k^E| - |a_k^E| \big|.
\label{eq:acc1}
\end{align}

Note that $|(\hat{a}_i^E)^2 - (a_i^E)^2| = |\hat{a}_i^E - a_i^E| |\hat{a}_i^E + a_i^E |$. By the reverse triangle inequality, $\big| |\hat{a}_i^E| - |a_i^E| \big| \le |\hat{a}_i^E - a_i^E| =  \frac{|(\hat{a}_i^E)^2 - (a_i^E)^2 |}{|\hat{a}_i^E + a_i^E|} \le \frac{1}{2e_{\min}} |(\hat{a}_i^E)^2 - (a_i^E)^2|$, because $|\hat{a}_i^E + a_i^E| = |\hat{a}_i^E| + |a_i^E| \ge 2 e_{\min}$. Then
\begin{align*}
&\big| |\hat{a}_i^E| - |a_i^E| \big| \le \frac{1}{2e_{\min}}|(\hat{a}_i^E)^2 - (a_i^E)^2| = \frac{1}{2e_{\min}} \Bigg| \frac{|\hat{M}_{ij}| |\hat{M}_{ik}|}{|\hat{M}_{jk}|} - \frac{|M_{ij}| |M_{ik}|}{|M_{jk}|} \Bigg| \\
= \;& \frac{1}{2e_{\min}} \Bigg| \frac{|\hat{M}_{ij}| |\hat{M}_{ik}|}{|\hat{M}_{jk}|} - \frac{|\hat{M}_{ij}| |\hat{M}_{ik}|}{|M_{jk}|} + \frac{|\hat{M}_{ij}| |\hat{M}_{ik}|}{|M_{jk}|} - \frac{|\hat{M}_{ij}| |M_{ik}|}{|M_{jk}|} +  \frac{|\hat{M}_{ij}| |M_{ik}|}{|M_{jk}|} - \frac{|M_{ij}| |M_{ik}|}{|M_{jk}|} \Bigg| \\
\le \; &\frac{1}{2e_{\min}} \left(\Big|\frac{\hat{M}_{ij} \hat{M}_{ik}}{\hat{M}_{jk} M_{jk}}\Big| \big| |\hat{M}_{jk}| - |M_{jk}| \big| + \Big|\frac{\hat{M}_{ij}}{M_{jk}} \Big| \big||\hat{M}_{ik}| - |M_{ik}|\big| + \Big|\frac{M_{ik}}{M_{jk}} \Big| \big||\hat{M}_{ij}| - |M_{ij}|\big|\right) \\
\le \; & \frac{1}{2e_{\min}} \left(\Big|\frac{\hat{M}_{ij} \hat{M}_{ik}}{\hat{M}_{jk} M_{jk}}\Big| |\hat{M}_{jk} - M_{jk}| + \Big|\frac{\hat{M}_{ij}}{M_{jk}} \Big| |\hat{M}_{ik} - M_{ik}| + \Big|\frac{M_{ik}}{M_{jk}} \Big| |\hat{M}_{ij} - M_{ij}|\right).
\end{align*}

Clearly, all elements of $\hat{M}$ and $M$ must be less than $1$. We further know that elements of $|M|$ and $|\hat{M}|$ are at least $c_1$. Define $\Delta{ij} = \hat{M}_{ij} - M_{ij}$. Then
\begin{align*}
\big| |\hat{a}_1^E| - |a_1^E| \big| &\le \frac{1}{2 e_{\min}} \left(\frac{1}{c_1^2} |\Delta_{jk}| + \frac{1}{c_1} |\Delta_{ik}| + \frac{1}{c)1} |\Delta_{ij}|\right) \\
&\le \frac{1}{2e_{\min}} \sqrt{\Delta_{jk}^2 + \Delta_{ik}^2 + \Delta_{ij}^2} \sqrt{\frac{1}{c_1^4} + \frac{2}{c_1^2}}.
\end{align*}

\eqref{eq:acc1} is now
\begin{align*}
\|\hat{a}^E - a^E \|_1 \le \frac{3}{2e_{\min}} \sqrt{\frac{1}{c_1^4} + \frac{2}{c_1^2}} \sum_{\{i, j, k \} \in T} \sqrt{\Delta_{jk}^2 + \Delta_{ik}^2 + \Delta_{ij}^2}.
\end{align*}

Note that the Frobenius norm of the error on the $3 \times 3$ submatrix of $M$ defined over $\tau = \{i, j, k\}$ is 
\begin{align*}
\|\hat{M}_{\tau} - M_{\tau} \|_F = \sqrt{2} \cdot \sqrt{\Delta^2_{ij} + \Delta^2_{ik} + \Delta^2_{jk}}.
\end{align*}

Moreover, $\|\hat{M}_{\tau} - M_{\tau} \|_F \le \sqrt{3} \|\hat{M}_{\tau} - M_{\tau} \|_2$. Putting everything together,
\begin{align*}
\| \hat{a}^E - a^E\|_1 &\le \frac{3\sqrt{3}}{2\sqrt{2} e_{\min}} \sqrt{\frac{1}{c_1^4} + \frac{2}{c_1^2}} \sum_{\tau \in T} \|\hat{M}_{\tau} - M_{\tau} \|_2,
\end{align*}

and the expectation is
\begin{align}
\E{}{\| \hat{a}^E - a^E\|_1} &\le \frac{3\sqrt{3}}{2\sqrt{2} e_{\min}} \sqrt{\frac{1}{c_1^4} + \frac{2}{c_1^2}} \sum_{\tau \in T} \E{}{\|\hat{M}_{\tau} - M_{\tau} \|_2}.
\label{eq:expected_err}
\end{align}

We use the matrix Hoeffding inequality as described in \cite{Ratner19}, which says
\begin{align*}
\Pr(||\hat{M} - M||_2 \ge \gamma) \le 2m\exp\left(-\frac{n\gamma^2}{32m^2}\right).
\end{align*} 

For $\tau$, we can integrate this to get
\begin{align*}
&\mathbb{E}[||\hat{M}_{\tau} - M_{\tau}||_2 ] = \int_0^{\infty} \Pr(||\hat{M_{\tau}} - M_{\tau}||_2 \ge \gamma) d\gamma \le \int_0^{\infty} 2m \exp \left(- \frac{n o_{\min}\gamma^2}{32 m^2} \right) d \gamma \\
= \; & 2m \sqrt{(32m^2 / (n \cdot o_{\min})) \pi} \int_0^{\infty} \frac{1}{\sqrt{(32m^2 / (n \cdot o_{\min})) \pi}} \exp \left(- \frac{\gamma^2}{32 m^2 / (n \cdot o_{\min})} \right) d \gamma \\
= \;& m \sqrt{\frac{32m^2 \pi}{n \cdot o_{\min}}}.
\end{align*}

Since $m = 3$ for $M_{\tau}$, this is equal to $36 \sqrt{\frac{2\pi}{n \cdot o_{\min}}}$. Substituting this back into \eqref{eq:expected_err}, we get
\begin{align*}
\E{}{\| \hat{a}^E - a^E\|_1} &\le \frac{3\sqrt{3}}{2\sqrt{2} e_{\min}} \sqrt{\frac{1}{c_1^4} + \frac{2}{c_1^2}} \sum_{\tau \in T} 36 \sqrt{\frac{2\pi}{n \cdot o_{\min}}} \\
&\le \frac{3\sqrt{3}}{ e_{\min}} \sqrt{\frac{1}{c_1^4} + \frac{2}{c_1^2}} \sum_{\tau \in T} 18 \sqrt{\frac{\pi}{n \cdot o_{\min}}}.
\end{align*}

At most $|T| = m$ if we were to do a separate triplet for each $a_i$, so our bound becomes
\begin{align*}
\E{}{\| \hat{a}^E - a^E\|_1} &\le \frac{3\sqrt{3}}{ e_{\min}} \sqrt{\frac{1}{c_1^4} + \frac{2}{c_1^2}} m \cdot  18 \sqrt{\frac{\pi}{n \cdot o_{\min}}} \\
&\le \frac{54 \sqrt{3}}{ e_{\min}} \sqrt{\frac{1}{c_1^4} + \frac{2}{c_1^2}} m \sqrt{\frac{\pi}{n \cdot o_{\min}}}.
\end{align*}

Finally, note that $\frac{1}{c_1^4} + \frac{2}{c_1^2} \le \frac{3}{c_1^4}$, and that $\E{}{\|\hat{a} - a \|_1} = \frac{1}{2} \E{}{\| \hat{a}^E - a^E \|_1}$. Therefore, the sampling error for the accuracy is bounded by
\begin{align*}
\mathbb{E}[\|\hat{a} - a \|_1] \le \frac{81 \sqrt{\pi}}{ e_{\min} c_1^2} \cdot \frac{m}{\sqrt{n \cdot o_{\min}}}.
\end{align*}

\end{proof}

Now, recall that $f_{WS}(X) = 2\Pr_{\mu}(\widetilde{Y} = 1 | \bm{\lf}(X)) - 1$, where $\mu$ corresponds to the accuracies. Define $\hat{f}_{WS}(X) = 2\Pr_{\hat{\mu}}(\widetilde{Y} = 1 | \bm{\lf}(X)) - 1$. We want to bound the gap between $\R(f_{WS})$ and $\R(\hat{f}_{WS})$.
\begin{align*}
\E{}{|\R(f_{WS}) - \R(\hat{f}_{WS})|} &\le \E{}{\E{X, Y \sim \mathcal{D}}{|\sign{Y} \big(\Pr_{\mu}(\widetilde{Y} = 1 | \bm{\lf}(X)) - \Pr_{\hat{\mu}}(\widetilde{Y} = 1 | \bm{\lf}(X)) \big) |}} \\
&= \E{}{\E{X, Y \sim \mathcal{D}}{\bigg| \frac{\Pr_{\mu}(\bm{\lf}(X) | Y = 1) p}{\Pr(\bm{\lf}(X))} - \frac{\Pr_{\hat{\mu}}(\bm{\lf}(X) | Y = 1)p}{\hat{\Pr}(\bm{\lf}(X))} \bigg|}} \\
&= \E{X, Y \sim \mathcal{D}}{\E{}{\bigg| \frac{\Pr_{\mu}(\bm{\lf}(X) | Y = 1) p}{\Pr(\bm{\lf}(X))} - \frac{\Pr_{\hat{\mu}}(\bm{\lf}(X) | Y = 1)p}{\hat{\Pr}(\bm{\lf}(X))} \bigg|}}.
\end{align*}

Here, we have abreviated $\Pr(\bm{\lf} = \bm{\lf}(X))$ as $\Pr(\bm{\lf}(X))$. Note that $\hat{\Pr}(\bm{\lf}(X))$ is observable. Let $|\Pr(\bm{\lf}(X)) - \hat{\Pr}(\bm{\lf}(X))| = \varepsilon$. Let $I_+$ be the event that $\frac{\Pr_{\mu}(\bm{\lf}(X) | Y = 1)}{\Pr(\bm{\lf}(X))} \ge \frac{\Pr_{\hat{\mu}}(\bm{\lf}(X) | Y = 1)}{\hat{\Pr}(\bm{\lf}(X))}$ and similarly $I_-$. Then, the risk gap becomes
\begin{align}
\mathbb{E}_{X, Y \sim \mathcal{D}} \Bigg[&\E{I_+}{\frac{\Pr_{\mu}(\bm{\lf}(X) | Y = 1) p}{\Pr(\bm{\lf}(X))} - \frac{\Pr_{\hat{\mu}}(\bm{\lf}(X) | Y = 1)p}{\hat{\Pr}(\bm{\lf}(X))}}\Pr(I_+) + \nonumber \\
&\E{I_-}{ \frac{\Pr_{\hat{\mu}}(\bm{\lf}(X) | Y = 1)p}{\hat{\Pr}(\bm{\lf}(X))} - \frac{\Pr_{\mu}(\bm{\lf}(X) | Y = 1) p}{\Pr(\bm{\lf}(X))}}\Pr(I_-)\Bigg]. \label{eq:conditioned_gap}
\end{align}

We can then bound $\hat{\Pr}(\bm{\lf}(X))$ with $\Pr(\bm{\lf}(X)) + \varepsilon$ in the case of $I_+$ and with $\Pr(\bm{\lf}(X)) - \varepsilon$ in the case of $I_-$. We can write \eqref{eq:conditioned_gap} as being less than 
\begin{align*}
\mathbb{E}_{\mathcal{D}} \Bigg[&\E{I_+}{\frac{\Pr_{\mu}(\bm{\lf}(X) | Y = 1) p}{\Pr(\bm{\lf}(X))} - \frac{\Pr_{\hat{\mu}}(\bm{\lf}(X) | Y = 1)p}{\Pr(\bm{\lf}(X)) + \varepsilon}}\Pr(I_+) + \\
& \E{I_-}{ \frac{\Pr_{\hat{\mu}}(\bm{\lf}(X) | Y = 1)p}{\Pr(\bm{\lf}(X)) - \varepsilon} - \frac{\Pr_{\mu}(\bm{\lf}(X) | Y = 1) p}{\Pr(\bm{\lf}(X))}}\Pr(I_-)\Bigg] \\
= \mathbb{E}_{\mathcal{D}} \Bigg[ &\E{I_+}{\frac{\Pr(\bm{\lf}(X)) (\Pr_{\mu}(\bm{\lf}(X) | Y = 1) - \Pr_{\hat{\mu}}(\bm{\lf}(X) | Y = 1)) + \varepsilon \Pr_{\mu}(\bm{\lf}(X) | Y = 1)}{\Pr(\bm{\lf}(X)) (\Pr(\bm{\lf}(X)) + \varepsilon)}} p \Pr(I_+) + \\
&\E{I_-}{\frac{\Pr(\bm{\lf}(X)) (\Pr_{\hat{\mu}}(\bm{\lf}(X) | Y = 1) - \Pr_{\mu}(\bm{\lf}(X) | Y = 1)) + \varepsilon \Pr_{\mu}(\bm{\lf}(X) | Y = 1)}{\Pr(\bm{\lf}(X)) (\Pr(\bm{\lf}(X)) - \varepsilon)}} p \Pr(I_-) \Bigg] \\
\le \mathbb{E}_{\mathcal{D}} \Bigg[&\E{}{\frac{\Pr(\bm{\lf}(X)) \big| \Pr_{\mu}(\bm{\lf}(X) | Y = 1) - \Pr_{\hat{\mu}}(\bm{\lf}(X) | Y = 1)\big| + \varepsilon \Pr_{\mu}(\bm{\lf}(X) | Y = 1)}{\Pr(\bm{\lf}(X)) (\Pr(\bm{\lf}(X)) - \varepsilon)}}\Bigg] \cdot p \\
=  \mathbb{E}_{\mathcal{D}}\Bigg[ & \E{}{\frac{\big| \Pr_{\mu}(\bm{\lf}(X) | Y = 1) - \Pr_{\hat{\mu}}(\bm{\lf}(X) | Y = 1)\big|}{\Pr(\bm{\lf}(X)) - \varepsilon}} +  \frac{\Pr_{\mu}(\bm{\lf}(X) | Y = 1)}{\Pr(\bm{\lf}(X))}\E{}{\frac{ \varepsilon}{\Pr(\bm{\lf}(X)) - \varepsilon}}\Bigg] \cdot p.
\end{align*}

We can apply Hoeffding's inequality on $\hat{\Pr}(\bm{\lf}(X)) - \Pr(\bm{\lf}(X))$. For any $\bm{\lf}(X)$, we have that $\Pr(| \hat{\Pr}(\bm{\lf}(X)) - \Pr(\bm{\lf}(X)) | \ge \gamma) \le 2\exp \left(-2n \gamma^2 \right)$. Therefore, with probability at least $1 - \delta$, $| \hat{\Pr}(\bm{\lf}(X)) - \Pr(\bm{\lf}(X)) | \le \sqrt{\frac{\log(2/\delta)}{2n}}$. Denote $\varepsilon_n = \sqrt{\frac{\log(2/\delta)}{2n}}$. Then with high probability, 
\begin{align*}
&\E{}{\frac{| \Pr_{\mu}(\bm{\lf}(X) | Y = 1) - \Pr_{\hat{\mu}}(\bm{\lf}(X) | Y = 1) |}{\Pr(\bm{\lf}(X)) - \varepsilon}} 
\le \; & \frac{\E{}{| \Pr_{\mu}(\bm{\lf}(X) | Y = 1) - \Pr_{\hat{\mu}}(\bm{\lf}(X) | Y = 1) |}}{\Pr(\bm{\lf}(X)) - \varepsilon_n}.
\end{align*}

Similarly, $\E{}{\frac{ \varepsilon}{\Pr(\bm{\lf}(X)) - \varepsilon}}$ can be written as $\E{}{\frac{\Pr(\bm{\lf}(X))}{\Pr(\bm{\lf}(X)) - \varepsilon}} - 1$. Under the same event that $\varepsilon \le \varepsilon_n$, we have that $\E{}{\frac{ \varepsilon}{\Pr(\bm{\lf}(X)) - \varepsilon}}$ is less than $\frac{\Pr(\bm{\lf}(X))}{\Pr(\bm{\lf}(X)) - \varepsilon_n} - 1 = \frac{\varepsilon_n}{\Pr(\bm{\lf}(X)) - \varepsilon_n}$. Therefore, with high probability our expression is bounded by
\begin{align}
 \mathbb{E}_{\mathcal{D}}\Bigg[ & \frac{\E{}{\big| \Pr_{\mu}(\bm{\lf}(X) | Y = 1) - \Pr_{\hat{\mu}}(\bm{\lf}(X) | Y = 1)\big|}}{\Pr(\bm{\lf}(X)) - \varepsilon_n)} +  \frac{\Pr_{\mu}(\bm{\lf}(X) | Y = 1)}{\Pr(\bm{\lf}(X))} \frac{\varepsilon_n}{\Pr(\bm{\lf}(X)) - \varepsilon_n}\Bigg] \cdot p
 \label{eq:no_delta}
\end{align}

We now focus on the value of $\big| \Pr_{\mu}(\bm{\lf}(X) | Y = 1) - \Pr_{\hat{\mu}}(\bm{\lf}(X) | Y = 1)\big|$. Regardless of what $X$ is, $\Pr_{\mu}(\bm{\lf}(X) | Y = 1)$ can be written as a product $\prod_{i = 1}^m b_i$, where $b_i \in \{a_i, 1 - a_i\}$, and similarly, $\Pr_{\hat{\mu}}(\bm{\lf}(X) | Y = 1) = \prod_{i = 1}^m \hat{b}_i$, where $\hat{b}_i \in \{\hat{a}_i, 1 - \hat{a}_i \}$. Then this difference of probabilities can be written as
\begin{align*}
&\big| \Pr_{\mu}(\bm{\lf}(X) | Y = 1) - \Pr_{\hat{\mu}}(\bm{\lf}(X) | Y = 1)\big| = \Big|\prod_{i = 1}^m b_i - \prod_{i = 1}^m \hat{b}_i \Big| \\
= \; &\Big| b_1 \prod_{i = 2}^m b_i - \hat{b}_1 \prod_{i = 2}^m b_i + \hat{b}_1 \prod_{i = 2}^m b_i -  \hat{b}_i \prod_{i = 2}^m \hat{b}_i \Big| \\
= \; & \Big|(b_1 - \hat{b}_1) \prod_{i = 2}^m b_i + \hat{b}_i \Big(\prod_{i = 2}^m b_i - \prod_{i = 2}^m \hat{b_i} \Big)  \Big| \\
\le \; & |b_1 - \hat{b}_1| \prod_{i = 2}^m b_i  +  \hat{b}_i \Big|\prod_{i = 2}^m b_i - \prod_{i = 2}^m \hat{b_i}  \Big| \\
\le \; & |b_1 - \hat{b}_1| + \Big|\prod_{i = 2}^m b_i - \prod_{i = 2}^m \hat{b_i}  \Big|.
\end{align*}

Therefore, $\big| \Pr_{\mu}(\bm{\lf}(X) | Y = 1) - \Pr_{\hat{\mu}}(\bm{\lf}(X) | Y = 1)\big|  \le \sum_{i = 1}^m |b_i - \hat{b}_i| = \sum_{i = 1}^m |a_i - \hat{a}_i|$. \eqref{eq:no_delta} now becomes
\begin{align*}
 \mathbb{E}_{X, Y \sim \mathcal{D}}\Bigg[ & \frac{\E{}{ \|a - \hat{a} \|_1} p}{\Pr(\bm{\lf}(X)) - \varepsilon_n)} \Bigg]+  \E{X, Y \sim \mathcal{D}}{\Pr(Y = 1 | \bm{\lf}(X)) \cdot \frac{\varepsilon_n}{\Pr(\bm{\lf}(X)) - \varepsilon_n}}.
\end{align*}

Define $c_p$ as the minimum value of $\Pr(\bm{\lf}(X))$, e.g. the smallest ``region'' and define $c_2 = \E{X, Y \sim D}{\Pr(Y = 1 | \bm{\lf}(X))}$. Then our risk gap is less than 
\begin{align*}
\frac{\E{}{ \|a - \hat{a} \|_1} p}{c_p - \varepsilon_n} +   \frac{\varepsilon_n \cdot c_2}{c_p - \varepsilon_n}.
\end{align*}

Using the accuracy estimation error result from Lemma \ref{lemma:a}, we have that with high probability,
\begin{align*}
\E{}{|\R(f_{WS}) - \R(\hat{f}_{WS})|} \le \frac{81 \sqrt{\pi} p}{ e_{\min} c_1^2 (c_p - \varepsilon_n)} \cdot \frac{m}{\sqrt{n \cdot o_{\min}}} + \frac{\varepsilon_n \cdot c_2}{c_p - \epsilon_n}.
\end{align*}

Since the parameter recovery method is the same for \sysx, we replace $n \cdot o_{min}$ with $n \cdot o_{min}(1 + L_{\min}(r_{\min}) \cdot p_{d(r_{\min})})$ using Lemma \ref{lemma:overlap}. Note that some of the constants will also be renamed. Our estimation error for \sysx is then
\begin{align*}
\E{}{|\R(f_{E}) - \R(\hat{f}_{E})|} \le \frac{81 \sqrt{\pi} p}{ \bar{e}_{\min} \bar{c}_1^2 (\bar{c}_p - \varepsilon_n)} \cdot \frac{m}{\sqrt{n \cdot o_{\min} (1 + L_{\min}(r_{\min}) \cdot p_{d(r_{\min})})}} + \frac{\varepsilon_n \cdot \bar{c}_2}{\bar{c}_p - \varepsilon_n}.
\end{align*}


\subsection{Theorem~\ref{thm:risk_gap} Generalization Lift}
\label{subsec:thm1lift}

We now present proof of the asymptotic generalization lift result in Theorem \ref{thm:risk_gap}, for which we compare label models parametrized by $\mu$ and $\bar{\mu}$.

\begin{align}
&\mathcal{R}(f_{WS}) - \mathcal{R}(f_E) = \E{(X, Y) \sim \mathcal{D}}{\ell(f_{WS}(X), Y) - \ell(f_{E}(X), Y)} \nonumber \\
&= \frac{1}{2} \mathbb{E}_{(X, Y) \sim \mathcal{D}}\big[|Y - 2\Pr_{\mu}(\widetilde{Y} = 1 | \bm{\lf} = \bm{\lf}(X)) + 1 |  - |Y - 2\Pr_{\bar{\mu}}(\widetilde{Y} = 1 | \bar{\lf}_i = \bm{\bar{\lf}}(X)) + 1 |\big] \nonumber \\
&= \E{X, Y \sim \mathcal{D}}{ \sign{Y} \big(\Pr_{\bar{\mu}}(\widetilde{Y} = 1 | \bm{\bar{\lf}} = \bar{\lf}_i(X), \lf_{-i}(X)) - \Pr_{\mu}(\widetilde{Y} = 1 | \bm{\lf} = \lf_i(X), \lf_{-i}(X))\big)} \nonumber \\
&= \E{X, Y \sim \mathcal{D}}{ \sign{Y} \Big(\frac{\Pr_{\bar{\mu}}(\widetilde{Y} = 1, \bm{\bar{\lf}} = \bar{\lf}_i(X), \lf_{-i}(X))}{\Pr_{\bar{\mu}}(\bm{\bar{\lf}} = \bar{\lf}_i(X), \lf_{-i}(X))} 
- \frac{\Pr_{\mu}(\widetilde{Y} = 1, \bm{\lf} = \lf_i(X), \lf_{-i}(X))}{\hat{\Pr}(\bm{\lf} = \lf_i(X), \lf_{-i}(X))}\Big)}. 
\label{eq:risk_abs}
\end{align}


We rewrite the probabilities:
\begin{align}
&\mathbb{E}_{X, Y \sim \mathcal{D}} \bigg[ \sign{Y} \Big(\frac{\Pr_{\bar{\mu}}(\bm{\bar{\lf}} = \bar{\lf}_i(X), \lf_{-i}(X) | Y = 1) p}{\Pr_{\bar{\mu}}(\bm{\bar{\lf}} = \bar{\lf}_i(X), \lf_{-i}(X) | Y= 1) p + \Pr_{\bar{\mu}}(\bm{\bar{\lf}} = \bar{\lf}_i(X), \lf_{-i}(X) |  Y= -1) (1-p)} \nonumber \\
&- \frac{\Pr_{\mu}(\bm{\lf} = \lf_i(X), \lf_{-i}(X) | Y = 1) p}{\Pr_{\mu}(\bm{\lf} = \lf_i(X), \lf_{-i}(X) | Y = 1) p + \Pr_{\mu}(\bm{\lf} = \lf_i(X), \lf_{-i}(X) | Y = -1) (1-p)}\Big)\bigg].
\label{eq:risk_joints}
\end{align}

Since all labeling functions are conditionally independent, we have that $\Pr_{\bar{\mu}}(\bm{\bar{\lf}} = \bar{\lf}_i(X), \lf_{-i}(X) | Y = 1)  =  \Pr_{\bar{\mu}}(\bar{\lf}_i = \bar{\lf}_i(X) | Y = 1) \times \prod_{j \neq i} \Pr(\bar{\lf}_{j} =  \lf_{j}(X) |Y = 1)$. 
Define
\begin{align*}
p(X) &= \Pr(Y = 1 | \lf_{-i} = \lf_{-i}(X)) \\
&= \frac{\prod_{j \neq i} \Pr(\bar{\lf}_{j} =  \lf_{j}(X) |Y = 1)p}{\prod_{j \neq i} \Pr(\bar{\lf}_{j} =  \lf_{j}(X) |Y = 1)p + \prod_{j \neq i} \Pr(\bar{\lf}_{j} =  \lf_{j}(X) |Y = -1)(1-p)}.
\end{align*}

Then \eqref{eq:risk_joints} can be written as
\begin{align*}
&\mathbb{E}_{} \bigg[ \sign{Y} \Big(\frac{\Pr_{\bar{\mu}}(\bar{\lf}_i = \bar{\lf}_i(X) | Y = 1) p(X)}{\Pr_{\bar{\mu}}(\bar{\lf}_i = \bar{\lf}_i(X) | Y = 1)
p(X) + \Pr_{\bar{\mu}}(\bar{\lf}_i = \bar{\lf}_i(X) | Y = -1)(1 - p(X))} \nonumber \\
&- \frac{\Pr_{\mu}(\lf_i = \lf_i(X) | Y = 1) p(X)}{\Pr_{\mu}(\lf_i = \lf_i(X) | Y = 1) p(X) + \Pr_{\mu}(\lf_i = \lf_i(X) | Y = -1) (1 - p(X))}\Big)\bigg] .
\end{align*}

Now we look at the lift over three regions: $\supp(\lf_i)$, $B_r(\lf_i)$, and $(\supp(\lf_i)\cup B_r(\lf_i))^C$. When $X \in \supp(\lf_i)$, we choose to use the same label model parameters and votes as before, and therefore there is no improvement in the generalization error over this region. Similarly, when both $\lf_i$ and $\bar{\lf}_i$ abstain, there is no improvement. Lastly, when $X \in B_r(\lf_i)$, the original labeling function would have abstained, in which case $\Pr_{\mu}(\lf_i = \lf_i(X) | Y = 1) = \Pr_{\mu}(\lf_i = 0 | Y = 1) =  \Pr(\lf_i = 0)$, but on the other hand the extended labeling function no longer abstains. Therefore, the lift comes from increased prediction accuracy over $B_r(\lf_i)$, and we can write 
\begin{align}
&\mathcal{R}(f_{WS}) - \mathcal{R}(f_E) = \Pr(X \in B_r(\lf_i)) \times \nonumber \\
&\mathbb{E}_{\ X \in B_r(\lf_i)} \bigg[ \sign{Y} \bigg(\frac{\Pr_{\bar{\mu}}(\bar{\lf}_i(X) | Y = 1) p(X)}{\Pr_{\bar{\mu}}(\bar{\lf}_i(X) | Y = 1) p(X) + \Pr_{\bar{\mu}}(\bar{\lf}_i(X) | Y = -1) (1 - p(X))} - p(X) \bigg) \bigg].
\label{eq:delta_br}
\end{align}

Now, we look at when $Y = 1$ and $Y = -1$ separately and use Lemma \ref{lemma:fs_symmetry} to write the probabilities inside the expectation in terms of accuracies.
\begin{enumerate}
\item $X \in B_r(\lf_i), Y = 1$:

This region contributes 
\begin{align*}
&\mathbb{E}_{Y = 1} \bigg[ \frac{\bar{a}_i p(X)}{\bar{a}_i p(X) + (1 - \bar{a}_i) (1 - p(X))} - p(X) \bigg]\Pr(X \in B_r(\lf_i), Y = 1, \bar{\lf}_i(X) = 1) + \\
&\mathbb{E}_{Y = 1} \bigg[\frac{(1 - \bar{a}_i) p(X)}{(1 - \bar{a}_i) p(X) + \bar{a}_i (1 - p(X))} - p(X) \bigg] \Pr(X \in B_r(\lf_i), Y = 1, \bar{\lf}_i(X) =-1).
\end{align*}

We can write the probabilities in terms of accuracy:
\begin{align*}
\Pr(X \in &B_r(\lf_i), Y = 1, \bar{\lf}_i(X) = 1) \\
=\; &\Pr(\bar{\lf}_i(X) = 1 | Y = 1, X \in B_r(\lf_i)) \Pr(Y = 1, X \in B_r(\lf_i)) \\
=\; & \widetilde{a}_i \Pr(Y = 1, X \in B_r(\lf_i)) \\
\Pr(X \in &B_r(\lf_i), Y = 1, \bar{\lf}_i(X) = -1) \\
=\; & \Pr(\bar{\lf}_i(X) = -1 | Y = 1, X \in B_r(\lf_i)) \Pr(Y = 1, X \in B_r(\lf_i)) \\
= \; & (1 - \widetilde{a}_i) \Pr(Y = 1, X \in B_r(\lf_i)),
\end{align*}

where $\widetilde{a}_i = \Pr(\bar{\lf}_i(X) Y = 1 | X \in B_r(\lf_i))$. It might not be immediately obvious that $\Pr(\bar{\lf}_i(X) = 1 | Y = 1, X \in B_r(\lf_i)) = \widetilde{a}_i$, since we only assume that the pre-extension and post-extension labeling functions can be described using the binary Ising model. However, we can write $\Pr(\bar{\lf}_i = 1 | Y = 1, \bar{\lf}_i \neq 0)$ as $\Pr(\bar{\lf}_i(X) = 1 | Y = 1, X \in \supp(\lf_i)) \Pr(X \in \supp(\lf_i) | \bar{\lf}_i \neq 0) + \Pr(\bar{\lf}_i(X) = 1 | Y = 1, X \in B_r(\lf_i)) \Pr(X \in B_r(\lf_i) | \bar{\lf}_i \neq 0)$. Therefore, $\Pr(\bar{\lf}_i(X) = 1 | Y = 1, X \in B_r(\lf_i)) = \frac{\bar{a}_i \Pr(\bar{\lf}_i \neq 0)  - a_i \Pr(X \in \supp(\lf_i))}{\Pr(X \in B_r(\lf_i)}$, which is the definition of $\widetilde{a}_i$. 

Since conditioning on the value of $\bar{\lf}_i(X)$ can be ignored if the expectation is already conditioned on $Y$, we are able to combine the sum  into one expectation to get
\begin{align}
&\Pr(Y = 1, X \in B_r(\lf_i)) \cdot \mathbb{E}_{X \in B_r(\lf_i), Y = 1} \bigg[\frac{\tilde{a}_i \bar{a}_i p(X)}{\bar{a}_i p(X) + (1 - \bar{a}_i) (1 - p(X))} \nonumber \\
&+ \frac{(1 - \tilde{a}_i)(1 - \bar{a}_i) p(X)}{(1 - \bar{a}_i) p(X) + \bar{a}_i (1 - p(X))} - p(X) \bigg]. \label{eq:pos_part}
\end{align}

\item $X \in B_r(\lf_i), Y = -1$:

This region contributes
\begin{align*}
&\mathbb{E}_{Y = -1} \bigg[p(X) - \frac{\bar{a}_i p(X)}{\bar{a}_i p(X) + (1 - \bar{a}_i) (1 - p(X))} \bigg]  \Pr(X \in B_r(\lf_i), Y =-1, \bar{\lf}_i(X) = 1) + \\
& \mathbb{E}_{Y = -1} \bigg[ p(X) - \frac{(1 - \bar{a}_i) p(X)}{(1 - \bar{a}_i) p(X) + \bar{a}_i (1 - p(X))} \bigg]  \Pr(X \in B_r(\lf_i), Y =-1, \bar{\lf}_i(X) =-1).
\end{align*}

We can write the probabilities as $(1 - \widetilde{a}_i) \Pr(X \in B_r(\lf_i), Y = -1) $ and $\widetilde{a}_i \Pr(X \in B_r(\lf_i), Y = -1)$ respectively. Similarly, we note that the expectations can be combined since the behavior of $\lf_{-i}$ conditioned on $Y$ does not depend on $\bar{\lf}_i$. Therefore, the lift in this region is equal to 
\begin{align*}
& \Pr(Y = -1, X \in B_r(\lf_i)) \times \\
&\mathbb{E}_{Y = -1} \bigg[p(X) - \frac{(1 - \widetilde{a}_i) \bar{a}_i p(X)}{\bar{a}_i p(X) + (1 - \bar{a}_i) (1 - p(X))} - \frac{\widetilde{a}_i (1 - \bar{a}_i) p(X)}{(1 - \bar{a}_i) p(X) + \bar{a}_i (1 - p(X))} \bigg].
\end{align*}

We can convert this expectation to an expectation conditioned on $Y = 1$. Define $S = \{-1, 1\}^{m - 1}$. Then,
\begin{align*}
&\mathbb{E}_{Y = -1} \bigg[p(X) - \frac{(1 - \widetilde{a}_i) \bar{a}_i p(X)}{\bar{a}_i p(X) + (1 - \bar{a}_i) (1 - p(X))} - \frac{\widetilde{a}_i (1 - \bar{a}_i) p(X)}{(1 - \bar{a}_i) p(X) + \bar{a}_i (1 - p(X))} \bigg] =\\
& \sum_{s \in S} \Pr(\lf_{-i} = s | Y = -1, X \in B_r(\lf_i)) \bigg(\Pr(Y = 1 | \lf_{-i} = s) - \\
& \frac{(1 - \widetilde{a}_i) \bar{a}_i \Pr(Y = 1 | \lf_{-i} = s)}{\bar{a}_i \Pr(Y = 1 | \lf_{-i} = s) + (1 - \bar{a}_i) \Pr(Y =- 1 | \lf_{-i} = s)} - \\
&\frac{\widetilde{a}_i (1 - \bar{a}_i) \Pr(Y = 1 | \lf_{-i} = s)}{(1 - \bar{a}_i) \Pr(Y = 1 | \lf_{-i} = s) + \bar{a}_i \Pr(Y =- 1 | \lf_{-i} = s)} \bigg).
\end{align*}

Note that $\Pr(\lf_{-i} = s | Y = -1, X \in B_r(\lf_i)) = \prod_{j \neq i} \Pr(\lf_j Y = -s_j | X \in B_r(\lf_i)) = \Pr(\lf_{-i} = -s | Y = 1, X \in B_r(\lf_i))$ by Lemma \ref{lemma:fs_symmetry}. Then we flip the sign of $s$ in the above expression to get
\begin{align*}
& \sum_{s \in -S} \Pr(\lf_{-i} = s | Y = 1, X \in B_r(\lf_i))  \bigg(\Pr(Y = 1 | \lf_{-i} = -s) - \\
& \frac{(1 - \widetilde{a}_i) \bar{a}_i \Pr(Y = 1 | \lf_{-i} = -s)}{\bar{a}_i \Pr(Y = 1 | \lf_{-i} = -s) + (1 - \bar{a}_i) \Pr(Y =- 1 | \lf_{-i} = -s)} -\\
& \frac{\widetilde{a}_i (1 - \bar{a}_i) \Pr(Y = 1 | \lf_{-i} = -s)}{(1 - \bar{a}_i) \Pr(Y = 1 | \lf_{-i} = -s) + \bar{a}_i \Pr(Y =- 1 | \lf_{-i} = -s)} \bigg).
\end{align*}

Define $p(-X) = \Pr(Y = 1 | \lf_{-i} = -\lf_{-i}(X))$. Then the above becomes
\begin{align}
 \mathbb{E}_{X \in B_r(\lf_i), Y = 1} \bigg[&p(-X) - \frac{(1 - \widetilde{a}_i) \bar{a}_i p(-X)}{\bar{a}_i p(-X) + (1 - \bar{a}_i) (1 - p(-X))} - \nonumber \\
&\frac{\widetilde{a}_i (1 - \bar{a}_i) p(-X)}{(1 - \bar{a}_i) p(-X) + \bar{a}_i (1 - p(-X))} \bigg].
\label{eq:convert_to_pos_y}
\end{align}

Recall that $p = 0.5$. We claim that $p(-X) = 1 - p(X)$. To see this, we can write $p(-X)$ as
\begin{align*}
p(-X) &= \frac{\Pr(Y = 1, \lf_{-i} = -\lf_{-i}(X))}{\Pr(\lf_{-i} = - \lf_{-i}(X))} \\
&= \frac{\Pr(\lf_{-i} = -\lf_{-i}(X) | Y = 1)}{\Pr(\lf_{-i} = - \lf_{-i}(X) | Y = 1) + \Pr(\lf_{-i} = -\lf_{-i}(X) | Y = -1)}.
\end{align*}

From Lemma \ref{lemma:fs_symmetry}, this can be written as 
\begin{align*}
&\frac{\Pr(\lf_{-i} = \lf_{-i}(X) | Y = -1)}{\Pr(\lf_{-i} =  \lf_{-i}(X) | Y = -1) + \Pr(\lf_{-i} = \lf_{-i}(X) | Y = 1)} \\
=& 1 - \frac{\Pr(\lf_{-i} = \lf_{-i}(X) | Y = 1)}{\Pr(\lf_{-i} =  \lf_{-i}(X) | Y = -1) + \Pr(\lf_{-i} = \lf_{-i}(X) | Y = 1)} \\
=& 1 - \frac{\Pr(\lf_{-i} = \lf_{-i}(X) | Y = 1) p}{\Pr(\lf_{-i} =  \lf_{-i}(X))} = 1 - p(X).
\end{align*}

Plugging this back into \eqref{eq:convert_to_pos_y}, we have
\begin{align}
 \mathbb{E}_{X \in B_r(\lf_i), Y = 1} \bigg[&(1 - p(X)) - \frac{(1 - \widetilde{a}_i) \bar{a}_i (1 - p(X))}{\bar{a}_i (1 - p(X)) + (1 - \bar{a}_i) p(X)} - \nonumber \\
&\frac{\widetilde{a}_i (1 - \bar{a}_i) (1 - p(X))}{(1 - \bar{a}_i) (1 - p(X)) + \bar{a}_i p(X)} \bigg]. \label{eq:neg_part}
\end{align}

We now show that \eqref{eq:neg_part} and the expectation in \eqref{eq:pos_part} are equal: 
\begin{align*}
&(1 - p(X)) - \frac{(1 - \widetilde{a}_i) \bar{a}_i (1 - p(X))}{\bar{a}_i (1 - p(X)) + (1 - \bar{a}_i) p(X)} - \frac{\widetilde{a}_i (1 - \bar{a}_i) (1 - p(X))}{(1 - \bar{a}_i) (1 - p(X)) + \bar{a}_i p(X)} \\
=& (1 - \widetilde{a}_i) + \widetilde{a}_i - p(X) - \frac{(1 - \widetilde{a}_i) \bar{a}_i (1 - p(X))}{\bar{a}_i (1 - p(X)) + (1 - \bar{a}_i) p(X)} - \\
&\frac{\widetilde{a}_i (1 - \bar{a}_i) (1 - p(X))}{(1 - \bar{a}_i) (1 - p(X)) + \bar{a}_i p(X)} \\
=& (1 - \widetilde{a}_i) \left(1 - \frac{ \bar{a}_i (1 - p(X))}{\bar{a}_i (1 - p(X)) + (1 - \bar{a}_i) p(X)} \right) + \\
&\widetilde{a}_i \left(1 - \frac{(1 - \bar{a}_i) (1 - p(X))}{(1 - \bar{a}_i) (1 - p(X)) + \bar{a}_i p(X)} \right) - p(X) \\
=& (1 - \widetilde{a}_i) \left(\frac{ (1 - \bar{a}_i) p(X)}{\bar{a}_i (1 - p(X)) + (1 - \bar{a}_i) p(X)} \right) + \\
&\widetilde{a}_i \left(\frac{\bar{a}_i p(X)}{(1 - \bar{a}_i) (1 - p(X)) + \bar{a}_i p(X)} \right) - p(X)
\end{align*}

\end{enumerate}

Therefore, combining \eqref{eq:pos_part} and \eqref{eq:neg_part} gives us
\begin{align*}
&\mathcal{R}(f_{WS}) - \mathcal{R}(f_E) = \Pr(X \in B_r(\lf_i)) \mathbb{E}_{X \in B_r(\lf_i), Y = 1}\bigg[\frac{\tilde{a}_i \bar{a}_i p(X)}{\bar{a}_i p(X) + (1 - \bar{a}_i) (1 - p(X))} \nonumber \\
&+ \frac{(1 - \tilde{a}_i)(1 - \bar{a}_i) p(X)}{(1 - \bar{a}_i) p(X) + \bar{a}_i (1 - p(X))} - p(X) \bigg].
\end{align*}

The denominator of both fractions above is at most $\max \{p(X), 1 - p(X) \}$, so we can write
\begin{align*}
&\mathcal{R}(f_{WS}) - \mathcal{R}(f_E) \ge \Pr(X \in B_r(\lf_i)) \mathbb{E}_{X \in B_r(\lf_i), Y = 1}\bigg[\frac{\tilde{a}_i \bar{a}_i p(X)}{\max\{p(X), 1 - p(X) \}} \nonumber \\
&+ \frac{(1 - \tilde{a}_i)(1 - \bar{a}_i) p(X)}{\max\{p(X), 1 - p(X) \}} - p(X) \bigg].
\end{align*}

We can split this expectation based on if $p(X) \ge 0.5$. Furthermore, the condition that $X \in B_r(\lf_i)$ is no longer relevant when also conditioning on $Y = 1$, so the above expectation becomes
\begin{align}
 \mathbb{E}_{Y = 1}\bigg[&\frac{\tilde{a}_i \bar{a}_i p(X)}{\max\{p(X), 1 - p(X) \}} + \frac{(1 - \tilde{a}_i)(1 - \bar{a}_i) p(X)}{\max\{p(X), 1 - p(X) \}} - p(X) \bigg] \nonumber \\
 =& (\widetilde{a}_i \bar{a}_i + (1 - \widetilde{a}) (1 - \bar{a}_i)) \cdot \Pr(p(X) \ge 0.5 | Y = 1) - \E{Y = 1}{p(X)} + \nonumber \\
 &\E{Y = 1}{\frac{p(X)}{1 - p(X)} (\widetilde{a}_i \bar{a}_i + (1 - \widetilde{a}_i)(1 - \bar{a}_i)) \; \Big| \; p(X) < 0.5} \Pr(p(X) < 0.5 | Y = 1) \nonumber \\
 \ge & (\widetilde{a}_i \bar{a}_i + (1 - \widetilde{a}) (1 - \bar{a}_i)) \cdot \Pr(p(X) \ge 0.5 | Y = 1) - \E{Y = 1}{p(X)} + \nonumber \\
 &\E{Y = 1}{p(X) | p(X) < 0.5} (\widetilde{a}_i \bar{a}_i + (1 - \widetilde{a}_i)(1 - \bar{a}_i)) \cdot  \Pr(p(X) < 0.5 | Y = 1) \nonumber \\
 =&  \left( \Pr(p(X) \ge 0.5 | Y = 1)  + \E{Y = 1}{p(X) | p(X) < 0.5} \Pr(p(X) < 0.5 | Y = 1) \right) \times \nonumber  \\
 & (\widetilde{a}_i \bar{a}_i + (1 - \widetilde{a}) (1 - \bar{a}_i)) - \E{Y = 1}{p(X)}.
 \label{eq:condition_half}
\end{align}

For convenience, denote $p_{0.5} = \Pr(p(X) \ge 0.5 | Y = 1)$. We can write $\E{Y = 1}{p(X)}$ as
\begin{align*}
\E{Y = 1}{p(X)} &= \E{Y = 1}{p(X) | p(X) < 0.5} (1 - p_{0.5}) + \E{Y = 1}{p(X) | p(X) \ge 0.5} p_{0.5}.
\end{align*}

We can use this to substitute $\E{Y = 1}{p(X) | p(X) < 0.5} (1 - p_{0.5})$ into \eqref{eq:condition_half} to get
\begin{align*}
\left( p_{0.5}  + \E{Y = 1}{p(X)} - \E{Y = 1}{p(X) | p(X) \ge 0.5} p_{0.5} \right) \cdot  (\widetilde{a}_i \bar{a}_i + (1 - \widetilde{a}) (1 - \bar{a}_i)) - \E{Y = 1}{p(X)}.
\end{align*}

Because $\widetilde{a}_i \bar{a}_i + (1 - \widetilde{a}) (1 - \bar{a}_i)$ is at most $1$, we can lower bound the above expression by replacing $\E{Y = 1}{p(X)}$ with $C :=\E{Y = 1}{p(X) | p(X) \ge 0.5}$. The expression becomes
\begin{align}
\left( p_{0.5}  + C - C \cdot p_{0.5} \right) \cdot  (\widetilde{a}_i \bar{a}_i + (1 - \widetilde{a}) (1 - \bar{a}_i)) - C.
\label{eq:p_half_gap}
\end{align}

The last step is to show that $p_{0.5} \ge 0.5$. Define $\mathcal{S} = \{-1, 0, 1\}^{m - 1}$ to be all possible sets of votes for the remaining labeling functions. Then we can write
\begin{align}
p_{0.5} =& \Pr(\Pr(Y = 1 | \lf_{-i} = \lf_{-i}(X)) \ge 0.5 | Y = 1) \nonumber \\
=& \sum_{s \in \mathcal{S}} \Pr( \Pr(Y = 1 | \lf_{-i} = s) \ge 0.5 | Y = 1, \lf_{-i}(X) = s) \Pr(\lf_{-i}(X) = s | Y = 1).
\label{eq:p_half}
\end{align} 

Using the fact that $p = 0.5$, the event that $\Pr(Y = 1 | \lf_{-i} = s) \ge 0.5$ is equivalent to 
\begin{align*}
&\frac{\Pr(\lf_{-i} = s | Y = 1) p}{\Pr(\lf_{-i} = s)} \ge 0.5 \\
\Rightarrow \; &\Pr(\lf_{-i} = s | Y = 1) \ge 0.5 (\Pr(\lf_{-i} = s | Y = 1) + \Pr(\lf_{-i} = s | Y = -1)) \\
\Rightarrow \; &\Pr(\lf_{-i} = s | Y = 1 ) \ge \Pr(\lf_{-i} = s | Y = -1).
\end{align*}

For a given $s$, denote $s^+ = \{i: s_i = 1 \}$ and $s^- = \{i: s_i = -1 \}$. Then, since abstaining labeling functions cancel out, we can write the above condition as
\begin{align*}
\prod_{s^+} a_i \prod_{s^-} (1 - a_i) \ge \prod_{s^+}(1 - a_i) \prod_{s^-} a_i.
\end{align*}

This is a condition we can check for each $s$; let the set of $s$ that satisfy this be $\mathcal{S}^+ = \{s \in \mathcal{S}: \prod_{s^+} a_i \prod_{s^-} (1 - a_i) \ge \prod_{s^+}(1 - a_i) \prod_{s^-} a_i \}$. We can now write \eqref{eq:p_half} as
\begin{align}
p_{0.5} = \sum_{s \in \mathcal{S}^+} \Pr(\lf_{-i}(X) = s | Y = 1) = \sum_{s \in \mathcal{S}^+} \prod_{s^+} a_i \prod_{s^-} (1 - a_i).
\label{eq:p_half_final}
\end{align}

In order to show that this value is greater than $0.5$, we equivalently show that it is greater than $\Pr(p(X) < 0.5 | Y = 1)$. Using the same approach as before, this can be written as
\begin{align*}
& \sum_{s \in \mathcal{S}} \Pr(\Pr(Y = 1 | \lf_{-i} = -s) < 0.5 | Y = 1, \lf_{-i}(X) = -s) \Pr(\lf_{-i}(X) = -s | Y = 1) \\
= &\sum_{s \in \mathcal{S}} \Pr(\Pr(\lf_{-i} = -s | Y = 1) < \Pr(\lf_{-i} = -s | Y = -1) | Y = 1, \lf_{-i}(X) = -s) \times \\
&\Pr(\lf_{-i}(X) = -s | Y = 1).
\end{align*}

Using symmetry of our graphical model from Lemma \ref{lemma:fs_symmetry}, this becomes
\begin{align*}
&\sum_{s \in \mathcal{S}} \Pr(\Pr(\lf_{-i} = s | Y = -1) < \Pr(\lf_{-i} = s | Y = 1) | Y = 1, \lf_{-i}(X) = -s) \prod_{s^+}(1 - a_i) \prod_{s^-} a_i \\
=& \sum_{s \in \mathcal{S}^+}  \prod_{s^+}(1 - a_i) \prod_{s^-} a_i.
\end{align*}

We see this is clearly less than \eqref{eq:p_half_final} by definition of $\mathcal{S}^+$. Therefore $p_{0.5} \ge 0.5$. We plug this back into \eqref{eq:p_half_gap} and use Lemma \ref{lemma:coverage} in multiplying by $\Pr(X \in B_r(\lf_i))$ to get
\begin{align*}
\R(f_{WS}) - \R(f_E) \ge L_{\lf_i'}(r) p_{d(r)} p_i \left(\frac{1}{2}(C + 1) (\widetilde{a}_i \bar{a}_i + (1 - \widetilde{a}) (1 - \bar{a}_i)) - C\right).
\end{align*}

\paragraph{General cases} This concludes our proof of Theorem \ref{thm:risk_gap}. However, it is worth examining what the exact quantity is for the generalization lift when all labeling functions are extended. For each labeling function, we split up the embedding space into three regions by use of a \textit{participation function} $v: \mathcal{X} \rightarrow \{0, 1, 2\}^m$ such that
\begin{align*}
v_i(X) = \begin{cases}
0 & \lf_i(X) = \bar{\lf}_i(X) = 0 \\
1 & \lf_i(X) = \bar{\lf}_i(X) \neq 0 \\
2 & \lf_i(X) = 0, \bar{\lf}_i(X) \neq 0
\end{cases}.
\end{align*}

That is, the $0$ case indicates that $X$ will not be labeled by $\lf_i$ even after extension, the $1$ case indicates $X \in \supp(\lf_i)$, and the $2$ case indicates $X \in B_r(\lf_i)$. We can now write $\R(f_{WS}) - \R(f_E)$ using this idea of participation in various regions. Let $V = \{0, 1, 2\}^m$. 
\begin{align}
\R(f_{WS}) &- \R(f_E) = \E{}{\sign{Y} (\Pr_{\bar{\mu}}(Y = 1 | \bm{\bar{\lf}} = \bm{\bar{\lf}}(X)) - \Pr_{\mu}(Y = 1 | \bm{\lf} = \bm{\lf}(X)))} \nonumber \\
=& \sum_{v \in V} \E{v(X) = v}{\sign{Y} (\Pr_{\bar{\mu}}(Y = 1 | \bm{\bar{\lf}} = \bm{\bar{\lf}}(X)) - \Pr_{\mu}(Y = 1 | \bm{\lf} = \bm{\lf}(X)))} \times \nonumber \\
&\Pr(v(X) = v).
\label{eq:lift_all}
\end{align}

We can write $\Pr_{\bar{\mu}}(Y = 1 | \bm{\bar{\lf}} = \bm{\bar{\lf}}(X))$ as $\frac{\prod_{i = 1}^m \Pr(\bar{\lf}_i = \bar{\lf}_i(X) | Y = 1)p}{\prod_{i = 1}^m \Pr(\bar{\lf}_i = \bar{\lf}_i(X) | Y = 1)p + \prod_{i = 1}^m \Pr(\bar{\lf}_i = \bar{\lf}_i(X) | Y = -1)(1 - p)}$ and a similar expression for $\Pr_{\mu}(Y = 1 | \bm{\lf} = \bm{\lf}(X))$. Now for $X$ such that $v(X) = v$, we split the participation function's output vector into three disjoint sets of indices $\{V_0, V_1, V_2\}$ where $V_j = \{i: v_i(X) = j \}$ for $j = 0, 1, 2$. Then the product of conditional probabilities across the labeling functions can be grouped in this way. For $\Pr_{\mu}(Y = 1 | \bm{\lf} = \bm{\lf}(X))$, $\Pr(\lf_i = \lf_i(X) | Y = 1)$ is equal to $\Pr(\lf_i = 0)$ when $i \in V_0, V_1$. Therefore, we can denote $p_1(X)$ as
\begin{align*}
p_1(X) =& \Pr_{\mu}(Y = 1 | \lf_{V_1} = \lf_{V_1}(X)) \\
=& \frac{\prod_{i \in V_1} \Pr(\lf_i = \lf_i(X) | Y = 1)p}{\prod_{i \in V_1} \Pr(\lf_i = \lf_i(X) | Y = 1)p + \prod_{i \in V_1}\Pr(\lf_i = \lf_i(X) | Y = -1)(1 - p)}.
\end{align*}

For $\Pr_{\bar{\mu}}(Y = 1 | \bm{\bar{\lf}} = \bm{\bar{\lf}}(X))$, $\Pr(\bar{\lf}_i = \bar{\lf}_i(X) | Y = 1)$ becomes $\Pr(\bar{\lf}_i = 0)$ for $i \in V_0$, becomes $\Pr(\lf_i = \lf_i(X) | Y = 1)$ for $i \in V_1$, and stays the same for $i \in V_2$. Then we can write $\Pr_{\bar{\mu}}(Y = 1 | \bm{\bar{\lf}} = \bm{\bar{\lf}}(X))$ in terms of $p_1(X)$, and \eqref{eq:lift_all} becomes
\begin{align*}
&\sum_{v \in V}\Pr(v(X) = v) \mathbb{E}_{v(X) = v}\bigg[\sign{Y} \times \\
&\bigg(\frac{\prod_{i \in V_2} \Pr(\bar{\lf}_i = \bar{\lf}_i(X) | Y = 1) p_1(X)}{\prod_{i \in V_2} \Pr(\bar{\lf}_i = \bar{\lf}_i(X) | Y = 1) p_1(X) + \prod_{i \in V_2} \Pr(\bar{\lf}_i = \bar{\lf}_i(X) | Y = -1) (1 - p_1(X))} - p_1(X) \bigg)\bigg].
\end{align*}

We can further write each $\Pr(\bar{\lf}_i = \bar{\lf}_i(X) | Y = 1)$ as $\bar{a}_i$ or $1 - \bar{a}_i$ according to Lemma \ref{lemma:fs_symmetry}. Let $V_2$ be partitioned into $V_2^+ = \{i \in V_2: \bar{\lf}_i(X) = 1 \}$ and $V_2^- = \{i \in V_2: \bar{\lf}_i(X) = -1 \}$. Then our expression for generalization lift is now
\begin{align*}
&\sum_{v \in V} \sum_{V_2^+, V_2^-} \Pr(v(X) = v, \bar{\lf}_{V_2^+}(X) = 1, \bar{\lf}_{V_2^-}(X) = -1) \mathbb{E}_{v(X) = v}\bigg[\sign{Y} \times \\
&\bigg(\frac{\prod_{i \in V_2^+} \bar{a}_i \prod_{j \in V_2^-} (1 - \bar{a}_j) p_1(X)}{\prod_{i \in V_2^+} \bar{a}_i \prod_{j \in V_2^-} (1 - \bar{a}_j) p_1(X) + \prod_{i \in V_2^+} (1 - \bar{a}_i) \prod_{j \in V_2^-} \bar{a}_j (1 - p_1(X))} - p_1(X) \bigg)\bigg].
\end{align*}

We now condition on $Y = 1$ and $Y = -1$ separately. Suppose $Y = 1$. Then the probability to evaluate is $\Pr(v(X) = v, \bar{\lf}_{V_2^+}(X) = 1, \bar{\lf}_{V_2^-}(X) = -1, Y = 1)$. This is equivalent to
\begin{align}
&\Pr(X \in \supp(\lf_{V_1}), X \in B_r(\lf_{V_2}), \bar{\lf}_{V_0}(X) = 0, \bar{\lf}_{V_2^+}(X) = 1, \bar{\lf}_{V_2^-}(X) = -1, Y = 1) \nonumber \\
=& \Pr(\bar{\lf}_{V_2+}(X) = 1, \bar{\lf}_{V_2^-}(X) = -1 | Y = 1, X \in \supp(\lf_{V_1}), X \in B_r(\lf_{V_2}), \bar{\lf}_{V_0}(X) = 0) \times \nonumber \\
& \Pr(Y = 1, X \in \supp(\lf_{V_1}), X \in B_r(\lf_{V_2}), \bar{\lf}_{V_0}(X) = 0).
\label{eq:prob_full}
\end{align}

In the first probability, the fact that $X \in \supp(\lf_{V_1}), \bar{\lf}_{V_0}(X) = 0$ does not matter since the probability is conditioned on $Y$. Then \eqref{eq:prob_full} becomes
\begin{align*}
& \Pr(\bar{\lf}_{V_2+}(X) = 1, \bar{\lf}_{V_2^-}(X) = -1 | Y = 1, X \in B_r(\lf_{V_2})) \times \nonumber \\
& \Pr(X \in \supp(\lf_{V_1}), X \in B_r(\lf_{V_2}), \bar{\lf}_{V_0}(X) = 0 | Y = 1) p \\
=& \prod_{i \in V_2^+} \tilde{a}_i \prod_{j \in V_2^-} (1 - \tilde{a}_j)  \prod_{i \in V_1} \Pr(X \in \supp(\lf_{i})) \prod_{i \in V_2} \Pr(X \in B_r(\lf_{i}) \prod_{i \in V_0} \Pr( \bar{\lf}_i(X) = 0) p.
\end{align*}

Define $\bar{A}(V_a, V_b) = \prod_{i \in V_a} \bar{a}_i \prod_{j \in V_b}(1 - \bar{a}_j)$, and similarly, $\widetilde{A}(V_a, V_b) =  \prod_{i \in V_a} \widetilde{a}_i \prod_{j \in V_b}(1 - \widetilde{a}_j)$. The amount of lift that the $Y = 1$ case contributes is then
\begin{align}
&\sum_{v \in V} \sum_{V_2^+, V_2^-} \bar{A}(V_2^+, V_2^-) \prod_{i \in V_1} \Pr(X \in \supp(\lf_{i})) \prod_{i \in V_2} \Pr(X \in B_r(\lf_{i}) \prod_{i \in V_0} \Pr( \bar{\lf}_i(X) = 0) p \times \nonumber \\
&\mathbb{E}_{v(X) = v, Y = 1} \bigg[\frac{\bar{A}(V_2^+, V_2^-) p_1(X)}{\bar{A}(V_2^+, V_2^-) p_1(X) + \bar{A}(V_2^-, V_2^+) (1 - p_1(X))} - p_1(X) \bigg].
\label{eq:lift_all_pos}
\end{align}

Next, suppose $Y = -1$. The probability to evaluate, using the same approach as before, is just
\begin{align*}
&\Pr(v(X) = v, \bar{\lf}_{V_2^+}(X) = 1, \bar{\lf}_{V_2^-}(X) = -1, Y = -1) \\
=& \prod_{i \in V_2^+} (1 - \widetilde{a}_i) \prod_{j \in V_2^-} \widetilde{a}_j  \prod_{i \in V_1} \Pr(X \in \supp(\lf_{i})) \prod_{i \in V_2} \Pr(X \in B_r(\lf_{i}) \prod_{i \in V_0} \Pr( \bar{\lf}_i(X) = 0) (1 - p),
\end{align*}

and the amount of lift that the $Y = -1$ case contributes is
\begin{align}
&\sum_{v \in V} \sum_{V_2^+, V_2^-} \widetilde{A}(V_2^-, V_2^+) \prod_{i \in V_1} \Pr(X \in \supp(\lf_{i})) \prod_{i \in V_2} \Pr(X \in B_r(\lf_{i}) \prod_{i \in V_0} \Pr( \bar{\lf}_i(X) = 0) (1 - p) \times \nonumber \\
& \mathbb{E}_{v(X) = v, Y = -1} \bigg[ p_1(X) - \frac{\bar{A}(V_2^+, V_2^-) p_1(X)}{\bar{A}(V_2^+, V_2^-) p_1(X) + \bar{A}(V_2^-, V_2^+) (1 - p_1(X))} \bigg].
\label{eq:lift_all_neg}
\end{align}

Combining \eqref{eq:lift_all_pos} and \eqref{eq:lift_all_neg} gives an exact value for $\R(f_{WS}) - \R(f_E)$. Intuitively, in both quantities the terms $\prod_{i \in V_1} \Pr(X \in \supp(\lf_{i})) \prod_{i \in V_2} \Pr(X \in B_r(\lf_{i}) \prod_{i \in V_0} \Pr( \bar{\lf}_i(X) = 0)$ represent the size of the region corresponding to $v$. Then, the expectation represents the particular lift over that region depending on if $Y = 1$ or $Y = -1$; we see that when $\bar{a}_i$ and $\tilde{a}_i$ are equal to $0.5$, the value inside the expectation becomes $0$, and there is no improvement in performance due to \sysx. Otherwise, the lift for $Y = 1$ is positive when $\bar{A}(V_2^+, V_2^-) \ge \bar{A}(V_2^-, V_2^+)$, e.g. $\prod_{i \in V_2^+} \frac{\bar{a}_i}{1 - \bar{a}_i} \ge \prod_{j \in V_2^-} \frac{\bar{a}_j}{1 - \bar{a}_j}$, and vice versa when $Y = -1$. These positive and negative expectations are then weighted by $\widetilde{A}(V_2^+, V_2^-)$ and $\widetilde{A}(V_2^-, V_2^+)$, which allows for an overall positive lift.

\subsection{Proof of Proposition~\ref{thm:prop2}} Next, we prove the result connecting risk to embedding quality.
\begin{proof}
First, consider
\[ \E{(X',Y')\sim \D}{\Pr(f(X) = Y | d(X,X') \leq r)} .\]
We can rewrite this as 
\[ \E{(X',Y')\sim \D}{\E{}{\ind{f(X) = Y} | d(X,X') \leq r}} .\]
From the tower law of expectation, we have that
\begin{align}
 &\E{(X',Y')\sim \D}{\Pr(f(X) = Y | d(X,X') \leq r)} \nonumber \\
 &\qquad=  \E{(X',Y')\sim \D}{\E{}{\ind{f(X) = Y} | d(X,X') \leq r}} \nonumber \\
&\qquad=\E{}{\ind{f(X) = Y}} \nonumber \\
&\qquad= \Pr(f(X) = Y) \nonumber \\
&\qquad= \R(f_z).
\label{eq:r1}
\end{align}

We similarly have that
\begin{align}
\E{(X',Y')\sim \D}{\Pr(f(X') = Y' | d(X,X') \leq r)} = \R(f_z) .
\label{eq:r2}
\end{align}

Recall that 
\begin{align} 
\E{(X',Y')\sim \D}{\Pr(f(X) = f(X') | d(X,X') \leq r)}  \leq M(r).
\label{eq:m1}
\end{align}

Now, consider 
\[ \E{(X',Y')\sim \D}{\Pr(f(X) = f(X') = Y = Y' | d(X,X') \leq r)} .\]

We have that
\begin{align}
& \E{(X',Y')\sim \D}{\Pr(f(X) = f(X') = Y = Y' | d(X,X') \leq r)} \nonumber \\
 &\qquad \geq  1 - \E{(X',Y')\sim \D}{\Pr(f(X) \neq f(X') | d(X,X') \leq r)} \nonumber \\
 &\qquad \qquad - \E{(X',Y')\sim \D}{\Pr(f(X) \neq Y | d(X,X') \leq r)} \nonumber\\
 &\qquad  \qquad - \E{(X',Y')\sim \D}{\Pr(f(X') \neq Y' | d(X,X') \leq r)} \nonumber  \\
& \qquad = 1 - (M(r) + 2\R(f)).  
 \end{align}
To see this, first, we used the fact that $\Pr(A=B=C=D) \geq 1-\Pr(A \neq B) - \Pr(A \neq C) - \Pr(B \neq D)$. Then, using this decomposition, we applied \eqref{eq:r1}, \eqref{eq:m1}, and \eqref{eq:r2}. 
 
Finally, note that $\Pr(f(X) = f(X') = Y = Y' | d(X,X') \leq r) \leq \Pr(Y = Y' | d(X,X') \leq r)$. Then, we have that
\begin{align*}
 \E{(X',Y')\sim \D}{\Pr(Y = Y' | d(X,X') \leq r)} \geq 1 - (M(r) + 2\R(f_z)),
\end{align*}
or,
\begin{align}
 \E{(X',Y')\sim \D}{\Pr(Y \neq Y' | d(X,X') \leq r)} \leq M(r) + 2\R(f_z),
 \label{eq:final}
 \end{align}

Next, we can write 
\begin{align*}
\E{(X',Y')\sim \D}{\Pr(Y \neq Y' | d(X,X') \leq r)} &= \int \Pr(X' = x) \Pr(X = x | d(x,x') \leq r) \ind{Y \neq Y')} \\
&= \Pr(Y \neq Y' | d(X,X') \leq r),
\end{align*}
as desired.

\end{proof}


\subsection{Theorem~\ref{thm:transfer_learning} Proof}
\label{subsec:thm2}

The proof proceeds in the following steps.
\begin{enumerate}
\item Show that risk is higher if nonoverlapping votes are used.
\item Bound the risk that results from nonoverlapping supports in terms of individual accuracies.
\item Use Proposition $2$ to bound the risk in terms of $R(f_i)$.
\end{enumerate}

We first show that given $\{\supp(\lf_1), \dots \supp(\lf_m) \}$, there is a way to create nonoverlapping $\{\supp'(\lf_1), \dots \supp'(\lf_m)\}$ such that the risk is always higher.

\begin{lemma}
Suppose there exists a region formed by the intersection of some $k$ labeling functions, $\lf_1, \dots, \lf_k$, such that $S_k = \bigcap_{i = 1}^k \supp(\lf_i)$. Suppose that $a_1 \ge a_2 \ge \dots \ge a_k$. Then the risk of the label model over $S_k$ increases if $\lf_1$ decides to abstain in $S_k$.
\end{lemma}

\begin{proof}
The risk of the label model is $\R(f_{WS}) = \E{X, Y \sim \mathcal{D}}{\frac{1}{2} |Y - 2\Pr(\widetilde{Y} = 1 | \bm{\lf} = \bm{\lf}(X))+ 1 |}$ (this also holds for $R(f_E)$ with $\bm{\bar{\lf}}$, but for simplicity we work with $\bm{\lf}$). Then the risk over the region $S_k$ is $\mathcal{R}_{S_k} (f_{WS}) = \E{X, Y \in S_k}{\frac{1}{2} |Y - 2\Pr(\widetilde{Y} = 1 | \bm{\lf} = \bm{\lf}(X))+ 1 |}$. This can be written as
\begin{align*}
\mathcal{R}_{S_k}(f_{WS}) = \E{S_k, Y = 1}{1 - \Pr(\widetilde{Y} = 1 | \bm{\lf} = \bm{\lf}(X))} p + \E{S_k, Y = -1}{\Pr(\widetilde{Y} = 1 | \bm{\lf} = \bm{\lf}(X))} (1 - p),
\end{align*}

where we've used that $\Pr(Y = 1 | S_k) = \Pr(Y = 1)$ by Lemma \ref{lemma:abstain}. Recall that labeling functions are conditionally independent given $Y$, so the probabilities can be decomposed into the product of individual $\Pr(\lf_i = \lf_i(X) | Y = 1)$. Moreover, for $i > k$, this probability is equivalent to $\Pr(\lf_i = 0 | Y = 1)$, which is just $\Pr(\lf_i = 0)$. After canceling terms out, we have
\begin{align}
&\E{S_k, Y = 1}{1 - \frac{\prod_{i = 1}^k \Pr(\lf_i = \lf_i(X) | Y = 1)p}{\prod_{i = 1}^k \Pr(\lf_i = \lf_i(X) | Y = 1)p + \prod_{i = 1}^k \Pr(\lf_i = \lf_i(X) | Y = -1)(1- p)}} p \; + \nonumber \\
& \E{S_k, Y = -1}{ \frac{\prod_{i = 1}^k \Pr(\lf_i = \lf_i(X) | Y = 1)p}{\prod_{i = 1}^k \Pr(\lf_i = \lf_i(X) | Y = 1)p + \prod_{i = 1}^k \Pr(\lf_i = \lf_i(X) | Y = -1)(1- p)}} (1 - p). \label{eq:overall_risk}
\end{align}

Following Lemma \eqref{lemma:fs_symmetry}, we condition on the value of $\lf_1(X)$ to write the risk using the accuracy $a_1$. For ease of notation, define the following expression:
\begin{align*}
p(X) &= \frac{\prod_{i = 2}^k \Pr(\lf_i = \lf_i(X) | Y = 1)p}{\prod_{i = 2}^k \Pr(\lf_i = \lf_i(X) | Y = 1)p + \prod_{i = 2}^k \Pr(\lf_i = \lf_i(X) | Y = -1)(1 - p)} \\
&= \Pr(Y = 1 | \lf_{2:k} = \lf_2(X), \dots \lf_k(X)).
\end{align*}

\eqref{eq:overall_risk} can now be written as
\begin{align}
&\E{S_k, \lf_1(X) = 1, Y = 1}{1 - \frac{a_1 p(X)}{a_1 p(X) + (1 - a_1)(1 - p(X))}} \Pr(\lf_1(X) = 1, Y = 1 | S_k) \; + \nonumber \\
 &\E{S_k, \lf_1(X) = -1, Y = 1}{1 - \frac{(1- a_1) p(X)}{(1 - a_1) p(X) +  a_1 (1 - p(X))}} \Pr(\lf_1(X) = -1, Y = 1 | S_k) \; + \nonumber \\
 &\E{S_k, \lf_1(X) = 1, Y =-1}{\frac{a_1 p(X)}{a_1 p(X) +  (1 - a_1) (1 - p(X))}} \Pr(\lf_1(X) = 1, Y = -1 | S_k) \; + \nonumber \\
 &\E{S_k, \lf_1(X) = -1, Y =-1}{\frac{(1-a_1) p(X)}{(1-a_1) p(X) + a_1 (1 - p(X))}} \Pr(\lf_1(X) = -1, Y = -1 | S_k).
\label{eq:split_acc}
\end{align}

We now describe what we want to show. We want to upper bound this expression by $\mathcal{R}_{S_k'}(f_{WS})$, which is the risk over $S_k$ when $\lf_1$ abstains. For this risk, we can consider abstaining as removing the $a_1$ term from \eqref{eq:split_acc} (equivalently, considering $a_1 = 0.5$) to get
\begin{align}
\mathcal{R}_{S_k'}(f_{WS}) = &\E{S_k, Y = 1}{1 - p(X)} p  +  \E{S_k, Y = -1}{ p(X)} (1 - p). \label{eq:abstained_risk}
\end{align}

It is thus sufficient to compare across expectations over $Y = 1$, and apply a symmetry argument to $Y = -1$. Before we do that, note that $\Pr(\lf_1(X) = 1, Y = 1 | S_k) = \Pr(\lf_1(X) = 1 | Y = 1, S_k) \Pr(Y = 1 | S_k) = a_i \cdot p$, and we have similar expressions for the other probabilities in \eqref{eq:split_acc}. The inequality we want to show is thus
\begin{align}
&\E{S_k, Y = 1}{1 - \frac{a_1 p(X)}{a_1 p(X) + (1 - a_1)(1 - p(X))}} a_1 p \nonumber \\
&+ \E{S_k, Y = 1}{1 - \frac{(1 - a_1) p(X)}{(1 - a_1) p(X) + a_1(1 - p(X))}} (1 - a_1)  p \le \E{S_k, Y = 1}{1 - p(X)} p.
\label{eq:y_ineq}
\end{align}

Combining terms, this is equal to showing
\begin{align*}
&\E{S_k, Y = 1}{\frac{a_1^2 p(X)}{a_1 p(X) + (1 - a_1)(1 - p(X))} + \frac{(1 - a_1)^2 p(X)}{(1 - a_1) p(X) + a_1(1 - p(X))}}   \ge \E{S_k, Y = 1}{p(X)}.
\end{align*}

We now claim that $\frac{a_1^2 p(X)}{a_1 p(X) + (1 - a_1)(1 - p(X))} + \frac{(1 - a_1)^2 p(X)}{(1 - a_1) p(X) + a_1(1 - p(X))} \ge p(X)$. If we rename $x := a_1$ and $p = p(X)$, this becomes equivalent to showing that
\begin{align*}
\frac{x^2}{xp + (1 - x)(1 - p)} + \frac{(1 - x)^2}{(1-x)p + x(1- p)} \ge 1.
\end{align*}

Define a function $f(x) = \frac{x}{xp + (1 - x)(1 -p)}$. Then we want to show that
\begin{align}
x f(x) + (1 - x) f(1 - x) \ge 1.
\label{eq:cvx}
\end{align}

Note that $f(x)$ is convex and increasing on $x \in [0, 1]$. Then we know that $f(x^2 + (1 - x)^2) \le x f(x) + (1 - x)f(1 - x)$. The smallest value of $x^2 + (1 - x)^2$ is $\frac{1}{2}$ when $x = \frac{1}{2}$, so we know that any $x f(x) + (1 - x)f(1 - x) \ge \frac{1/2}{p/2 + (1 - p)/2} = 1$. Therefore, we have shown that \eqref{eq:cvx} is true, and hence \eqref{eq:y_ineq} is also true. Furthermore, \eqref{eq:cvx} is sufficient to show that the expectations over $Y = -1$ also satisfy the same condition. We conclude that $\mathcal{R}_{S_k}(f_{WS}) \le \mathcal{R}_{S_k'}(f_{WS})$.

\end{proof}

An immediate result follows:
\begin{corollary}
We set $\supp'(\lf_1), \dots, \supp'(\lf_m)$ such that for each $S_k$ overlapping region, only one labeling function votes on it, and denote the resulting risk from this transformation as $\mathcal{R}_{\text{disjoint}}(f_{WS})$. Then $\mathcal{R}(f_{WS}) \le \mathcal{R}_{\text{disjoint}}(f_{WS})$.
\end{corollary} 

We stress that $\R(f_{WS})$ and $\R_{\text{disjoint}}(f_{WS})$ use the same set of accuracies - the only change is that at the inference step, $\R_{\text{disjoint}}(f_{WS})$ involves more abstains. For convenience, denote $\X_i = \supp'(\lf_i) \subseteq \supp(\lf_i)$ and $\X_0 = \{X: \lf_i(X) = 0 \; \forall \; i\}$, such that the distribution over $\X \backslash \X_0$ can be partitioned into $\X_1, \dots \X_m$. We can bound the risk of \sysx with $\R_{\text{disjoint}}(f_E)$, which simply replaces $\bm{\lf}$ with $\bm{\bar{\lf}}$, combined with the risk on $\X_0$. We have that
\begin{align}
\R(f_E) =& \E{X, Y \sim D}{\frac{1}{2}|Y - 2\Pr_{\bar{\mu}}(\widetilde{Y} = 1 | \bm{\bar{\lf}} = \bm{\bar{\lf}}(X)) + 1 |} \nonumber \\
\le &\sum_{i = 1}^m \E{\X_i}{\frac{1}{2}| Y - 2 \Pr_{\bar{\mu}}(\widetilde{Y} = 1 | \bm{\bar{\lf}} = \bm{\bar{\lf}}(X)) + 1 | } \Pr(X \in \X_i) + \R_0 \nonumber \\
= &\sum_{i = 1}^m \Big(\E{\X_i, Y = 1}{1 - \Pr_{\bar{\mu}}(\widetilde{Y} = 1 | \bm{\bar{\lf}} = \bm{\bar{\lf}}(X))} \Pr(Y = 1, X \in \X_i) + \nonumber  \\
& \E{\X_i, Y = -1}{\Pr_{\bar{\mu}}(\widetilde{Y} = 1 | \bm{\bar{\lf}} = \bm{\bar{\lf}}(X))} \Pr(Y = -1, X \in \X_i) \Big) + \R_0,
\label{eq:nonoverlapping}
\end{align}

where $\R_0$ is the risk over $\X_0$. We can use the fact that only $\lf_i$ votes in $\X_i$ to simplify $\E{\X_i, Y = \pm 1}{\Pr_{\bar{\mu}}(\widetilde{Y} = 1 | \bm{\bar{\lf}} = \bm{\bar{\lf}}(X))}$. 
\begin{align*}
&\E{\X_i, Y = 1}{1 - \Pr_{\bar{\mu}}(\widetilde{Y} = 1 | \bm{\bar{\lf}} = \bm{\bar{\lf}}(X))} \\
= \; &\E{\X_i, Y = 1}{1 - \frac{\Pr_{\bar{\mu}}(\bm{\bar{\lf}} = \bm{\bar{\lf}}(X) | Y = 1) p}{\Pr_{\bar{\mu}}(\bm{\bar{\lf}} = \bm{\bar{\lf}}(X) | Y = 1)p + \Pr_{\bar{\mu}}(\bm{\bar{\lf}} = \bm{\bar{\lf}}(X) | Y = -1) (1 - p)}}.
\end{align*}

Again, we write $\Pr_{\bar{\mu}}(\bm{\bar{\lf}} = \bm{\bar{\lf}}(X) | Y = 1)$ as $\prod_{i = 1}^m \Pr(\bar{\lf}_i = \bar{\lf}_i(X) | Y = 1)$, and since only $\lf_i$ votes on $\X_i$, all other terms are equal to $\Pr(\bar{\lf}_j = 0)$ for $j \neq i$. After canceling terms in the numerator and denominator, we simply have that
\begin{align}
&\E{\X_i, Y = 1}{1 - \Pr_{\bar{\mu}}(\widetilde{Y} = 1 | \bm{\bar{\lf}} = \bm{\bar{\lf}}(X))} \nonumber \\
= \;& \E{\X_i, Y = 1}{1 - \frac{\Pr(\bar{\lf}_i = \bar{\lf}_i(X) | Y = 1) p}{\Pr(\bar{\lf}_i = \bar{\lf}_i(X) | Y = 1) p + \Pr(\bar{\lf}_i = \bar{\lf}_i(X) | Y = -1) (1 - p)}} \nonumber \\
= \; & \left(1 - \frac{\bar{a}_i p}{\bar{a}_i p + (1 - \bar{a}_i)(1 - p)} \right) \Pr(\bar{\lf}_i(X) = 1 | Y = 1, X \in \X_i) + \nonumber \\
&\left(1 - \frac{(1 - \bar{a}_i) p}{(1 - \bar{a}_i) p + \bar{a}_i(1 - p)} \right) \Pr(\bar{\lf}_i(X) = -1 | Y = 1, X \in \X_i).
\label{eq:pos_decomp}
\end{align}

Similarly,
\begin{align}
\E{\X_i, Y = -1}{\Pr_{\bar{\mu}}(\widetilde{Y} = 1 | \bm{\bar{\lf}} = \bm{\bar{\lf}}(X))} =& \frac{\bar{a}_i p}{\bar{a}_i p + (1 - \bar{a}_i)(1 - p)}  \Pr(\bar{\lf}_i(X) = 1 | Y = -1, X \in \X_i) + \nonumber \\
& \frac{(1 - \bar{a}_i) p}{(1 - \bar{a}_i) p + \bar{a}_i(1 - p)}  \Pr(\bar{\lf}_i(X) = -1 | Y = -1, X \in \X_i).
\label{eq:neg_decomp}
\end{align}

We now show that when $p \ge \frac{1}{2}$, we have that $1 - \frac{\bar{a}_i p}{\bar{a}_i p + (1 - \bar{a}_i)(1 - p)} \le \frac{(1 - \bar{a}_i)p}{(1 - \bar{a}_i) p + \bar{a}_i (1 - p)}$, and vice versa for when $p \le \frac{1}{2}$. To see this, note that both $\frac{\bar{a}_i p}{\bar{a}_i p + (1 - \bar{a}_i)(1 - p)}$ and $\frac{(1 - \bar{a}_i)p}{(1 - \bar{a}_i) p + \bar{a}_i (1 - p)}$ are increasing in $p$. Then for $p \ge \frac{1}{2}$, we have that $1 - \frac{\bar{a}_i p}{\bar{a}_i p + (1 - \bar{a}_i)(1 - p)} \le 1 - \frac{\bar{a}_i/2}{\bar{a}_i /2 + (1 - \bar{a}_i)/2} = 1 - \bar{a}_i$, and $\frac{(1 - \bar{a}_i)p}{(1 - \bar{a}_i) p + \bar{a}_i (1 - p)} \ge \frac{(1 - \bar{a}_i)/2}{(1 - \bar{a}_i)/2 + \bar{a}_i/2} = 1 - \bar{a}_i$. Note that the opposite sequence of inequalities holds for $p \le \frac{1}{2}$. Combining this observation with \eqref{eq:pos_decomp} and \eqref{eq:neg_decomp}, we can write \eqref{eq:nonoverlapping} as:
\begin{align*}
\R(f_E) \le \sum_{i = 1}^m &\frac{(1 - \bar{a}_i)p}{(1 - \bar{a}_i) p + \bar{a}_i (1 - p)} \bar{a}_i \Pr(Y = 1, X \in \X_i) + \\
& \frac{\bar{a}_i p}{\bar{a}_i p + (1 - \bar{a}_i)(1 - p)} (1 - \bar{a}_i) \Pr(Y = 1, X \in \X_i) + \\
& \frac{\bar{a}_i p}{\bar{a}_i p + (1 - \bar{a}_i)(1 - p)} (1 - \bar{a}_i) \Pr(Y = -1, X \in \X_i) + \nonumber \\
& \frac{(1 - \bar{a}_i) p}{(1 - \bar{a}_i) p + \bar{a}_i(1 - p)} \bar{a}_i \Pr(Y = -1, X \in \X_i)  + \R_0,
\end{align*}

where we have used that $\Pr(\bar{\lf}_i(X) = 1 | Y = 1, X \in \X_i) = \Pr(\bar{\lf}_i(X) = 1 | Y = 1, \bar{\lf}(X_i) \neq 0) = \bar{a}_i$. This is equivalent to
\begin{align*}
\R(f_E) &\le \sum_{i = 1}^m \frac{(1 - \bar{a}_i) \bar{a}_i p}{(1 - \bar{a}_i) p + \bar{a}_i (1 - p)} \Pr(X \in \X_i) + \frac{\bar{a}_i (1 - \bar{a}_i)p}{\bar{a}_i p + (1 - \bar{a}_i)(1 - p)}  \Pr(X \in \X_i) + \R_0 \\
&= \sum_{i = 1}^m \bar{a}_i (1 - \bar{a}_i) p  \Pr(X \in \X_i) \left(\frac{1}{(1 - \bar{a}_i) p + \bar{a}_i (1 - p)} + \frac{1}{\bar{a}_i p + (1 - \bar{a}_i)(1 - p)} \right) + \R_0.
\end{align*}

The sum of fractions can be upper bounded by $\frac{2}{1-p}$, so our bound is now
\begin{align*}
\sum_{i = 1}^m \bar{a}_i (1 - \bar{a}_i)  \Pr(X \in \X_i) \frac{2p}{1 - p} + \R_0 \le \sum_{i = 1}^m (1 - \bar{a}_i)  \Pr(X \in \X_i) \frac{2p}{1 - p} + \R_0.
\end{align*}

We use Proposition \ref{thm:prop2} on the specialized models $f_i$ on $\X_i$ to get that 
\begin{align*}
\R(f_E) &\le \sum_{i = 1}^m \left( 1 - a_i + (2a_i - 1) \frac{M_{f_i}(r_i) + 2\R(f_i)}{ p_i^2 (1 + L_{\lf_i'}(r_i) p_{d(r_i)})} \right) \frac{2p}{1 - p} \Pr(X \in \X_i) + \R_0 \\
&= \frac{2p}{1-p} \E{\mathcal{D}_i}{1 - a_i + (2a_i - 1) \frac{M_{f_i}(r_i) + 2\R(f_i)}{p_i^2 (1 + L_{\lf_i'}(r_i) p_{d(r_i)})}}  + \R_0.
\end{align*}

Lastly, we compute $\R_0$. This is just $\Pr(X \in \X_0) \E{\X_0}{\frac{1}{2} |Y - 2 \Pr_{\bar{\mu}}(\widetilde{Y} = 1 | \bm{\bar{\lf}} = 0) + 1 |} = \Pr(X \in \X_0) \E{\X_0}{\frac{1}{2}|Y - 2p + 1|} = \Pr(X \in X_0) 2p(1 - p)$. Finally, we note that this entire derivation also holds for $p \le \frac{1}{2}$, so we define the max odds to be $b = \max \{\frac{p}{1 - p}, \frac{1- p}{p} \}$. Then our final bound becomes
\begin{align*}
\R(f_E) \le 2b \cdot \E{\mathcal{D}_i}{1 - a_i + (2a_i - 1)\frac{M_{f_i}(r_i) + 2\R(f_i)}{p_i^2 (1 + L_{\lf_i'}(r_i) p_{d(r_i)})} } + 2 \Pr(X \in \X_0) p(1 - p).
\end{align*}

\subsection{Controlling the minimum increases in support and overlaps from extensions}

We want to lower bound $\Pr(X \in B_r(\lf_i))$ and $\Pr(X \in \supp(\bar{\lf}_i \cap \bar{\lf}_j))$ using $L_{\lf_i'}(r)$ from our definition of probabilistic Lipschitzness.

\begin{lemma}
For a labeling function $\lf_i$ with threshold radius $r$,
\begin{align*}
\Pr(X \in B_r(\lf_i)) \ge L_{\lf_i'}(r) p_{d(r)} p_i,
\end{align*}
where $p_{d(r)} = \Pr(\dist(X, X') \le r)$ and $p_i = \Pr(X \in \supp(\lf_i))$.
\label{lemma:coverage}
\end{lemma}

\begin{proof}
By conditioning on the event that $X' \in \supp(\lf_i)$, we have that
\begin{align}
\Pr(X \in B_r(\lf_i)) \ge &\Pr(X \in B_r(\lf_i) | \dist(X, X') \le r, X' \in \supp(\lf_i))  \times \nonumber \\
&\Pr(\dist(X, X') \le r, X' \in \supp(\lf_i)).
\label{eq:B_r}
\end{align}

We focus on the first probability and aim to write it in terms of $L_{\lf_i'}(r)$. By definition of $B_r(\lf_i)$,
\begin{align*}
\Pr(X \in B_r(\lf_i) &| \dist(X, X') \le r, X' \in \supp(\lf_i)) \\
&=\Pr(\bar{\lf}_i(X) \neq 0, \lf_i(X) = 0 | \dist(X, X') \le r, X' \in \supp(\lf_i)) \\
&=\Pr(\lf_i(X) = 0 | \dist(X, X') \le r, X' \in \supp(\lf_i)) \\
&=\Pr(\lf_i(X) = 0, \lf_i(X') \neq 0 | \dist(X, X') \le r, X' \in \supp(\lf_i)).
\end{align*}

Notice that $\Pr(\lf_i(X) = 0, \lf_i(X') \neq 0 | \dist(X, X') \le r)  \le \Pr(\lf_i(X) = 0, \lf_i(X') \neq 0 | \dist(X, X') \le r, X \in \supp(\lf_i))$. This is because $\Pr(\lf_i(X) = 0, \lf_i(X') \neq 0, \dist(X, X') \le r)$ is equivalent to $\Pr(\lf_i(X) = 0, \lf_i(X') \neq 0, \dist(X, X') \le r, X \in \supp(\lf_i))$, and $\Pr(\dist(X, X') \le r) \ge \Pr(\dist(X, X') \le r, X \in \supp(\lf_i))$. By the definition of $L_{\lf_i'}(r)$, \eqref{eq:B_r} becomes
\begin{align*}
\Pr(X \in B_r(\lf_i)) \ge L_{\lf_i'}(r) \Pr(\dist(X, X') \le r, X' \in \supp(\lf_i)).
\end{align*}

If we suppose that the distribution of $X$ in $\mathcal{Z}$ is independent of an individual labeling function's support, this gives us
\begin{align*}
\Pr(X \in B_r(\lf_i)) \ge L_{\lf_i'}(r) p_{d(r)} p_i.
\end{align*}

\end{proof}

\begin{lemma}
Suppose that $\lf_i$ is extended $r_i$ and $\lf_j$ is extended $r_j$, and $\lf_i, \lf_j$ are used to estimate accuracy parameters. Let $r_{\min} = \min\{r_i, r_j \}$, and $L_{\min} = \min\{L_{\lf_i'}(r_{\min}), L_{\lf_j'}(r_{\min})\}$. Then,
\begin{align*}
\Pr(X \in \supp(\bar{\lf}_i) & \cap \supp(\bar{\lf}_j)) \ge \big(1 + ( 2L_{\min} - L_{\min}^2) \times p_{d(r_{\min})} \big) o_{\min}.
\end{align*}
\label{lemma:overlap}
\end{lemma}

\begin{proof}

For notation, we use $\supp(\lf_i \cap \lf_j)$ to refer to $\supp(\lf_i)\cap \supp(\lf_j)$. We want to look at how much the overlap between the labeling functions has increased after extending, e.g. $\supp(\bar{\lf}_i \cap \bar{\lf}_i) \backslash \supp(\lf_i \cap \lf_j)$.
\begin{align}
\Pr(X \in &\supp(\bar{\lf}_i \cap \bar{\lf}_j) \backslash \supp(\lf_i \cap \lf_j)) \nonumber \\
\ge &\Pr(X \in \supp(\bar{\lf}_i \cap \bar{\lf}_j) \backslash \supp(\lf_i \cap \lf_j) | \dist(X, X') \le r_{\min}, X' \in \supp(\lf_i \cap \lf_j)) \times \nonumber \\
&\Pr(\dist(X, X') \le r_{\min}, X' \in \supp(\lf_i \cap \lf_j)). \label{eq:overlap_increase}
\end{align}

We focus on the first term and split it into three cases:
\begin{align*}
& \Pr(X \in \supp(\bar{\lf}_i \cap \bar{\lf}_j) \backslash \supp(\lf_i \cap \lf_j) | \dist(X, X') \le r_{\min}, X' \in \supp(\lf_i \cap \lf_j)) \\
=& \Pr(\bar{\lf}_i(X) \neq 0, \lf_i(X) = 0, \bar{\lf}_j(X) \neq 0, \lf_j(X) = 0 | \dist(X, X') \le r_{\min}, X' \in \supp(\lf_i \cap \lf_j)) + \\
& \Pr(\bar{\lf}_i(X) \neq 0, \lf_i(X) \neq 0, \bar{\lf}_j(X) \neq 0, \lf_j(X) = 0 | \dist(X, X') \le r_{\min}, X' \in \supp(\lf_i \cap \lf_j)) + \\
& \Pr(\bar{\lf}_i(X) \neq 0, \lf_i(X) = 0, \bar{\lf}_j(X) \neq 0, \lf_j(X) \neq 0 | \dist(X, X') \le r_{\min}, X' \in \supp(\lf_i \cap \lf_j)) .
\end{align*}

The probability $\Pr(\bar{\lf}_i(X) \neq 0, \lf_i(X) = 0, \bar{\lf}_j(X) \neq 0, \lf_j(X) = 0 | \dist(X, X') \le r_{\min}, X' \in \supp(\lf_i \cap \lf_j))$ means that $X$ was previously in neither the support of $\lf_i$ or $\lf_j$ and now is in the support of both. This probability can be written as 
\begin{align*}
\Pr(\lf_i(X) = 0, \lf_j(X) = 0 | \dist(X, X') \le r_{\min}, X' \in \supp(\lf_i \cap \lf_j)),
\end{align*}

because $X$ being within $r_{\min}$ of the overlapping support automatically implies that $\bar{\lf}_i(X), \bar{\lf}_j(X) \neq 0$. Since $X' \in \supp(\lf_i \cap \lf_j)$ is equivalent to $\lf_i(X'), \lf_j(X') \neq 0$, we now have
\begin{align*}
\Pr(\lf_i(X) = 0, \lf_i(X') \neq 0, \lf_j(X) = 0, \lf_j(X') \neq 0 | \dist(X, X') \le r_{\min}, X' \in \supp(\lf_i \cap \lf_j)).
\end{align*}

By logic from Lemma \ref{lemma:coverage}, this is at least
\begin{align*}
&\Pr(\lf_i(X) = 0, \lf_i(X') \neq 0, \lf_j(X) = 0, \lf_j(X') \neq 0 | \dist(X, X') \le r_{\min}) \\
=&  \Pr(\lf_i(X) = 0, \lf_i(X') \neq 0 | \dist(X, X') \le r_{\min}) \Pr( \lf_j(X) = 0, \lf_j(X') \neq 0 | \dist(X, X') \le r_{\min}).
\end{align*}

Similarly, the other two probabilities can be written as
\begin{align*}
\Pr(\lf_i(X) \neq 0, \lf_i(X') \neq 0 | \dist(X, X') \le r_{\min}) \Pr( \lf_j(X) = 0, \lf_j(X') \neq 0 | \dist(X, X') \le r_{\min}), \\
\Pr(\lf_i(X) = 0, \lf_i(X') \neq 0 | \dist(X, X') \le r_{\min}) \Pr( \lf_j(X) \neq 0, \lf_j(X') \neq 0 | \dist(X, X') \le r_{\min}) .
\end{align*}

Adding these together and bounding by $L_{\min}$, \eqref{eq:overlap_increase} becomes
\begin{align*}
\Pr(X \in &\supp(\bar{\lf}_i \cap \bar{\lf}_j) \backslash \supp(\lf_i \cap \lf_j)) \\
\ge & \big( L_{\min}^2 + 2L_{\min}(1 - L_{\min}) \big) \times \Pr(\dist(X, X') \le r_{\min}, X' \in \supp(\lf_i \cap \lf_j)) \\
\ge & \big( 2L_{\min} - L_{\min}^2 \big) \times \Pr(\dist(X, X') \le r_{\min}, X' \in \supp(\lf_i \cap \lf_j)) \\
\ge &\big( 2L_{\min} - L_{\min}^2 \big) \times p_{d(r_{\min})} o_{\min}.
\end{align*}

Therefore, the new overlap is at least
\begin{align*}
\Pr(X \in \supp(\bar{\lf}_i) & \cap \supp(\bar{\lf}_j)) \ge \big(1 + ( 2L_{\min} - L_{\min}^2) \times p_{d(r_{\min})} \big) o_{\min}.
\end{align*}

\end{proof}

\section{Experimental Details}
\label{sec:supp_details}

We describe additional details about each task, 
including details about data sources (Section~\ref{sec:supp_details_dataset}),
end models (Section~\ref{sec:supp_details_end_model}),
supervision sources (Section~\ref{sec:supp_details_lfs}),
and setting extension thresholds (Section~\ref{sec:supp_details_thresholds}).

\subsection{Dataset Details}
\label{sec:supp_details_dataset}


\begin{table}[ht!]
    \centering
    \begin{tabular}{lccclccc}
        \toprule
        \textbf{Task} (Embedding) & $T$ & $m/T$ & \textbf{Prop} & \textbf{End Model} & $N_{train}$ & $N_{dev}$ & $N_{test}$   \\ \midrule
        \spam\ (BERT)             & 1   & 10    & 0.49          & BERT-Base          & 1,586       & 120       & 250          \\
        \weather\ (BERT)          & 1   & 103   & 0.53          & BERT-Base          & 187         & 50        & 50           \\
        \spouse\ (BERT)           & 1   & 9     & 0.07          & LSTM               & 22,254      & 2,811     & 2,701        \\
        \basketball\ (RN-101)     & 8   & 4     & 0.12          & ResNet-18          & 3,594       & 212       & 244          \\ 
        \commercial\ (RN-101)     & 6   & 4     & 0.32          & ResNet-50          & 64,130      & 9,479     & 7,496        \\
        \tennis\ (RN-101)         & 14  & 6     & 0.34          & ResNet-50          & 6,959       & 746       & 1,098        \\
        \basketball\ (BiT-M)      & 8   & 4     & 0.12          & BiT-M-R50x1        & 3,594       & 212       & 244          \\ 
        \commercial\ (BiT-M)      & 6   & 4     & 0.32          & BiT-M-R50x1        & 64,130      & 9,479     & 7,496        \\
        \tennis\ (BiT-M)          & 14  & 6     & 0.34          & BiT-M-R50x1        & 6,959       & 746       & 1,098        \\
        \bottomrule
    \end{tabular}
    \caption{
    Details for each dataset.
    $T$: the number of related elements modeled by the weak supervision label
    model.
    $m/T$: the number of supervision sources per element.
    \textbf{Prop}: The proportion of positive examples in each dataset.
    \textbf{End Model}: The end model used for fully-trained deep network
    baselines (TS-FT, WS-FT, and \sysx-FT).
    $N_{train}$: The size of the unlabeled training set.
    $N_{dev}$: The size of the labeled dev set.
    $N_{test}$: The size of the held-out test set.
    }
    \label{table:stats}
\end{table}

Table~\ref{table:stats} provides details on train/dev/test splits for each
dataset, as well as statistics about the positive class proportion and the
number of labeling functions.
Additional details about each dataset are provided below.

\paragraph{\spam}
We use the dataset as provided by
Snorkel\footnote{https://www.snorkel.org/use-cases/01-spam-tutorial} and those
train/dev/test splits.

\paragraph{\weather, \spouse}
These datasets are used in~\cite{Ratner18} and~\cite{fu2020fast} for
evaluation, and we use the train/dev/test splits from
those works (\weather\ is called \textbf{Crowd} in that work).

\paragraph{\basketball}
This dataset is a subset of ActivityNet and was used for evaluation
in~\cite{sala2019multiresws} and~\cite{fu2020fast}.
We use the train/dev/test splits from those works.

\paragraph{\commercial}
We use the dataset from~\cite{fu2019rekall} and~\cite{fu2020fast} and the
train/dev/test splits from those works.

\paragraph{\tennis}
We use the dataset from~\cite{fu2020fast} and the train/dev/test splits from
those works.

\subsection{Task-Specific End Models}
\label{sec:supp_details_end_model}

For \spam\ and \weather, we use BERT-Base for sequence classification as
implemented by HuggingFace~\cite{Wolf2019HuggingFacesTS}, and tune the learning
rate and number of epochs with grid search (ultimately using $0.00005, 4$ epochs
for \spam\ and $0.00004, 40$ epochs with early stopping for \weather).
For \spouse, we use an LSTM as used in~\cite{Ratner18} and~\cite{fu2020fast},
and use the hyperparameters from~\cite{fu2020fast}.
For \basketball, \commercial, and \tennis, we use different end models
depending on which embedding we are evaluating in order to provide a fair
comparison.
When using RN-101 pre-trained on ImageNet for embeddings, we use ResNet's
pre-trained on ImageNet, as in~\cite{fu2020fast} (RN-18 for \basketball,
RN-50 for \commercial\ and \tennis).
For these end models, we use hyperparameters from the previous
work~\cite{fu2020fast}.
When using BiT-M for the embeddings, we fine-tune BiT-M for the end model.
We use the default hyperparameters from BiT-M using the BiT ``hyperrule" as
detailed in~\cite{kolesnikov2019large}, but we resize frames to $256$ by $256$
and use batch size of $128$ for memory.
To adjust for the smaller batch size, we tune learning rate between $0.001$ and
$0.003$, as recommended by the official BiT implementation.
We end up using $0.003$ for \basketball, and $0.001$ for \commercial\ and
\tennis.

\subsection{Supervision Sources}
\label{sec:supp_details_lfs}

Supervision sources are expressed as short Python functions.
Each source relied on different information to assign noisy labels:

\paragraph{\spam, \weather, \spouse}
For these tasks, we used the same supervision sources as used in previous
work~\cite{Ratner18, fu2020fast}.
These are all text classification tasks, so they rely on text-based heuristics
such as the presence or absence of certain words, or particular regex patterns.

\paragraph{\basketball, \commercial, \tennis}
Again, we use sources from previous work~\cite{sala2019multiresws, fu2020fast}.
For \basketball, these sources rely on an off-the-shelf object detector to
detect balls or people, and use heuristics based on the average pixel of the
detected ball or distance between the ball and person to determine whether the
sport being played is basketball or not.
For \commercial, there is a strong signal for the presence or absence of
commercials in pixel histograms and the text; in particular, commercials are
book-ended on either side by sequences of black frames, and commercial segments
tend to have mixed-case or missing transcripts (whereas news segments are in
all caps).
For \tennis, we use an off-the-shelf pose detector to provide primitives for the
weak supervision sources.
The supervision sources are heuristics based on the number of people on court
and their positions.
Additional supervision sources use color histograms of the frames (i.e., 
how green the frame is, or whether there are enough white pixels for the court
markings to be shown).

\subsection{Setting $r_i$}
\label{sec:supp_details_thresholds}

We tune $r_i$ using the dev set in two steps.
First, we set all the $r_i$ to the same value $r$ and use grid search over $r$.
Then, we perform a series of small per-coordinate searches for a subset of the
labeling functions to optimize individual $r_i$ values.
For labeling functions with full coverage, we set the threshold to have no
extensions.

Now we report thresholds in terms of \textit{cosine similarities} (note that
this is a different presentation than in terms of distances).
For \spam, all thresholds are set to $0.85$, except for LF's $4$ and $7$, which
have thresholds $0.81$ and $0.88$.
For \weather, all thresholds are set to $0.87$, except for LF 0, which has
threshold $0.88$.
For \spouse, thresholds are set to $[0.955, 0.953, 0.95, 0.95, 0.951, 0.945, 0.95, 0.95, 0.95]$
for LF's $0-8$.
For \basketball, thresholds are set to $[1.0, 1.0, 0.1, 1.0]$ and
$[0.6, 1.0, 0.53, 1.0]$ for RN-101 and BiT-M embeddings, respectively.
For \commercial, thresholds are set to $[0.2, 1.0, 0.15, 0.15]$ and
$[0.63, 1.0, 0.4, 0.565]$ for RN-101 and BiT-M embeddings, respectively.
For \tennis, thresholds are set to $[1.0, 1.0, 1.0, 0.25, 1.0, 0.4]$ and
$[1.0, 1.0, 1.0, 0.8, 1.0, 0.9]$ for RN-101 and BiT-M embeddings, respectively.


\section{Additional Evaluation}
\label{sec:supp_eval}

We provide additional evaluation.
We present synthetic experiments validating our theoretical insights
(Section~\ref{sec:supp_eval_synthetics});
additional interactive-speed baseline
methods including other label models and other methods for transfer learning
without fine-tuning (Section~\ref{sec:supp_eval_baselines});
more detailed
metrics including standard deviations of training runs and precision/recall
breakdowns for label models (Section~\ref{sec:supp_eval_metrics});
measurements of how fine-tuning changes embeddings
(Section~\ref{sec:supp_eval_ft});
details and results for our ablation studies
(Section~\ref{sec:supp_eval_ablations});
and measurements of one-time preprocessing cost (Section~\ref{sec:supp_eval_preprocess}).

\subsection{Synthetics}
\label{sec:supp_eval_synthetics}

\begin{figure}[t]
  \centering
  \includegraphics[width=4in]{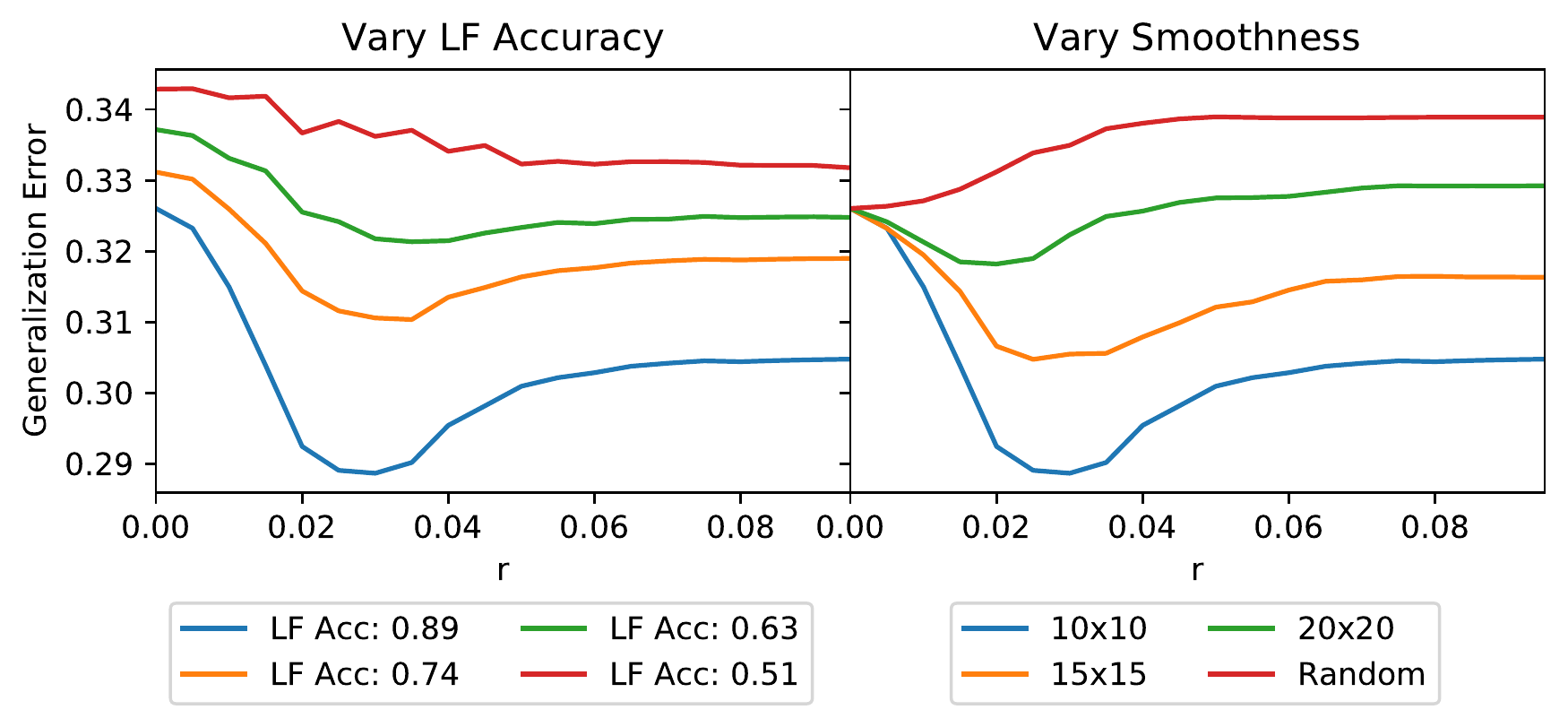}
  \caption{
    Reduction in generalization error from extending labeling functions of
    varying accuracies (left), and on
    embedding spaces of varying smoothness (right).
  }
  \label{fig:synthetics}
\end{figure}

We evaluate our method of extending one labeling function on synthetic data in two experiments to confirm the insights about accuracy, embedding smoothness, and choice of threshold radius from Theorem \ref{thm:risk_gap}.
We create an embedding space over $10000$ uniformly sampled points in $[0, 1]^2$ with a fixed class balance $\Pr(Y)$ and $m = 3$ labeling functions, where only $\lf_1$ is extended. To understand the impact of a labeling function's accuracy, we fix a task distribution by assigning $Y$ labels in a $10 \times 10$ ``checkerboard'' pattern and run our algorithm on four versions of $\lf_1$ with varying accuracy, keeping $\lf_1$'s support consistent. 
In Figure \ref{fig:synthetics} (left), we extend $\lf_1$ based on $r$ for each of the four versions of the labeling function.
This confirms that extending a highly accurate labeling function results in greater generalization lift. 
To understand the impact of Lipschitzness of the task distribution, we produce four distributions of $Y$ over the embedding space, three of which follow a checkerboard clustering pattern (such that more divisions mean less smoothness), and one that spatially distributes the values of $Y$ at random. In Figure \ref{fig:synthetics} (right), each curve represents performance of the same high-accuracy labeling function $(a_1 = 0.89$) over embeddings of varying Lipschitzness. This confirms that the greatest improvement due to an extension occurs for the smoothest embedding. Lastly, both of these graphs illustate the tradeoff in setting a threshold radius, confirming the theoretical insight that over-extending a labeling function can yield worse performance.

Note that our synthetic experiments reflect the observations from Theorem \ref{thm:risk_gap}. This suggests selecting a threshold radius by maximizing the performance of the lower bound $\R(f_{WS}) - \R(f_E)$ from the theorem. We sketch the process. All elements of the bound can be estimated or are already known. We can use $\hat{a}_i$ and $\hat{C}$ as estimates of the accuracies in standard WS, and $p_i$, $p_{d(r)}$, and $L_{\lf_i'}(r)$ are directly estimatable. The only other term left to approximate is $M_Y(r)$. There are two ways to do this: first, we can use the development set to directly compute the Lipschitzness upper bound. Alternatively, we can use Proposition \ref{thm:prop2} and estimate $M_{f_z}(r) + 2 \mathcal{R}(f_z)$ instead by choosing a simple model $f_z$ on the given embedding. Using these values, it is possible to select a threshold radius for $\lf_i$ to optimize the lower bound on the generalization lift.

\subsection{Additional Baselines}
\label{sec:supp_eval_baselines}

\begin{table*}[t]
    \centering
    \scriptsize
    \begin{tabular}{@{}rlccccccccccccc@{}}
    \toprule
    \multicolumn{3}{c}{}                           & \multicolumn{3}{c}{\textbf{Label Models}}             &\multicolumn{3}{c}{\textbf{NFT Baselines}}    \\
    \cmidrule(l){4-6} \cmidrule(l){7-10}                                                              
    & \textbf{Task} (Embedding) & \textbf{\sysx}   & \textbf{MV-LM} & \textbf{DP-LM} & \textbf{WS-LM}      & \textbf{TS-NFT} & \textbf{WS-NFT} & \textbf{TS-KNN} & \textbf{WS-KNN} \\
    \midrule                                                                                                            
    \parbox[t]{0mm}{\multirow{3}{*}{\rotatebox[origin=c]{90}{\textbf{NLP}}}}
    & \spam\ (BERT)             & \textbf{89.6}    & 84.0           & 86.0           & 83.6                & 78.5            & 87.2            & 49.2            & 48.4            \\
    & \weather\ (BERT)          & \textbf{82.0}    & 78.0           & 78.0           & 78.0                & 71.6            & 74.4            & 64.0            & 58.0            \\
    & \spouse\ (BERT)           & \textbf{48.4}    & 46.6           & 42.3           & 47.0                & 17.7            & 17.5            & 0.0             & 4.9             \\
    \midrule                                                                                                                                                     
    \parbox[t]{0mm}{\multirow{6}{*}{\rotatebox[origin=c]{90}{\textbf{Video}}}}
    & \basketball\ (RN-101)     & \textbf{31.3}    & 17.7           & 14.0           & 27.9                & 18.8            & 16.8            & 0.0             & 0.0             \\
    & \commercial\ (RN-101)     & \textbf{90.1}    & 83.7           & 83.7           & 88.4                & 73.6            & 75.5            & 30.7            & 30.3            \\
    & \tennis\ (RN-101)         & \textbf{82.9}    & 80.8           & 80.7           & 82.0                & 76.7            & 79.5            & 43.5            & 52.6            \\
    \cmidrule(l){2-10}
    & \basketball\ (BiT-M)      & \textbf{42.5}    & 17.7           & 14.0           & 27.9                & 22.8            & 23.2            & 5.6             & 10.2            \\
    & \commercial\ (BiT-M)      & \textbf{91.8}    & 83.7           & 83.7           & 88.4                & 71.7            & 73.8            & 29.3            & 45.3            \\
    & \tennis\ (BiT-M)          & \textbf{83.1}    & 80.8           & 80.7           & 82.0                & 75.5            & 79.0            & 50.1            & 52.1            \\
    \bottomrule
    \end{tabular}
    \caption{
    \sysx\ performance compared to additional label models and TL without
    fine-tuning baselines.
    MV-LM is majority vote over labeling function outputs; DP-LM is the label
    model from~\cite{Ratner18}; WS-LM is the label model from~\cite{fu2020fast}.
    TS-NFT and WS-NFT train a single fully-connected layer over pretrained
    embeddings using dev set labels or WS labels over the training set;
    TS-KNN and WS-KNN train a nearest neighbors classifier over pretrained
    embeddings using dev set labels or WS labels over the training set.
    }
    \vspace{-1em}
    \label{table:extra_baselines}
\end{table*}

Table~\ref{table:extra_baselines} shows additional interactive-speed baselines
that we compared against.
In addition to the label model from~\cite{fu2020fast}, which we compare against
in Section~\ref{sec:evaluation} (WS-LM), we also compare against the label
model from Snorkel~\cite{Ratner18} (DP-LM) and a majority vote baseline (MV-LM).
We also report results from another method of using pretrained embeddings
without fine-tuning---KNN search.
TS-KNN reports the performance of using the embeddings for a KNN classifier
trained with labels over the dev set, while WS-KNN reports the performance of a
KNN classifier trained over labels generated by WS-LM over the unlabeled
training set.

\subsection{Detailed Metrics}
\label{sec:supp_eval_metrics}


\begin{table*}[t]
    \centering
    \scriptsize
    \begin{tabular}{@{}rlccccccccccccc@{}}
    \toprule
    & \textbf{Task} (Embedding) & \textbf{TS-NFT}& \textbf{WS-NFT}&\textbf{TS-FT}&\textbf{WS-FT}&\textbf{\sysx-FT} \\
    \midrule                                                                                                            
    \parbox[t]{0mm}{\multirow{3}{*}{\rotatebox[origin=c]{90}{\textbf{NLP}}}}
    & \spam\ (BERT)             & 78.5 (10.8)    & 87.2 (4.2)     & 76.7 (17.1)& 90.0 (4.1)   & \textbf{94.1 (1.8)}\\
    & \weather\ (BERT)          & 71.6 (1.7)     & 74.4 (5.5)     & 71.2 (10.1)& 85.6 (2.2)   & \textbf{86.8 (3.0)}\\
    & \spouse\ (BERT)           & 17.7 (7.5)     & 17.5 (2.8)     & 20.4 (0.2) & 49.6 (2.4)   & \textbf{51.3 (0.9)}\\
    \midrule
    \parbox[t]{0mm}{\multirow{6}{*}{\rotatebox[origin=c]{90}{\textbf{Video}}}}
    & \basketball\ (RN-101)     & 18.8 (8.5)     & 16.8 (4.7)     & 26.8 (1.3) & 35.8 (1.5)   & \textbf{36.7 (1.3)}\\
    & \commercial\ (RN-101)     & 73.6 (2.6)     & 75.5 (1.3)     & 90.9 (1.0) & 92.5 (0.3)   & \textbf{93.0 (0.3)}\\
    & \tennis\ (RN-101)         & 76.7 (3.3)     & 79.5 (7.1)     & 57.6 (3.4) & 82.9 (1.0)   & \textbf{83.1 (0.5)}\\
    \cmidrule(l){2-7}                                             
    & \basketball\ (BiT-M)      & 22.8 (7.5)     & 23.2 (3.6)     & 29.1 (5.4) & 33.8 (7.3)   & \textbf{45.8 (5.2)}\\
    & \commercial\ (BiT-M)      & 71.7 (2.0)     & 73.8 (2.8)     & 93.2 (0.4) & 93.7 (0.3)   & \textbf{94.4 (0.2)}\\
    & \tennis\ (BiT-M)          & 75.5 (3.1)     & 79.0 (1.6)     & 47.5 (7.4) & 83.7 (0.1)   & \textbf{83.8 (0.4)}\\
    \bottomrule
    \end{tabular}
    \caption{
    Average and standard deviation in parentheses of methods trained with SGD,
    taken over five random seeds.
    }
    \vspace{-1em}
    \label{table:stddev}
\end{table*}

We provide more details about our experimental results.
First, we report measures of variance for non-deterministic methods.
Table~\ref{table:stddev} reports the means and standard deviations for the
results from Table~\ref{table:main_results} that required training layers of a
deep network using SGD.
We report results from runs with five random seeds.
All other methods (label models) are deterministic.


\begin{table*}[t]
    \centering
    \scriptsize
    \begin{tabular}{@{}rlccccccccccccc@{}}
    \toprule
    \textbf{Task} (Embedding)
    & \textbf{Metric} & \textbf{\sysx} & \textbf{WS-LM} \\
    \midrule
    \multirow{3}{*}{\spam\ (BERT)}
    & Precision       & \textbf{88.3}  & 86.7           \\
    & Recall          & \textbf{89.8}  & 77.1           \\
    & F1              & \textbf{89.1}  & 81.6           \\
    \midrule
    \multirow{3}{*}{\weather\ (BERT)}
    & Precision       & \textbf{91.3}  & 90.5            \\
    & Recall          & \textbf{75.0}  & 67.9            \\
    & F1              & \textbf{82.4}  & 77.6            \\
    \midrule
    \multirow{3}{*}{\spouse\ (BERT)}
    & Precision       & \textbf{41.6}  & 37.5           \\
    & Recall          & 57.8           & \textbf{58.3}  \\
    & F1              & \textbf{48.4}  & 45.7           \\
    \midrule
    \multirow{3}{*}{\basketball\ (RN-101)}
    & Precision       & 26.1           & \textbf{26.7}  \\
    & Recall          & \textbf{39.0}  & 29.3           \\
    & F1              & \textbf{31.3}  & 27.9           \\
    \midrule
    \multirow{3}{*}{\commercial\ (RN-101)}
    & Precision       & \textbf{91.1}  & 89.9           \\
    & Recall          & \textbf{89.0}  & 86.8           \\
    & F1              & \textbf{90.1}  & 88.3           \\
    \midrule
    \multirow{3}{*}{\tennis\ (RN-101)}
    & Precision       & 79.6           & \textbf{79.8}  \\
    & Recall          & \textbf{86.5}  & 85.7           \\
    & F1              & \textbf{82.9}  & 82.7           \\
    \midrule
    \multirow{3}{*}{\basketball\ (BiT-M)}
    & Precision       & \textbf{27.8}  & 26.7           \\
    & Recall          & \textbf{89.4}  & 29.3           \\
    & F1              & \textbf{42.5}  & 27.9           \\
    \midrule
    \multirow{3}{*}{\commercial\ (BiT-M)}
    & Precision       & \textbf{90.7}  & 89.9           \\
    & Recall          & \textbf{93.0}  & 86.8           \\
    & F1              & \textbf{91.8}  & 88.3           \\
    \midrule
    \multirow{3}{*}{\tennis\ (BiT-M)}
    & Precision       & \textbf{80.0}  & 79.8           \\
    & Recall          & \textbf{86.5}  & 85.7           \\
    & F1              & \textbf{83.1}  & 82.7           \\
    \bottomrule
    \end{tabular}
    \caption{
    Precision, Recall, and F1 scores for \sysx\ and WS-LM for all tasks.
    In most cases, the lift from \sysx\ is largely due to improvement in
    recall.
    }
    \vspace{-1em}
    \label{table:pre_rec_f1}
\end{table*}

Next, we discuss lift of \sysx\ over WS-LM in more detail, in terms of both
precision and recall.
As Table~\ref{table:embedding_utility} from Section~\ref{sec:evaluation}
suggests, \sysx\ primarily improves over WS-LM by improving coverage of the
labeling functions.
Table~\ref{table:pre_rec_f1} shows precision, recall, and F1 scores for all
tasks for \sysx\ and WS-LM.
In most tasks, the performance lift of \sysx\ over WS-LM is primarily reflected
in the increase in recall---exposing the mechanism through which increased
coverage results in improved performance.

\subsection{Measuring Changes from Fine-Tuning}
\label{sec:supp_eval_ft}


\begin{table*}[t]
    \centering
    \tiny
    \begin{tabular}{@{}rcccccccccccccc@{}}
        \toprule
        \multicolumn{1}{c}{} & \multicolumn{3}{c}{BERT} & \multicolumn{3}{c}{RN-101 ImageNet} & \multicolumn{3}{c}{BiT-M RN-50x1} \\
        \cmidrule(l){2-4} \cmidrule(l){5-7} \cmidrule(l){8-10}
        & \spam & \weather & \spouse & \basketball & \commercial & \tennis & \basketball & \commercial & \tennis \\
        \num{min. pre-pool similarity}
        & 0.313 & 0.293    & -0.061  & 0.481       & -0.023      & 0.323   & -0.061      & -0.092      & -0.069\\
        $1 - d(PT, DN)$
        & 0.735 & 0.677    & 0.248   & 0.562       & 0.284       & 0.113   & 0.907       & 0.559       & 0.408\\
        \cmidrule(l){2-10}
        \textbf{WS-FT vs. WS-NFT}
        & +2.8  & +11.2    & +32.1   & +19.0       & +17.0       & +3.4    & +10.6       & +19.9       & +4.7 \\
        \bottomrule
    \end{tabular}
    \caption{
    \num{min. pre-pool similarity} is the minimum similarity between pretrained
    and fine-tuned embeddings before the pooling
    step (tokens for text, spatial areas for visual).
    Fine-tuning can change the embeddings in ways that are not reflected in the
    overall embedding distance ($1 - d(PT-DN)$) but still move the
    features---resulting in a gap bettween WS-FT and WS-NFT.
    }
    \label{table:supp_extra_results}
\end{table*}

As discussed in Section~\ref{sec:evaluation}, the similarity between pretrained
and fine-tuned embeddings $1 - d(PT, DN)$ is a good predictor for how much lift
\sysx\ can provide over WS-LM.
However, it is not a good predictor of the relative performance of transfer
learning without fine-tuning.
We investigated this phenomenon, and found that a major cause of this gap is
that fine-tuning can change the embeddings in ways that are reflected in a
distance metric.

Table~\ref{table:supp_extra_results} demonstrates this phenomenon and reports
\textit{min. pre-pool similarity}, an additional metric demonstrating that
embeddings can be changed significantly during fine-tuning without the
differences being reflected in an overall similarity metric.
This metric looks at similarity one step before the final embeddings are
generated.
In both NLP and video applications, the pre-trained networks
generate many descriptors for each dataset; BERT generates an embedding for
each token in the sentence, while ResNet's and BiT-M generate embeddings over
convolutional windows of each frame.
Normally, these embeddings are pooled to generate a single embedding for the
data point.
min. pre-pool similarity reports the minimum similarity between matching
descriptors \textit{before} this pooling step.
For example, for a sentence, it compares embeddings of the first token
generated by the pre-trained and fine-tuned network, embeddings of the second
token, and so on, and takes the minimum similarity.
For the ResNet's, it compares matching convolutional windows for the same data
point.
For each of the datasets, this metric is \textit{low}, demonstrating that the
embeddings before the pooling step are changing significantly.
This helps explain why there can be a gap between fine-tuned and pre-trained
performance, even though the overall (pooled) similarity between fine-tuned and
pre-trained embeddings is high.

\subsection{Ablations}
\label{sec:supp_eval_ablations}


\begin{figure}
\begin{minipage}{0.36\textwidth}
    \centering
    \tiny
    \captionsetup{type=table}
    \begin{tabular}{@{}lcccccccccccccc@{}}
    \toprule
    \textbf{Task}         & \textbf{\sysx-fixed-$r$} & \textbf{\sysx}   \\
    \midrule                                                                                                           
    \spam\ (BERT)         & 88.0                     & \textbf{89.6} \\
    \weather\ (BERT)      & 80.0                     & \textbf{82.0} \\
    \spouse\ (BERT)       & 43.9                     & \textbf{48.4} \\
    \midrule                                                                                                             
    \basketball\ (RN-101) & 23.8                     & \textbf{31.3} \\
    \commercial\ (RN-101) & 89.1                     & \textbf{90.1} \\
    \tennis\ (RN-101)     & 82.6                     & \textbf{82.9} \\
    \midrule                                                                                                           
    \basketball\ (RN-101) & 34.5                     & \textbf{42.5} \\
    \commercial\ (RN-101) & 90.5                     & \textbf{91.8} \\
    \tennis\ (RN-101)     & 82.7                     & \textbf{83.1} \\
    \bottomrule \\
    \end{tabular}
    \caption{\sysx\ performance degrades when setting a fixed threshold for
    all labeling functions.
    }
    \label{table:ablations}
\end{minipage}
\begin{minipage}{0.01\textwidth}
~
\end{minipage}
\begin{minipage}{0.61\textwidth}
    \centering
    \includegraphics[width=3.4in]{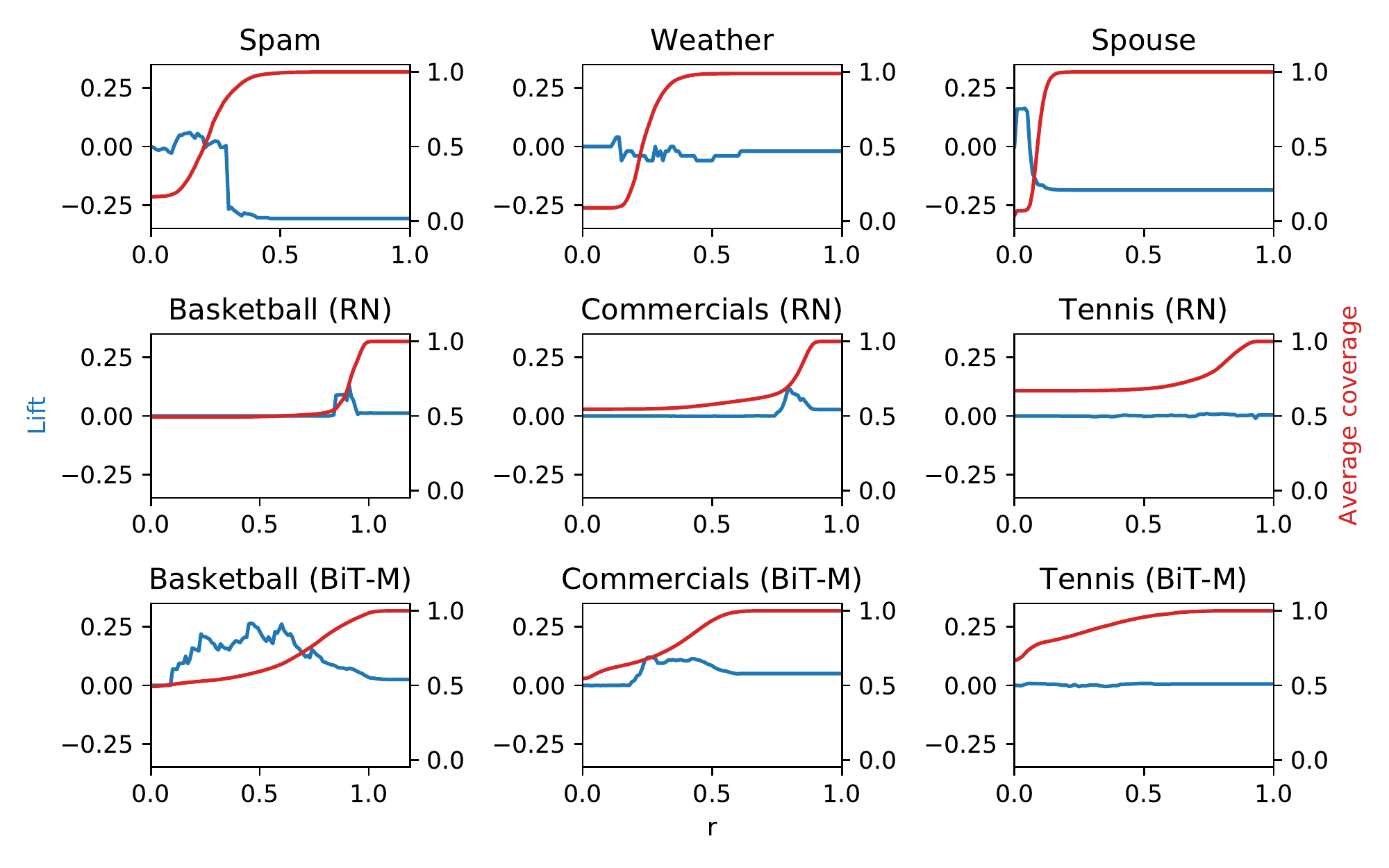} 
    \caption{Lift (blue, left Y axis) and coverage (red, right Y axis) for
    different thresholds (same threshold for all labeling sources).}
    \label{fig:threshold_ablations}
\end{minipage}
\end{figure}

We report the results of two ablation studies.
In the first, we use a single fixed threshold for extension for each labeling
function in a given task, instead of tuning different thresholds for different
labeling functions.
The results are shown in Table~\ref{table:ablations} in the \sysx-fixed-$r$
column.
Second, we report the effects of sweeping a fixed threshold for each task.
Figure~\ref{fig:threshold_ablations} shows relative lift and average coverage
of labeling functions for different values of a fixed threshold $r$.
For each task, there is a region where lift increases as $r$ increases
(corresponding to an increase in average coverage), but lift decreases if the
threshold is set too large (corresponding to the region where the extended
labeling functions become too inaccurate).

\subsection{Pre-Processing Time}
\label{sec:supp_eval_preprocess}


\begin{table}[ht!]
    \centering
    \scriptsize
    \begin{tabular}{lccclccc}
        \toprule
        \textbf{Task} (Embedding) & Embedding pre-processing time (s) \\ \midrule
        \spam\ (BERT)             & 21.3                              \\
        \weather\ (BERT)          & 6.1                               \\
        \spouse\ (BERT)           & 417.6                             \\
        \basketball\ (RN-101)     & 71.1                              \\ 
        \commercial\ (RN-101)     & 329.5                             \\
        \tennis\ (RN-101)         & 47.1                              \\
        \basketball\ (BiT-M)      & 85.6                              \\ 
        \commercial\ (BiT-M)      & 401.9                             \\
        \tennis\ (BiT-M)          & 91.5                              \\ \midrule
        Average                   & 163.5                             \\
        \bottomrule
    \end{tabular}
    \caption{
    One-time pre-processing cost in seconds to compute embeddings and pairwise
    distances for each dataset.
    }
    \label{table:preprocess}
\end{table}

Both \sysx\ and the baseline transfer learning without fine-tuning methods
require running inference over the datasets with pre-trained deep networks as a
pre-processing step.
Table~\ref{table:preprocess} reports this one-time pre-processing cost for each
dataset.
On average, pre-processing the data requires less than \num{three minutes}.

\end{document}